\documentclass[11pt]{article}
\usepackage[margin=1in]{geometry}
\usepackage{amsmath,amssymb,amsthm,mathtools}
\usepackage{enumitem}
\usepackage[round,authoryear]{natbib}
\usepackage[hidelinks]{hyperref}

\DeclareMathOperator{\Var}{Var}
\DeclareMathOperator{\Leb}{Leb}

\newcommand{\R}{\mathbb R}
\newcommand{\E}{\mathbb E}
\newcommand{\Prob}{\mathbb P}
\newcommand{\1}{\mathbf 1}
\newcommand{\dd}{\,\mathrm d}
\newcommand{\eps}{\varepsilon}
\newcommand{\OPT}{\operatorname{OPT}}
\newcommand{\Reg}{\operatorname{Reg}}
\newcommand{\SPM}{\operatorname{SPM}}
\newcommand{\ALG}{\operatorname{ALG}}

\newcommand{\supp}{\operatorname{supp}}
\newcommand{\betamin}{\underline\beta}
\newcommand{\betamax}{\overline\beta}
\newcommand{\vbar}{\overline v}
\newcommand{\Unif}{\mathrm{Unif}}

\newcommand{\pp}{\mathfrak p}

\newcommand{\Proj}{\Pi}

\theoremstyle{plain}
\newtheorem{theorem}{Theorem}[section]
\newtheorem{lemma}[theorem]{Lemma}
\newtheorem{proposition}[theorem]{Proposition}
\newtheorem{corollary}[theorem]{Corollary}
\theoremstyle{definition}
\newtheorem{definition}[theorem]{Definition}
\newtheorem{assumption}{Assumption}
\theoremstyle{remark}
\newtheorem{remark}[theorem]{Remark}

\numberwithin{equation}{section}

\title{Online Resource Allocation with Continuous Random Consumption: Regret under Degeneracy}
\author{Jiawei Zhang\\
Stern School of Business\\ New York University}
\date{\today}

\newcommand{\polylog}{\operatorname{polylog}}
\begin{document}
\maketitle

\begin{abstract}
We study online resource allocation when both rewards and consumption sizes may
be continuously distributed.  Requests arrive sequentially and must be accepted
or rejected irrevocably under fixed resource capacities.  Each request belongs to one of
finitely many observable types; conditional on the type, both the reward and the
scalar size are random, and the realized size scales a fixed type-specific
resource-consumption vector.  The
model allows the deterministic fluid relaxation to be degenerate.

We show that additive regret is governed by the size-weighted mass of requests
whose value-to-size ratios lie near the active acceptance cutoffs.  We formalize
this quantity through an active weighted-mass exponent \(\pp\).  When \(\pp>1\),
this cutoff mass is thin, and the problem is genuinely hard: every online policy
must incur regret of order at least
$
        T^{1/2-1/(2\pp)},
$
and this holds for every \(\pp>1\).  A sample-path marginal policy matches this
lower bound up to polylogarithmic factors; and when \(\pp=1\), so that the mass
grows linearly near the cutoff, it attains \(O((\log T)^2)\) regret.  For example,
if the size and the value-to-size ratio are independent and uniformly
distributed, then \(\pp=1\); if instead the size and the reward are independent
and uniformly distributed, then \(\pp=2\).  Thus the policy achieves \(o(\sqrt T)\)
regret throughout this regularity class without any fluid non-degeneracy
assumption, allowing both primal degeneracy and dual non-uniqueness.
\end{abstract}

\section{Introduction}

Online resource allocation asks how scarce resources should be allocated before
the future is known.  Requests arrive over time, reveal rewards and resource
requirements, and require immediate decisions.  Accepted or assigned requests
consume resources irreversibly, while future requests remain uncertain.  Network
revenue management, online advertising, and online order fulfillment fit this
broad template.

A central question is how much value is lost because decisions are made online.
We measure this loss by additive regret relative to a hindsight benchmark: the
offline fractional allocation that observes the entire arrival sequence before
choosing which requests to serve.  This benchmark isolates the cost of the
information gap between online and offline decision making.  The goal is to
understand how regret grows with the horizon \(T\).

This paper studies a stochastic online allocation model with accept/reject decisions in which both rewards
and resource consumptions may be continuously distributed.  Requests belong to
finitely many observable types.  Each arriving request reveals its reward and its
size; conditional on the type, both variables may be continuous.  The per-unit
resource-consumption vector is deterministic conditional on the type, so
accepting a request consumes its realized size times this vector.  When the size
is deterministic conditional on type, the model specializes to the semi-discrete
model of \citet{JiangMaZhang2025Degeneracy}.  In the one-resource, one-type case, it contains the
classical stochastic knapsack model with random profits and weights studied by
\citet{Lueker1998}.  Our main finding is that allowing continuous random consumption can change the worst-case regret exponent: there are bounded-density instances on which every online policy incurs polynomial regret.  We also prove that a sample-path marginal policy attains the matching polynomial exponent, up to logarithmic factors.

\subsection{Literature review}
Our model is most closely related to the network revenue management literature \citep{GallegoVanRyzin1994, TalluriVanRyzin2006}.  Most algorithms and
analyses are built around the deterministic fluid relaxation of the online
problem and the dual prices of that relaxation. A standard approach is to
solve a fluid relaxation and periodically re-solve it as time and remaining
capacity evolve.  These dual prices, however, need not be well behaved: when the
fluid relaxation is degenerate, the optimal dual price is not unique, and the
acceptance threshold it induces can jump under an arbitrarily small change in
remaining capacity, so a re-solving policy chases a moving target.

When rewards and resource consumptions have finite support, this approach yields $o(\sqrt{T})$ regret \citep{ReimanWang2008}. Under a non-degeneracy assumption on the fluid relaxation, \citet{JasinKumar2012} obtain an $O(1)$ regret bound. Without such a non-degeneracy assumption, \citet{ArlottoGurvich2019} establish an $O(1)$ bound for the multisecretary problem. Subsequent work obtains $O(1)$ bounds for the more general network revenue management model \citep{BumpensantiWang2020,VeraBanerjee2021,VeraEtAl2021,LiWangZhang2025}.
For more general reward distributions,
logarithmic regret is achievable under suitable distributional regularity
\citep{Lueker1998}, but existing
analyses typically also impose non-degeneracy or stability conditions on the
deterministic fluid relaxation \citep{LiYe2022,Bray2024,BalseiroBesbesPizarro2023}.
These conditions appear in several forms,
including strict complementarity, uniqueness of the optimal primal basis,
uniqueness of the optimal dual solution, and second-order growth of the dual
objective.

However, the non-degeneracy condition may fail, and degeneracy is not a pathological corner case.  As highlighted by
\citet{BumpensantiWang2020}, it is likely to occur in practice.  In capacity
planning, when the expected demand for a resource is of order \(T\), the
square-root law of inventory suggests a safety-stock scale of order \(\sqrt T\).
When demand fluctuates, different resources may become bottlenecks at different
times.  In the associated linear program, the set of binding constraints may
therefore change as capacity is depleted.  Thus degeneracy is often the natural
operating regime, not an exception.  See \citet{Bray2024} and
\citet{JiangMaZhang2025Degeneracy} for related discussions.

Recent work clarifies which form of non-degeneracy is most relevant for
certainty-equivalent (CE) resolving policies.  \citet{ChenWang2025} separate dual uniqueness
from primal non-degeneracy and show that CE's performance is governed by stability of the fluid dual
price, rather than by primal non-degeneracy itself.  When the optimal dual price,
and hence the acceptance threshold it induces, is stable, CE can attain \(o(\sqrt T)\) regret, and in some distributional settings
logarithmic-level regret, even if the primal fluid solution is degenerate.  The
fluid dual price, however, need not be stable.  \citet{BesbesEtAl2024} show that
the CE algorithm can lose logarithmic guarantees and incur
\(\sqrt T\)-scale regret for the multisecretary problem with multiple types,
each with a uniformly distributed reward.  They
design a different policy and obtain \((\log T)^2\) regret.  \citet{JiangMaZhang2025Degeneracy} obtain a
\((\log T)^2\) bound for a special case of our model when the size of each type is deterministic. \citet{Zhang2026MultisecretaryGaps} shows that this \((\log T)^2\) rate is
tight.  Thus, even when consumption has finite support and reward densities are
continuous and bounded below near the relevant cutoffs, degeneracy has a price.
But the price is mild: logarithmic regret becomes \((\log T)^2\).

\subsection{Main results}

This picture changes once both consumption and reward are continuously
distributed.  The change is already visible in one-type, one-resource stochastic
knapsack instances.  We use the following two examples throughout the paper.

\begin{enumerate}[label=\emph{Example \arabic*.},leftmargin=*,labelindent=1.5em,itemsep=2pt,topsep=4pt,parsep=0pt]
\item The reward \(V\) and the size \(\beta\) are independent and uniform on
\([1,2]\).
\item The size \(\beta\) is uniform on \([1,2]\), and the value-to-size ratio
\(R:=V/\beta\) is uniform on \([1/2,2]\) and independent of \(\beta\).
\end{enumerate}

A request with reward \(V\) and size \(\beta\) consumes \(\beta\) units of the
resource. The relevant score is therefore the reward per unit consumed,
\(R=V/\beta\). At the fluid scale, the optimal rule is a ratio-threshold rule:
requests with larger values of \(R\) are accepted before requests with smaller
values of \(R\).

Both examples have mean size \(\E[\beta]=\tfrac32\), so their expected total
demand is \(\tfrac32 T\). At each fixed capacity \(cT\) with \(0<c<3/2\), the resource
is scarce. The fluid solution rejects the lowest-ratio requests and uses a unique
interior cutoff, with a unique dual price. At the critical capacity \(\tfrac32 T\),
the resource constraint is exactly tight when every request is accepted.
Accept-all is then fluid-optimal, and every price in \([0,1/2]\) is an optimal
dual price. Thus the fluid dual degenerates
(Appendix~\ref{app:two-uniform-duals} computes both fluid duals).

Now compare a non-critical capacity, say \(1.4T\), with the critical capacity
\(\tfrac32 T\). In Example~1, the regret is \(O((\log T)^2)\) at \(1.4T\), by
the interior bound in Appendix~\ref{app:local}. At the critical capacity
\(\tfrac32 T\), Theorem~\ref{thm:lower} and Corollary~\ref{cor:regret-bounded} show a
\(T^{1/4}\) polynomial order, up to logarithmic factors. Thus the critical
degeneracy costs a polynomial factor. In Example~2, by contrast, the regret is
\(O((\log T)^2)\) at both capacities: degeneracy does not increase the order of
regret.\footnote{For these single-type instances, the polylogarithmic upper
bounds can be sharpened to \(O(\log T)\). The extra logarithmic factor in the
displayed bound comes from a uniform argument over type-level cutoffs. When there
is a single scalar cutoff, that uniformity is unnecessary, and the sharper bound
recovers the \(O(\log T)\) order of \citet{Lueker1998}. We do not carry out that
single-type refinement here.} Both examples are degenerate at \(\tfrac32 T\), and
both have the same dual-optimal interval \([0,1/2]\). Nevertheless, one pays a
polynomial price of degeneracy and the other pays none. The difference is not
dual multiplicity alone; it is how the size-weighted ratio mass behaves near the
degenerate cutoff.

We formalize this behavior through an exponent \(\pp\ge1\), which measures how
fast the size-weighted ratio mass accumulates near an active acceptance cutoff.
In Example~2, the ratio \(R\) is uniform and independent of size, so this mass
grows linearly with distance from the cutoff, and \(\pp=1\). In Example~1, the
critical cutoff is the lower edge \(1/2\) of the ratio support. This edge is a
corner of the joint \((\beta,R)\) support: a ratio near \(1/2\) forces the reward
to be near its lower endpoint and the size to be near its upper endpoint. The
nearby mass therefore grows quadratically, and \(\pp=2\).

The regret rate is governed by this active exponent. It is polylogarithmic when
\(\pp=1\), and its polynomial order is
\[
        T^{1/2-1/(2\pp)}
\]
when \(\pp>1\), up to logarithmic factors. At \(\pp=2\), this gives the
\(T^{1/4}\) order in Example~1. Thus the exponent at the degenerate cutoff is the
price of degeneracy: no polynomial price when \(\pp=1\), and a polynomial price
when \(\pp>1\). This is the new continuous-consumption effect. With deterministic
size and a reward density that is bounded and positive at the cutoff, this corner
mechanism cannot occur.

These rates are attained by a sample-path marginal policy
(Theorem~\ref{thm:main}), and the polynomial exponent is matched by a lower bound
for every \(\pp>1\) (Theorem~\ref{thm:lower}). The policy therefore achieves
\(o(\sqrt T)\) regret throughout this class.

We emphasize that existing \(o(\sqrt T)\) additive-regret bounds for models with
continuously distributed rewards and continuously distributed resource
consumption are obtained under conditions that rule out the degeneracies studied
here.  Some papers, such as \citet{LiYe2022} and \citet{Bray2024}, impose
non-degeneracy assumptions explicitly.  Others, such as \citet{Lueker1998} and
\citet{ChenWang2025}, impose primitive distributional assumptions under which the
relevant fluid dual price is unique or stable.  By contrast, our model and bounds
allow both primal degeneracy and dual non-uniqueness.

The closest prior logarithmic-regret results do not directly cover this
continuous-consumption, continuous-reward regime.  The models of
\citet{BesbesEtAl2024} and \citet{JiangMaZhang2025Degeneracy} have finite-support
resource consumptions: the former is a multisecretary model, and the latter has
deterministic consumption vectors conditional on type.  The stochastic-knapsack
model of \citet{ArlottoXie2020} allows random sizes, but assumes equal
deterministic rewards.  Thus these results do not directly apply to the continuous-size,
continuous-reward instances that drive the new rates here.

\subsection{Overview of the analysis}\label{subsec:overview}

The policy prices each accepted request by the marginal value of the capacity it
consumes in the expected offline problem---that is, by the drop in the average
hindsight value over future arrivals caused by reserving that capacity.  This is
the RAMS principle of \citet{BesbesEtAl2024}: charge this marginal loss as the
price of capacity, so that minimizing the resulting per-step loss controls total
regret.  The reduction of regret to a sum of per-step losses goes back to
\citet{VeraBanerjee2021}, and is also used by \citet{Bray2024} and
\citet{JiangMaZhang2025Degeneracy}.  Our contribution is to bound the per-step
loss for our model---with both the reward and the consumption continuously
distributed, and without a non-degeneracy assumption.

Proving this bound rests on three ideas.  The first explains how the policy
prices capacity without selecting a dual price.  The second identifies the
quantity that remains stable under degeneracy.  The third handles the new
difficulty created by random consumption.

\paragraph{Price by an average of cutoffs, not by a selected dual price.}

Consider a request of type \(k\) with realized size \(z\), arriving when the
remaining capacity is \(b\).  Accepting it moves the capacity from \(b\) to
\(b-za^k\).  A classical analysis would price this capacity change by selecting
a dual price at one of these capacities.  This is unstable under degeneracy:
many prices may be optimal at the same capacity, and a selected price can jump
after an arbitrarily small perturbation.

We avoid selecting a price.  Instead, we traverse the segment from \(b\) to
\(b-za^k\).  Along this segment, the hindsight value is a concave function of
one scalar.  Such a function is differentiable almost everywhere, and at each
differentiability point its slope is unambiguous.  That slope acts as an
acceptance cutoff for value-to-size ratios: it is the threshold \(q\) above
which a request with ratio \(V/\beta\) is worth accepting and below which it is
not.  The marginal value charged by the
policy is the average of these cutoffs along the segment.  The kinks where dual
prices are ambiguous form a null set, and the average passes through them.

This averaging makes the policy well defined without a non-degeneracy
assumption.  It does not, by itself, prove a regret bound.  The loss still
depends on how far the cutoffs generated by different future sample paths can
spread, and degeneracy is precisely the case in which cutoff movement need not
be stable.

\paragraph{Measure losses by borderline mass, and prove the product is stable.}

For a feasible request of fixed type and size, changing the cutoff can change
the decision only when the value-to-size ratio lies between the two cutoffs.
Thus the one-period loss is controlled by two quantities: the width of the band
between the cutoffs, and the resource carried by requests whose ratios fall
inside that band.  We call these requests borderline.

The key point is that the product of these two quantities is stable, even when
the cutoff width itself is not.  Under degeneracy, two future sample paths can
produce cutoffs that are far apart.  But if little ratio mass lies between those
cutoffs, the movement is mostly harmless.  The analysis therefore controls
\[
        \text{cutoff width}
        \times
        \text{borderline resource mass},
\]
rather than cutoff movement alone.

We prove this product bound by comparing two typical future sample paths
directly.  Their empirical type totals and capacities are close.  If their
cutoffs differ, the difference in the resource they accept is carried entirely
by requests whose ratios lie between the two cutoffs.  A classical stability
estimate for linear programs, due to \citet{Hoffman1952} and requiring no
uniqueness of the optimal solution, makes the resource each path accepts close
whenever the two paths themselves are close.  Combining these two facts nearly
closes the product bound.

The remaining term is proportional to the cutoff width alone.  It corresponds to
degenerate stretches where a cutoff moves while sweeping little or no mass.  The
distributional exponent \(\pp\) closes this gap: a ratio band of width \(\ell\)
near an active cutoff must carry weighted mass at least of order \(\ell^\pp\).
This converts leftover width back into mass.  Thus a smaller \(\pp\) means more
mass near the cutoff and less regret, while a larger \(\pp\) means thinner mass
and a larger price for degeneracy.

With \(n\) arrivals remaining, empirical fluctuations are of order
\(\sqrt{\log n/n}\).  After the active-mass closure, the one-period bounds sum
to \(O((\log T)^2)\) when \(\pp=1\), and to
\(O(T^{1/2-1/(2\pp)}\polylog T)\) when \(\pp>1\).  The stable object is therefore
not a dual price.  It is the product of cutoff width and borderline resource
mass.

\paragraph{Random consumption: different sizes see different bands.}

The product bound just described aggregates over sizes.  The loss from a
specific arriving request, however, is conditional on that request's realized
size.  A request of size \(z\) is borderline when its ratio lies in the cutoff
band, and this event depends on the ratio distribution conditional on size
\(z\).

When consumption is deterministic conditional on the type, the conditional and
aggregate distributions coincide, and the product bound finishes the one-period
analysis.  With continuously distributed consumption, they can differ sharply.
The difficulty is most visible near a corner of the joint reward-size support.
For each size \(z\), the attainable ratios may begin at a size-dependent lower
edge; as \(z\) changes, that edge can move into or out of the cutoff band.  A
short band can then contain little aggregate mass but still capture a large share
of the conditional mass for sizes whose edge falls inside the band.  A pointwise
comparison between conditional and aggregate mass is therefore false.

We show that an averaged comparison is enough.  We integrate the conditional
band mass over the size distribution.  Sizes whose moving edge lies inside the
band and sizes whose moving edge lies outside the band balance each other after
integration.  This recovers the aggregate product bound, up to one logarithmic
factor.  This size-integration step is the technical price of continuously
distributed consumption.

The lower bound isolates the same quantity from the other side.  At the corner
of Example~1, a band of width \(\varepsilon\) carries resource of order
\(\varepsilon^\pp T\).  Such a band cannot be reliably distinguished from the
\(\sqrt T\)-scale fluctuation of total demand until
\(\varepsilon\asymp T^{-1/(2\pp)}\).  Misclassifying requests in that band costs
order \(\varepsilon\) per unit of resource, which forces regret of order
\[
        \varepsilon\sqrt T
        \asymp
        T^{1/2-1/(2\pp)} .
\]
Section~\ref{subsec:proof-roadmap} turns this outline into a step-by-step
roadmap with pointers to the formal statements.  Sections~\ref{sec:spm}
and~\ref{sec:lower} then prove the upper and lower bounds.

\subsection{Further related work}
Classical dynamic stochastic knapsack models with sequential arrivals and
admission control were studied by
\citet{KleywegtPapastavrou1998,KleywegtPapastavrou2001}.
\citet{MarchettiVercellis1995} prove a $(\log T)^{3/2}$ regret bound for the online stochastic knapsack problem.  \citet{ArlottoXie2020} prove logarithmic regret for an equal-reward stochastic knapsack with random item sizes. \citet{JiangZhang2020} generalize this result to the multi-resource case.

A separate literature studies stochastic knapsack and online packing through
multiplicative approximation guarantees.  One line compares algorithms with the
optimal adaptive policy for stochastic knapsack or stochastic packing
\citep{DeanGoemansVondrak2005,DeanGoemansVondrak2008,
BhalgatGoelKhanna2011,Ma2018}.  Another line studies prophet or LP benchmarks
for online stochastic knapsack and obtains constant competitive ratios
\citep{DuttingFeldmanKesselheimLucier2020,JiangMaZhang2022}.  In random-order
online packing and online linear programming, large-capacity assumptions lead to
\(1-o(1)\) competitive ratios relative to the offline optimum
\citep{KesselheimRadkeTonnisVocking2018,AgrawalEtAl2014}.

Other finite-type online allocation models also admit constant or uniformly
bounded additive regret.  Examples include online packing, matching, and pricing
\citep{VeraBanerjee2021,VeraEtAl2021}, overbooking
\citep{FreundZhao2022}, online decision-making with an uncertain horizon
\citep{BanerjeeFreund2024}, dynamic matching
\citep{AsadpourWangZhang2020,Gupta2024,WeiXuYu2023}, and online resource allocation via primal-dual
policies \citep{HeWeiXuYu2025}.  A related line develops computationally
efficient primal-dual, first-order, and re-solving methods for broader
distributional settings.  Classical first-order and primal-dual methods attain
\(O(\sqrt T)\) regret
\citep{BalseiroEtAl2023,LiEtAl2020,JiangLiZhang2025Wasserstein}, while recent
refinements obtain sub-\(\sqrt T\) regret under additional regularity or
non-degeneracy conditions \citep{GaoEtAl2024,MaEtAl2025}.

The rest of the paper is organized as follows.
Section~\ref{sec:model} states the model, the distributional regularity
condition, and the main regret theorem.  Section~\ref{sec:spm} defines the
sample-path marginal policy and proves the upper bound.
Section~\ref{sec:lower} proves matching lower bounds for the polynomial
exponent.  Section~\ref{sec:conclusion} concludes.

\section{Model, Assumptions, and Main Results}\label{sec:model}

This section gives the formal setup and states the main regret bounds. We first define the
online allocation model and the fractional hindsight benchmark. We then introduce the
sample-path marginal policy. Finally, we define the weighted-ratio distribution, state the
standing distributional assumption, and state the main theorem.

\subsection{Model and the Sample-Path Marginal Policy}\label{subsec:model}

There are $d$ resources and $K$ request \emph{types}. Type $k$ has a fixed
consumption direction $a^k\in\R_+^d$, with $a^k\neq0$. Over a horizon of
$T$ periods, requests arrive i.i.d. The period-$t$ request is
$Z_t=(J_t,\beta_t,V_t)$, where $\Prob(J_t=k)=\pi_k >0$, $\beta_t>0$ is the \emph{size}, and $V_t\ge0$ is the reward.
The decision maker observes $Z_t$ and irrevocably accepts or rejects it. An
accepted request consumes $\beta_t a^{J_t}$ from a budget $b_T\in\R_+^d$.
Writing $x_t\in\{0,1\}$ for the online decision, the realized reward is
$\sum_t V_t x_t$, and feasibility requires
\[
        \sum_t \beta_t a^{J_t}x_t\le b_T.
\]

We measure performance of an online algorithm against the fractional hindsight
(offline) optimum. For a length-$n$ arrival sequence $W_n=(Z_1,\dots,Z_n)$ and a
capacity $b\in\R_+^d$, define the $n$-period fractional hindsight optimum
\begin{equation*}
        \OPT_n(b;W_n)=
        \max_{0\le x_i\le1}\left\{\sum_{i=1}^n V_i x_i:
        \sum_{i=1}^n \beta_i a^{J_i}x_i\le b\right\}.
\end{equation*}
The additive regret of an online algorithm $\ALG$ is
\begin{equation*}
        \Reg_T(\ALG;b_T)=\E\bigl[\OPT_T(b_T;W_T)\bigr]-\E\Bigl[\sum_{t=1}^T V_t x_t^{\ALG}\Bigr].
\end{equation*}
The fractional optimum $\OPT_T(b_T;W_T)$ upper bounds the binary hindsight
optimum, and the two differ by at most $d\vbar$: a basic optimal solution of the
fractional LP has at most $d$ fractional variables (the rest are pinned at $0$ or
$1$ by the box constraints $0\le x_i\le1$), and rounding these down preserves
feasibility while losing at most $d\vbar$ of reward. The fractional benchmark
therefore yields the same regret rates as long as the reward has a bounded
support.

The instances are indexed by $T$. We assume the capacity grows linearly in the
horizon, $b_T=\Theta(T)$, the standard fluid scaling in this literature.

The sample-path marginal policy prices the capacity consumed by the current
request through the expected fractional hindsight value of the remaining
periods. For $n\ge0$, set
\[
        \Phi_n(b)=\E[\OPT_n(b;W_n)],
\]
the expected fractional hindsight value of a length-$n$ i.i.d.\ sequence, with
$\Phi_0\equiv0$. We use the convention that $\OPT_n(b;W_n)=-\infty$ and
$\Phi_n(b)=-\infty$ when $b\notin\R_+^d$, where the feasible set is empty; thus
$\Phi_n$ is evaluated only at capacities in $\R_+^d$. In the Bellman expansions
below, every term $\Phi_{s-1}(B_s-\beta_s a^{J_s})$ carries the feasibility
indicator $\1\{\beta_s a^{J_s}\le B_s\}$, and the product is read as $0$ when
that indicator vanishes. At a decision epoch with $s$ periods
remaining and remaining capacity $b$, suppose that the current arrival is
$(J,\beta,V)=(k,z,v)$. Define the marginal value of the capacity consumed by
this request as
\begin{equation}
        \Delta_s(b,k,z)
        =
        \begin{cases}
        \Phi_{s-1}(b)-\Phi_{s-1}(b-za^k), & za^k\le b,\\
        +\infty, & \text{otherwise.}
        \end{cases}
        \label{eq:marginal}
\end{equation}
The \emph{sample-path marginal policy} $\SPM$ accepts the request if and only if
\begin{equation}
        v\ge\Delta_s(b,k,z).
        \label{eq:accept}
\end{equation}

The rule \eqref{eq:accept} is the sample-path instance of the RAMS principle of
\citet{BesbesEtAl2024}, which prices each action by the marginal value of the
capacity it consumes in the hindsight problem. As with RAMS, implementing $\SPM$
requires simulation: the expected hindsight value $\Phi_{s-1}$ has no closed form
and is estimated by sampling future arrivals.

\begin{remark}
Notice that the threshold in \eqref{eq:accept} is the marginal value of the
expected hindsight relaxation $\Phi_{s-1}$. It is not a dual price of a
deterministic fluid program. When the fluid relaxation is degenerate, the
bid-price control algorithm with thresholds computed from fluid dual prices can
be highly sensitive to small changes in capacity. In contrast, the function
$\Phi_{s-1}$ is Lipschitz in the remaining capacity through the kinks at which a
fluid dual price would jump, so the marginal value in \eqref{eq:marginal} varies
continuously where a certainty-equivalent threshold does not. This distinction is
what lets $\SPM$ handle degenerate instances.
\end{remark}

\subsection{A primitive distributional assumption}\label{subsec:assumption}

We now state the distributional regularity condition used in the regret analysis. The
primitive ranges are bounded. Conditional on type $k$, the size satisfies
$\beta\in[\betamin_k,\betamax_k]$, with $\betamin_k>0$, and the reward satisfies
$V\in[0,\vbar]$. Put $\betamin=\min_k\betamin_k$ and
$\betamax=\max_k\betamax_k$. The value-to-size ratio is
$
        R=\frac{V}{\beta}.
$
Since $\beta\ge\betamin_k>0$ conditional on type $k$, this also gives
$\E[\beta\mid J=k]>0$ and $\overline m_k=\pi_k\,\E[\beta\mid J=k]>0$, so every
expression of the form $\betamax_k/(\pi_k\E[\beta\mid J=k])$ is well defined.

Every resource is consumed by at least one type; unused coordinates are dropped.
For each resource
$j$, define
\[
        \alpha_j=\min\{a^k_j:a^k_j>0\},
        \qquad
        M_j=\frac{\vbar}{\betamin\alpha_j},
        \qquad
        y_k=\sum_{j=1}^d M_j a^k_j,
        \qquad
        y=\max_k y_k.
\]
Then every realized ratio satisfies $R\in[0,y]$. 

The main regularity condition is imposed on a type-wise weighted ratio measure.
For each type $k$ and Borel set $B\subseteq[0,y]$, define
\begin{equation*}
        \mu_k(B)=\pi_k\,\E\!\left[\beta\,\1\{R\in B\}\mid J=k\right].
\end{equation*}
In words, $\mu_k(B)$ is the expected type-$k$ resource carried by requests whose
value-to-size ratio lies in $B$.
All hypotheses below are conditions on the arrival distribution through the
measures $\mu_k$ and the conditional curvature measures defined next. For each
size $z$, define
\[
        \Lambda_{k,z}(B)=z\,\Prob(R\in B\mid J=k,\beta=z),
        \qquad B\subseteq[0,y].
\]
This is a finite Borel measure on ratio space. The kernel $\Lambda_{k,z}$ is the
size-conditioned counterpart of $\mu_k$: it records how ratio mass is distributed
among type-$k$ requests of size $z$. Let $P_k^\beta$ denote the
conditional law of $\beta$ given $J=k$, and denote
$
        \mathcal B_k=[\betamin_k,\betamax_k].
$
Define the finite kernel measure
\[
        \mathfrak M_k(\dd z,\dd r)
        :=
        \pi_k P_k^\beta(\dd z)\Lambda_{k,z}(\dd r).
\]
Equivalently, for every nonnegative measurable $g$,
\[
        \int g(z,r)\,\mathfrak M_k(\dd z,\dd r)
        =
        \pi_k\int_{\mathcal B_k}
        \left[
        \int g(z,r)\Lambda_{k,z}(\dd r)
        \right]
        P_k^\beta(\dd z).
\]
The ratio marginal of $\mathfrak M_k$ is the weighted-ratio measure:
\begin{equation*}
        \mathfrak M_k(\mathcal B_k\times B)
        =
        \pi_k\int_{\mathcal B_k}\Lambda_{k,z}(B)\,P_k^\beta(\dd z)
        =
        \mu_k(B),
        \qquad B\subseteq[0,y].
\end{equation*}
Endpoint-contact assumptions below are imposed on submeasures of
$\mathfrak M_k$, not on the law conditioned on $R\in U$.

\begin{definition}[Contact branch]\label{def:branch}
Fix a type $k$ and a one-sided endpoint neighborhood $U$. Let $x$ denote the
oriented distance into the support, so that $x=R-r$ at a lower endpoint and
$x=r-R$ at an upper endpoint. After this orientation, identify $U$ with a local
interval $[0,x_0]$.

A \emph{contact branch with exponent $\theta>0$} is a branch submeasure
$
        \mathfrak M_k^{\mathrm{br}}
        \le
        \mathfrak M_k|_{\mathcal B_k\times U}
$
represented by a size coordinate $\omega\in(0,\omega_0)$. It consists of an
injective $C^1$ size map $\beta_k(\omega)$, whose Jacobian is bounded above and
below by primitive constants, a weight $w$ with $c\le w(\omega)\le C$, a size
density $f(\omega)$, an edge function $e(\omega)$, and exponents $\alpha,\tau>0$
and $\gamma\ge1$ with
$
        \theta=\gamma+\alpha/\tau,
$
such that, with $\Lambda^{\mathrm{br}}_{k,\omega}$ denoting a finite branch
curvature measure supported on $[e(\omega),x_0]$ and having density
$\lambda^{\mathrm{br}}_{k,\omega}$,
\begin{equation}
\begin{gathered}
        f(\omega)\asymp\omega^{\alpha-1},
        \qquad
        e(\omega)\asymp\omega^{\tau},
        \qquad
        \lambda^{\mathrm{br}}_{k,\omega}(x)\asymp(x-e(\omega))^{\gamma-1},\\[1mm]
        \int g(z,x)\,\mathfrak M_k^{\mathrm{br}}(\dd z,\dd x)
        =
        \int_0^{\omega_0}
        w(\omega)
        \int g(\beta_k(\omega),x)\,
        \Lambda^{\mathrm{br}}_{k,\omega}(\dd x)\,
        f(\omega)\,\dd\omega
\end{gathered}
        \label{eq:branch-disint-def}
\end{equation}
for every nonnegative measurable $g$.

Equivalently, the branch ratio marginal is
\begin{equation*}
        \mu_k^{\mathrm{br}}(I)
        :=
        \mathfrak M_k^{\mathrm{br}}(\mathcal B_k\times I)
        =
        \int_0^{\omega_0}
        w(\omega)\,\Lambda^{\mathrm{br}}_{k,\omega}(I)\,
        f(\omega)\,\dd\omega,
        \qquad I\subseteq U.
\end{equation*}
The submeasure $\mu_k^{\mathrm{br}}$ is the branch's \emph{contact submeasure}.
All constants are primitive. The exponents $\alpha,\tau,\gamma$ and the contact
exponent $\theta=\gamma+\alpha/\tau$ are local to the endpoint-contact machinery;
they are not to be confused with the capacity-sweep parameter $\theta$ of
Section~\ref{sec:spm}, the resource constants $\alpha_j$, or the capacity
tolerance $\tau$ of Proposition~\ref{prop:active}.
\end{definition}

Informally, an endpoint-contact representation captures the case in which the
feasible ratio support begins at a size-dependent boundary, so that type-$k$
ratio mass accumulates only as the size moves away from that boundary.

\begin{definition}[Endpoint-contact representation]\label{def:endpoint-contact}
Fix a type $k$. A one-sided endpoint neighborhood $U$ admits a
\emph{single-branch endpoint-contact representation with exponent $\theta$} if
there exist measurable finite subkernels
$
        \Lambda^D_{k,z}$ and $
        \Lambda^{\mathrm{br}}_{k,z}
$
on $U$ such that, for $P_k^\beta$-a.e.\ $z$,
\begin{equation}
        \Lambda_{k,z}|_U
        =
        \Lambda^D_{k,z}
        +
        \Lambda^{\mathrm{br}}_{k,z}.
        \label{eq:endpoint-kernel-split}
\end{equation}
Define the corresponding kernel submeasures by
\[
        \mathfrak M_k^D(\dd z,\dd r)
        :=
        \pi_k P_k^\beta(\dd z)\Lambda^D_{k,z}(\dd r),
        \qquad
        \mathfrak M_k^{\mathrm{br}}(\dd z,\dd r)
        :=
        \pi_k P_k^\beta(\dd z)\Lambda^{\mathrm{br}}_{k,z}(\dd r).
\]
The branch submeasure $\mathfrak M_k^{\mathrm{br}}$ is required to be a single
contact branch with exponent $\theta$ in the sense of
Definition~\ref{def:branch}. The ratio marginals are
\begin{equation}
        \mu_k^D(I)
        :=
        \mathfrak M_k^D(\mathcal B_k\times I),
        \qquad
        \mu_k^{\mathrm{br}}(I)
        :=
        \mathfrak M_k^{\mathrm{br}}(\mathcal B_k\times I).
        \label{eq:endpoint-ratio-marginals}
\end{equation}
They satisfy, modulo null sets,
\begin{equation*}
        \mu_k|_U
        =
        \mu_k^D+\mu_k^{\mathrm{br}}.
\end{equation*}
The dominated part satisfies, in the local endpoint coordinate,
\begin{equation}
        \mu_k^D([0,x])\le Cx^\theta,
        \qquad
        \Lambda^D_{k,z}(I)\le C\mu_k(I)
        \quad
        \text{for every interval }I\subseteq U
        \text{ and for }P_k^\beta\text{-a.e.\ }z .
        \label{eq:dominated-component}
\end{equation}
All constants are primitive and uniform over $U$.
\end{definition}

\begin{definition}[Local endpoint mass exponent]\label{def:exponent}
Let $\nu$ be a finite Borel measure on $\mathbb R$ with interval support
$S_\nu$. For $r\in S_\nu$, a right local exponent is a number $\theta>0$, when
it exists, such that
\[
        \nu([r,r+x])\asymp x^\theta
        \qquad\text{as }x\downarrow0.
\]
A left local exponent is defined by the reflected relation
\[
        \nu([r-x,r])\asymp x^\theta
        \qquad\text{as }x\downarrow0.
\]
If $r$ is an endpoint of $S_\nu$, the one-sided local exponent from within the
support is called the endpoint exponent of $\nu$ at $r$. For the weighted ratio
measure $\nu=\mu_k$, we write $\pp_{k,r}$ for this endpoint exponent when it
exists.
\end{definition}

Informally, Assumption~\ref{ass:regularity} controls how the size-weighted ratio
mass behaves near an active acceptance cutoff.  Part~(a) says this mass is at
least \(\ell^{\pp}\) and at most \(\ell\) over a window of width \(\ell\), with
the exponent \(\pp\) governing the lower bound; part~(b) says that conditioning
on a request's realized size creates no additional concentration of mass, except
possibly at a corner of the support, which is handled separately; and part~(c)
says such corners cannot arise in the benign regime \(\pp=1\).

\begin{assumption}[Weighted-ratio regularity]\label{ass:regularity}
In addition to the boundedness assumptions above, the arrival distribution
satisfies the following conditions for some exponent $\pp\ge1$ and primitive
constants $0<c\le C<\infty$.
\begin{enumerate}[label=\textup{(\alph*)},leftmargin=2.2em]
\item \textup{(single-interval support and active mass)} For each type $k$, the
weighted ratio measure has compact single-interval support
\[
        S_k=\supp\mu_k=[r_k^-,r_k^+]\subseteq[0,y].
\]
For every interval $I\subseteq[0,y]$, write
\[
        \ell_k(I)=\Leb(I\cap S_k).
\]
The active-mass bounds are
\begin{equation}
        c\,\ell_k(I)^{\pp}\le \mu_k(I)\le C\,\ell_k(I).
        \label{eq:wr}
\end{equation}

\item \textup{(finite conditional-curvature cover)} For each type $k$, the support $S_k$
admits a finite cover by neighborhoods of the following two kinds, each relatively
open in $S_k$ and possibly overlapping, so that the cover has a positive Lebesgue number.

\emph{Dominated neighborhoods.} On a dominated neighborhood, for a.e.\ $z$ and
every interval $I$ contained in that neighborhood,
\[
        \Lambda_{k,z}(I)\le C\,\mu_k(I) .
\]
This simultaneous-in-$I$ requirement is equivalent to the per-interval bound---that
for each fixed $I$ the inequality holds for a.e.\ $z$---by applying the latter to
the countable family of rational-endpoint intervals, intersecting the
corresponding full-measure sets of $z$, and extending to all intervals by
continuity of the finite measures.

\emph{Endpoint-contact neighborhoods.} An endpoint-contact neighborhood is a one-sided
neighborhood of an endpoint of $S_k$ that admits the endpoint-contact representation of
Definition~\ref{def:endpoint-contact}.

\item \textup{(regular case)} When $\pp=1$, the finite cover consists only of dominated
neighborhoods. Endpoint-contact neighborhoods may appear only when $\pp>1$.
\end{enumerate}

All constants in the active-mass bounds, finite cover, dominated-neighborhood comparison, and
endpoint-contact representation are primitive and independent of $T$.
\end{assumption}

The single-interval support condition in Assumption~\ref{ass:regularity} is a
modeling convention rather than a substantive restriction when support components
are observable. A type whose ratio distribution has several observable support
components can be split into several observable types, one for each component.
This convention keeps the projection and normal-cone arguments below free of
internal gaps. No rectangular or product support for $(\beta,R)$ is assumed,
except inside the endpoint-contact neighborhoods of
Definition~\ref{def:endpoint-contact}. For later use, define the type-wise
projection onto the active ratio support by
\begin{equation*}
        \Pi_k(q)=\min\{r_k^+,\max\{r_k^-,q\}\},\qquad q\in[0,y].
\end{equation*}
For a general compact interval $B=[B_-,B_+]$, we use the analogous notation
\[
        \Pi_B(q)=\min\{B_+,\max\{B_-,q\}\}.
\]
The projection collapses only the two
value-flat rays outside a single support interval; it does not cross internal
gaps.

The two local structures in Assumption~\ref{ass:regularity} have different
roles. On a dominated neighborhood, the size-conditioned curvature is comparable
to the weighted ratio measure. Thus conditioning on the realized size does not
create an additional singularity. Endpoint-contact neighborhoods are the only
places where this domination may fail. There, conditional on size, the feasible
ratio interval can start at a size-dependent edge $e(\omega)$. The contact
conditions in Definition~\ref{def:branch} provide exactly the extra structure
needed for the endpoint Hardy estimate, a weighted integral inequality used in
Section~\ref{sec:spm}. The fully dominated requirement when $\pp=1$ rules out
such non-dominated contact in the regular regime.

The exponent $\pp$ in Assumption~\ref{ass:regularity} is the active
weighted-mass exponent that determines the regret rate in Theorem~\ref{thm:main}.
It is a property of the arrival distribution, not of the policy or the capacity
ratio. If several values of $\pp$ satisfy the active-mass condition
\eqref{eq:wr}, the sharpest bound is obtained by taking the smallest admissible
one. In the structured classes below, this exponent can be computed explicitly
from the endpoint growth of the weighted ratio measure.

\subsection{Main Regret Bounds}\label{subsec:mainbound}

We now state the regret guarantee for the sample-path marginal policy. The theorem
gives a distribution-dependent rate through the active weighted-mass exponent
$\pp$, and the corollaries translate this exponent into primitive conditions for
several common model classes.

\begin{theorem}[Regret bound for the sample-path marginal policy]\label{thm:main}
Under the model assumptions of Section~\ref{subsec:model}, the linear capacity
scaling $b_T=\Theta(T)$, and Assumption~\ref{ass:regularity} with active weighted-mass
exponent $\pp$, the sample-path marginal policy defined with the exact expected hindsight value $\Phi_n=\E[\OPT_n]$ satisfies
\begin{equation*}
        \Reg_T(\SPM;b_T)\le
        \begin{cases}
        C\,(\log(eT))^2, & \pp=1,\\[1mm]
        C\,T^{1/2-1/(2\pp)}\,(\log(eT))^{(\pp+1)/(2\pp)+1}+C, & \pp>1.
        \end{cases}
\end{equation*}
Here the constant $C$ depends only on the model primitives (including the
direction matrix) and the constants in Assumption~\ref{ass:regularity}; it
depends neither on $T$ nor on the capacity $b_T$. In fact the bound holds for
every capacity $b_T\in\R_+^d$; the linear scaling $b_T=\Theta(T)$ is the regime of
interest and enters only through the matching lower bound of
Section~\ref{sec:lower}.
\end{theorem}

The proof of Theorem~\ref{thm:main} is presented in Section~\ref{sec:spm}.
Concentration bounds the empirical cutoff error; the active weighted-mass
condition \eqref{eq:wr} then governs how that error becomes lost mass, and the
dominated-neighborhood comparison and endpoint-contact structure convert this
into a bound on the one-step loss.

The next corollaries translate Assumption~\ref{ass:regularity} into more
primitive distributional conditions. They are grouped by the source of
nonregularity: first the ratio distribution itself, then the joint distribution
of value and consumption, and finally the shape of the consumption distribution.
Each proof identifies the
corresponding value of $\pp$ and then applies Theorem~\ref{thm:main}; the proofs
are collected in Appendix~\ref{app:examples}.

\begin{corollary}[Independent size and ratio, bounded density]\label{cor:regret-regular}
In the setting of Theorem~\ref{thm:main}, for each type $k$, let $(\beta_k,V_k,R_k)$ have the conditional distribution of $(\beta,V,R)$ given $J=k$. Suppose that $\beta_k$ and $R_k$ are independent, and
that $R_k$ has a density bounded above and below by positive constants on its
single-interval support. Then
\[
        \Reg_T(\SPM;b_T)\le C\,(\log(eT))^2 .
\]
\end{corollary}

\begin{corollary}[Independent size and ratio]\label{cor:regret-ratio}
In the setting of Theorem~\ref{thm:main}, for each type $k$, let $(\beta_k,V_k,R_k)$ have the conditional distribution of $(\beta,V,R)$ given $J=k$. Suppose that $\beta_k$ and $R_k$ are independent.
Suppose also that $R_k$ has single-interval support, is regular in the interior
of that support---meaning $c|I|\le\Prob(R_k\in I)\le C|I|$ for every compact
interval $I$ in the interior---and at each endpoint $r$ satisfies the one-sided
interval bound
\[
        c\,|I|^{\theta_{k,r}}\le \Prob(R_k\in I)\le C\,|I|
\]
for every interval $I$ in a one-sided endpoint neighborhood, with endpoint
intervals of order $x^{\theta_{k,r}}$. With
\[
        \theta=\max_{k,r}\theta_{k,r}>1,
\]
the sample-path marginal policy satisfies
\[
        \Reg_T(\SPM;b_T)\le
        C\,T^{1/2-1/(2\theta)}
        (\log(eT))^{(\theta+1)/(2\theta)+1}+C .
\]
\end{corollary}

When \(\beta_k\equiv1\), the model reduces to deterministic consumption with
continuous rewards: the ratio is simply \(R_k=V_k\), and type \(k\) consumes the
deterministic resource vector \(a^k\). In the regular case \(\pp=1\),
Corollary~\ref{cor:regret-regular} recovers the \(O((\log(eT))^2)\) guarantee of
\citet{JiangMaZhang2025Degeneracy}. Corollary~\ref{cor:regret-ratio} further
extends this comparison beyond densities bounded away from zero. It allows the
reward density, equivalently the ratio density, to vanish polynomially near
active cutoffs, with the regret rate determined by the corresponding active-mass
exponent.

This deterministic-consumption specialization should also be compared with
\citet{BesbesEtAl2024}. Their multisecretary model has a single resource and
unit consumption, whereas the specialization here allows multiple resources and
type-dependent deterministic consumption vectors. In this sense,
Corollary~\ref{cor:regret-ratio} extends the same cutoff-mass perspective to a
multi-resource deterministic-consumption setting.

The two preceding corollaries take the endpoint behavior of the ratio
distribution as primitive. This behavior can also arise from simpler primitives.
When value and consumption are independent, the ratio support can acquire a
corner at an endpoint even if both marginal densities are bounded above and
below. In that case the active weighted-mass exponent is two.

\begin{corollary}[Bounded independent densities]\label{cor:regret-bounded}
In the setting of Theorem~\ref{thm:main}, for each type $k$, let $(\beta_k,V_k,R_k)$ have the conditional distribution of $(\beta,V,R)$ given $J=k$. Suppose that $V_k$ and $\beta_k$ are independent, have
compact supports bounded away from zero, and have densities bounded above and
below by positive constants on their supports. Then
\[
        \Reg_T(\SPM;b_T)\le C\,T^{1/4}(\log(eT))^{7/4} .
\]
\end{corollary}

The exponent two in Corollary~\ref{cor:regret-bounded} is not special. Holding
the value distribution fixed and allowing the consumption density to vanish at an
endpoint gives a continuum of active weighted-mass exponents.

\begin{corollary}[Beta consumption and uniform value]\label{cor:regret-beta}
In the setting of Theorem~\ref{thm:main}, for each type $k$, let $(\beta_k,V_k,R_k)$ have the conditional distribution of $(\beta,V,R)$ given $J=k$. Suppose that $V_k\sim\Unif[v_k^-,v_k^+]$, with
$0<v_k^-<v_k^+$, is independent of
\[
        \beta_k=\betamin_k+(\betamax_k-\betamin_k)Y_k,
\]
where $Y_k\sim\mathrm{Beta}(a_k,b_k)$ with $a_k,b_k>0$. Then, writing
$q=\max_k\{a_k,b_k\}$,
\[
        \Reg_T(\SPM;b_T)\le
        C\,T^{1/2-1/(2(1+q))}
        (\log(eT))^{(q+2)/(2(1+q))+1}+C .
\]
\end{corollary}

Across the four corollaries the single exponent $\pp$ interpolates between the
polylogarithmic regime and the $T^{1/2}$ barrier.

\begin{remark}[Capacity dependence and a local exponent]\label{rem:local}
The exponent \(\pp\) in Theorem~\ref{thm:main} is a worst-case quantity over
capacities: it must control active mass at every cutoff covered by the global
regularity assumption. At a fixed capacity, only a local part of the ratio
support is active, and the active mass may be less sparse than the global
exponent suggests.

This distinction is visible in Example~1, where reward and size are independent
and uniform on \([1,2]\). The global exponent is \(\pp=2\), and the global theorem
gives a \(T^{1/4}\)-order polynomial rate, up to logarithmic factors, at the
critical capacity
\[
        b_T=\frac32T=\E[\beta]\,T .
\]
At any fixed binding capacity \(b_T=cT\) with \(0<c<3/2\), however, the active
cutoff is interior. The capacity-local exponent is then \(\pp_{\mathcal R}=1\),
and \(\SPM\) attains \(O((\log(eT))^2)\) regret, with a constant that may
deteriorate as \(c\) approaches \(0\) or \(\tfrac32\). Thus the polynomial behavior in
Example~1 is a critical-capacity phenomenon: the same primitive distribution has
polylogarithmic regret at every fixed binding non-critical capacity.
Appendix~\ref{app:local} defines the capacity-local exponent and proves the
interior bound in Theorem~\ref{thm:local}.
\end{remark}

\subsection{Proof roadmap}\label{subsec:proof-roadmap}

We give the proof of Theorem~\ref{thm:main} in Section~\ref{sec:spm}. Before
turning to the details, we summarize the four steps of the argument.

\paragraph{Step 1: Regret reduces to one-period Jensen losses.}
The first step is to show that a Bellman comparison along the $\SPM$ trajectory
bounds regret by a sum of one-period Jensen losses:
\[
        \Reg_T
        \le
        C+\sum_s
        \left(
        \E\!\left[
        \E_W\!\left[H(Y_W)\right]
        -
        H\!\left(\E_W[Y_W]\right)
        \right]
        +\frac{C}{s}
        \right).
\]
The harmonic term $\sum_s C/s=O(\log T)$ comes from rounding the offline
solution's decision on the current arrival to a binary value; it is dominated
by the Jensen losses in every regime below.
Here $H(r)=\E[(V-zr)^+\mid J=k,\beta=z]$ is the expected acceptance surplus of a
type-$k$ request of size $z$ at per-unit-size price $r$, convex in $r$. For a
future path $W$ of length $n=s-1$,
\begin{equation}
        Y_W(b,k,z)
        =
        \frac{\OPT_n(b;W)-\OPT_n(b-za^k;W)}{z}
        \label{eq:YW}
\end{equation}
is the per-unit-size drop in the offline value when the future path $W$ is
forced to reserve $za^k$ units of capacity for the current request. Thus the
dynamic regret analysis reduces to bounding the Jensen gap generated by the
random pathwise marginal $Y_W$.

\paragraph{Step 2: The marginal is an average of bid prices.}
The marginal in \eqref{eq:YW} is a capacity difference, and we analyze it by
sweeping continuously between the two capacities. For $\theta\in[0,1]$, set
\[
        g(\theta)=\OPT_n(b-\theta za^k;W).
\]
The function $g$ is concave and Lipschitz, so it is differentiable for a.e.\
$\theta$. At such points, define
\[
        q_{\theta,W}=-\frac{1}{z}g'(\theta).
\]
This is the per-unit-size offline bid price at the intermediate capacity
$b-\theta za^k$. It is also the acceptance cutoff at that capacity: the offline
fractional optimum accepts the request when its value-to-size ratio exceeds
$q_{\theta,W}$. Therefore,
\[
        Y_W
        =
        \frac{g(0)-g(1)}{z}
        =
        \int_0^1 q_{\theta,W}\,d\theta.
\]
This representation is the first of the ideas described in
Section~\ref{subsec:overview}: it makes the price well defined under
degeneracy, while Steps 3--4 below supply the stability estimate that bounds
the resulting loss. Under degeneracy, the dual price jumps with capacity, already under
deterministic consumption, as the marginal accepted type switches. SPM never
commits to a bid price at a single such capacity; it acts on the averaged
marginal $Y_W$, using only the a.e.-defined derivative along the sweep, so the
analysis integrates through the degenerate capacities rather than selecting a
dual price at one of them. Along the sweep, a dual price $p_{\theta,W}$ assigns each type $j$ a cutoff
$q^j_{\theta,W}=a^j\cdot p_{\theta,W}$. The scalar $q_{\theta,W}$ is the
component corresponding to the arriving type. The later stability argument
compares two such cutoff vectors from two independent future paths.

\paragraph{Step 3: Convexity converts cutoff variation into swept active mass.}
The acceptance-surplus function $H$ is convex in the cutoff. Therefore the
Jensen gap from Step 1 can be bounded by comparing two cutoff vectors generated
by two independent future paths. For type $k$, let $q_k$ and $\bar q_k$ be the
two empirical cutoffs. Only ratios between these two cutoffs can contribute to
the loss. The relevant active interval is
\[
        I_k^a(q,\bar q)
        =
        [q_k\wedge\bar q_k,\ q_k\vee\bar q_k]\cap S_k .
\]
Let $\ell_k(q,\bar q)$ be the length of this interval, and let
$M_k(q,\bar q)$ be its weighted ratio mass. The one-period loss is bounded by an
active-mass product of the form
\[
        \sum_k \ell_k(q,\bar q)M_k(q,\bar q).
\]
The length measures how far the two empirical cutoffs move. The mass measures
how much type-$k$ demand lies where the two cutoffs disagree. The probabilistic
Jensen loss is therefore reduced to a deterministic stability question: how
large can this active-mass product be for empirical cutoffs generated by two
nearby future paths?

\paragraph{Step 4: Stability controls the swept active mass.}
It remains to bound the active-mass product from Step 3. The two cutoff vectors
come from two independent future paths. Their empirical type masses and
empirical capacities concentrate around the same population quantities. The
deterministic stability argument then shows that, when these empirical inputs
are close, the projected cutoffs cannot sweep much active weighted mass:
\[
        \sum_k \ell_k(q,\bar q)M_k(q,\bar q)
        \le
        C\bigl(\tau+\eps^{1+1/\pp}\bigr),
\]
where $\eps$ is the empirical type-mass error and $\tau$ is the empirical
capacity error. The exponent $\pp$ enters through the active-mass condition
\eqref{eq:wr}. It determines how much weighted mass an interval must contain
relative to its length, and therefore determines the scale
$\eps^{1+1/\pp}$.

Finally, concentration gives per-stage errors of order $\delta_s=\sqrt{\log(es)/s}$,
and summing the resulting per-stage bounds over $s$ yields the rates of
Theorem~\ref{thm:main}: polylogarithmic when $\pp=1$, and of order
$T^{1/2-1/(2\pp)}$ (up to logarithmic factors) when $\pp>1$.

\section{Proof of the Main Regret Bound}\label{sec:spm}

We now turn the roadmap of Section~\ref{subsec:proof-roadmap} into the proof of
Theorem~\ref{thm:main}.
Assumption~\ref{ass:regularity} is in force throughout this section. The
constant $C$ may change from line to line, but it depends only on the primitive
constants, the direction matrix, and the regularity constants. It never depends
on $T$ or on the initial capacity.

The analysis works directly with the expected hindsight value $\Phi_n$. The only
ordering principle needed is the finite-path fractional-knapsack fact that,
within each type, an optimal fractional allocation accepts requests in decreasing
order of the value-to-size ratio $R=V/\beta$.

We first record the Bellman reduction. The expected acceptance surplus of a
type-$k$ request of size $z$ at per-unit-size cutoff $r$ is
$H_{k,z}(r)=\E[(V-zr)^+\mid J=k,\beta=z]$, which is convex in $r$ with curvature
measure $\Lambda_{k,z}$; see Section~\ref{subsec:assumption}. Recall the
future-path marginal $Y_W$ from \eqref{eq:YW}.

For $s\ge3$, capacity $b$, and an independent future path $W$ of length $s-1$,
define
\[
        \Xi_s(b)
        =
        \E_{J,\beta}\!\Big[
        \big(\E_W\!\left[H_{J,\beta}(Y_W)\right]
        -H_{J,\beta}\!\left(\E_W\!\left[Y_W\right]\right)\big)
        \1\{\beta a^J\le b\}\Big],
\]
where $Y_W=Y_W(b,J,\beta)$. The quantity $\Xi_s(b)$ is the one-period Jensen
loss at capacity $b$ when $s$ periods remain. For $s\le2$ we set $\Xi_s\equiv0$:
the future path of length $s-1\le1$ is too short for the per-stage analysis, and
the two corresponding one-step gaps are $O(1)$ and absorbed into the constant of
\eqref{eq:telescope}.

The next proposition combines the one-step Bellman comparison with telescoping
along the $\SPM$ trajectory. The argument follows the Bellman comparison in
\citet[Sec.~3.1]{JiangMaZhang2025Degeneracy}; related reductions from regret to
per-stage loss appear in \citet{VeraBanerjee2021}, \citet{Bray2024}, and
\citet{BesbesEtAl2024}. We give
the details in Appendix~\ref{app:telescope}.

\begin{proposition}[Bellman comparison and telescoping]\label{prop:telescope}
Under $\SPM$,
\begin{equation}
        \Reg_T(\SPM;b_T)\ \le\ C+\sum_{s=1}^T
        \left(\E[\Xi_s(B_s)]+\frac{C}{s}\right),
        \label{eq:telescope}
\end{equation}
where $B_s$ is the remaining capacity when $s$ periods remain.
\end{proposition}

The remainder of this section is to bound $\Xi_s(b)$ uniformly over feasible
capacities $b$.

\subsection{The per-stage loss is an active-mass product}\label{subsec:sweep}

This subsection starts from the per-stage Jensen loss $\Xi_s(b)$ and reduces it
to a product of two active quantities: the length of a swept ratio interval and
the weighted ratio mass inside that interval. The first step is to replace the
future-path marginal $Y_W$ by finite-path cutoffs.

Fix a future path $W$ of length $n$ and a feasible arriving request of type $k$
and size $z$, that is, with $za^k\le b$. For $\theta\in[0,1]$, set
\[
        b_\theta=b-\theta z a^k,
        \qquad
        \rho_\theta=\frac{b_\theta}{n}.
\]
Thus $b_0=b$ and $b_1=b-za^k$. Define
\[
        g(\theta)=\OPT_n(b_\theta;W).
\]
The function $g$ is concave and Lipschitz, and hence differentiable for a.e.\
$\theta$.

For the future path $W=(J_i,\beta_i,V_i)_{i=1}^n$, write $R_i=V_i/\beta_i$. We
work on the probability-one event that each realized ratio $R_i$ lies in the
support $S_{J_i}=[r_{J_i}^-,r_{J_i}^+]$ of its type; this holds almost surely,
since $\mu_\ell(S_\ell^c)=0$ and $\beta\ge\betamin_\ell>0$ give
$\Prob(R\in S_\ell\mid J=\ell)=1$. The
empirical tail functions are normalized by the future-path length $n$:
\begin{equation}
        \widehat C_{k,W}(r)
        =
        \frac1n\sum_{i=1}^n
        \beta_i\1\{J_i=k,\ R_i\ge r\},
        \qquad
        \widehat C_{k,W}(r+)
        =
        \frac1n\sum_{i=1}^n
        \beta_i\1\{J_i=k,\ R_i>r\}.
        \label{eq:emptail}
\end{equation}
Let
\[
        \widehat m_{k,W}=\widehat C_{k,W}(0),
        \qquad k=1,\ldots,K.
\]
For $m\in\mathbb R_+^K$ and $\rho\in\mathbb R_+^d$, define the finite-path fluid
feasible set
\begin{equation}
        K_m(\rho)=\{u\in\mathbb R_+^K:0\le u\le m,\ Au\le\rho\}.
        \label{eq:feasible-set}
\end{equation}

The next lemma extracts bounded scalar cutoffs from the finite-path linear
program. These cutoffs encode feasibility through the empirical tails and
optimality through a projected normal inequality. The proof is given in
Appendix~\ref{app:cutoff-details}.

\begin{lemma}[Bounded cutoff selection]\label{lem:bounded-cutoff-selection}
For almost every $\theta\in[0,1]$, there exist $u_{\theta,W}\in\R_+^K$ and
$q_{\theta,W}\in[0,y]^K$ such that
\begin{flalign}
&\textnormal{(i)}\quad q_{\theta,W,k}=-\frac{g'(\theta)}{z},
        &&\label{eq:selection-derivative}\\
&\textnormal{(ii)}\quad
        \widehat C_{\ell,W}(q_{\theta,W,\ell}+)
        \le u_{\theta,W,\ell}\le
        \widehat C_{\ell,W}(q_{\theta,W,\ell}),
        \qquad \ell=1,\ldots,K,
        &&\label{eq:selection-tail}\\
&\textnormal{(iii)}\quad
        u_{\theta,W}\in K_{\widehat m_W}(\rho_\theta),
        &&\\
&\textnormal{(iv)}\quad
        \sum_{\ell=1}^K
        \Proj_\ell(q_{\theta,W,\ell})
        (u'_{\ell}-u_{\theta,W,\ell})\le 0,
        \qquad u'\in K_{\widehat m_W}(\rho_\theta).
        &&\label{eq:selection-normal-ineq}
\end{flalign}
\end{lemma}

The next lemma is what makes the scalar cutoffs sufficient. It shows that the
Jensen gap for the future-path marginal is no larger than the Jensen gap for the
selected cutoff coordinate along the sweep. After this reduction, the proof only
has to control empirical cutoffs.

\begin{lemma}[Marginal-to-cutoff reduction]\label{lem:marginal-cutoff-reduction-new}
Fix a feasible current request of type $k$ and size $z$, a capacity $b$, and a
future horizon $n$. Let $W$ be an independent future path of length $n$, and let
$q_{\theta,W,k}$ be the cutoff coordinate selected in
Lemma~\ref{lem:bounded-cutoff-selection} for the sweep from $b$ to $b-za^k$.
Then
\begin{equation}
        \E_W\!\left[H_{k,z}(Y_W)\right]
        -H_{k,z}\!\left(\E_W[Y_W]\right)
        \le
        \E_{W,\tilde\theta}\!\left[H_{k,z}(q_{\tilde\theta,W,k})\right]
        -H_{k,z}\!\left(\E_{W,\tilde\theta}[q_{\tilde\theta,W,k}]\right),
        \label{eq:marginal-jensen-reduction-new}
\end{equation}
where $Y_W=Y_W(b,k,z)$ and $\tilde\theta\sim {\rm Unif}[0,1]$ is independent of
$W$.
\end{lemma}

\begin{proof}
Fix a future path $W$. Define
\[
        g(\theta)=\OPT_n(b-\theta za^k;W),\qquad 0\le\theta\le1.
\]
The function $g$ is concave and Lipschitz, and hence it is absolutely continuous
and differentiable for a.e.\ $\theta$. At differentiability points, set
\[
        q_{\theta,W,k}=-\frac{g'(\theta)}{z}.
\]
By Lemma~\ref{lem:bounded-cutoff-selection}, this representative agrees a.e.\
with the selected cutoff coordinate. Absolute continuity gives
\[
        Y_W(b,k,z)
        =
        \frac{g(0)-g(1)}{z}
        =
        -\frac1z\int_0^1 g'(\theta)\,d\theta
        =
        \int_0^1 q_{\theta,W,k}\,d\theta .
\]
Equivalently,
\[
        Y_W(b,k,z)=
        \E_{\tilde\theta}\!\left[q_{\tilde\theta,W,k}\mid W\right].
\]
Taking expectation over $W$ yields
\[
        \E_W[Y_W]
        =
        \E_{W,\tilde\theta}[q_{\tilde\theta,W,k}].
\]
Since $H_{k,z}$ is convex, Jensen's inequality gives
\[
        H_{k,z}(Y_W)
        \le
        \E_{\tilde\theta}\!\left[
        H_{k,z}(q_{\tilde\theta,W,k})\mid W
        \right].
\]
Taking expectation over $W$ and using the preceding identity for $\E_W[Y_W]$
proves \eqref{eq:marginal-jensen-reduction-new}.
\end{proof}

For scalar cutoffs $q,\bar q\in[0,y]$, define the active swept interval
\[
        I_k^a(q,\bar q)=[q\wedge\bar q,q\vee\bar q]\cap S_k,
\]
together with its length and weighted mass
\[
        \ell_k(q,\bar q)=\Leb(I_k^a(q,\bar q)),
        \qquad
        M_k(q,\bar q)=\mu_k(I_k^a(q,\bar q)).
\]
This is the two-cutoff form of the single-interval length
$\ell_k(I)=\Leb(I\cap S_k)$ of Assumption~\ref{ass:regularity}, applied to the
interval $I_k^a(q,\bar q)$.

\paragraph{Jensen under a pairwise active cap.}
The right-hand side of \eqref{eq:marginal-jensen-reduction-new} is a Jensen gap
for the convex function $H_{k,z}$. We use the following general bound. Let $h$
be a convex function on $[0,y]$, and let $\mu_h$ be the measure generated by the
right derivative of $h$. Thus, for $B=[B_-,B_+]\subset[0,y]$,
\[
        \mu_h(B)=h'_+(B_+)-h'_+(B_-).
\]
In the applications below, $\mu_h$ is finite, atomless, and supported on the
active ratio support $S_k$:
\[
        \mu_h([0,y]\setminus S_k)=0 .
\]

We first record a general convex-analysis bound: the Jensen gap of a convex
function is controlled by the length and the curvature mass of the projected hull
of the cutoff set.

\begin{lemma}[Projected-hull Jensen gap]\label{lem:hull-jensen}
Let $J=[a,b]$ be a compact interval with $|J|>0$, and let $\nu$ be a finite,
atomless Borel measure supported on $J$. Let $h$ be convex with curvature measure
$\nu$, so that
\[
        h(x)=\alpha+\beta x+\int_J(x-t)^+\,\nu(\dd t)
\]
for some $\alpha,\beta\in\mathbb R$. Let $\mathcal Q$ be nonempty with
$\mathcal Q\cup J$ contained in an interval of length $D$, let $Q$ be a random
variable taking values in $\mathcal Q$, and let $\Pi_J$ be the projection onto
$J$. With
\[
        u=\inf\{\Pi_J(q):q\in\mathcal Q\},
        \qquad
        v=\sup\{\Pi_J(q):q\in\mathcal Q\},
\]
so that $[u,v]\subseteq J$,
\[
        \E h(Q)-h(\E Q)
        \le
        \left(1+\frac{D}{|J|}\right)(v-u)\,\nu([u,v]).
\]
\end{lemma}

\begin{proof}
By the curvature representation,
\[
        \E h(Q)-h(\E Q)
        =
        \int_J G_Q(t)\,\nu(\dd t),
        \qquad
        G_Q(t)=\E(Q-t)^+-(\E Q-t)^+,
\]
and $G_Q\ge0$ by convexity of $x\mapsto(x-t)^+$. If $t\in J$ and $t<u$, then
$\Pi_J(Q)\ge u>t$ almost surely; since $t\in J$, this forces $Q>t$ almost surely,
so $G_Q(t)=0$. Likewise $t>v$ gives $Q<t$ almost surely and $G_Q(t)=0$. Hence
\[
        \E h(Q)-h(\E Q)=\int_{[u,v]}G_Q(t)\,\nu(\dd t).
\]
Fix $t\in[u,v]$. If $\{\Pi_J(q):q\in\mathcal Q\}$ does not contain both endpoints
of $J$, then at most one clamp side of $J$ is used. If the upper side is unused,
then $(q-t)^+\le v-t\le v-u$ for every $q\in\mathcal Q$, so
$G_Q(t)\le\E(Q-t)^+\le v-u$; if instead the lower side is unused, then $Q\ge u$
almost surely and
\[
        G_Q(t)\le\E(Q-t)^+-(\E Q-t)=\E(t-Q)^+\le v-u .
\]
If both endpoints of $J$ are attained, then $[u,v]=J$, and since
$\mathcal Q\cup J$ lies in an interval of length $D$ we have $(Q-t)^+\le D$, so
$G_Q(t)\le\E(Q-t)^+\le D=(D/|J|)(v-u)$. In every case
$G_Q(t)\le(1+D/|J|)(v-u)$, and therefore
\[
        \E h(Q)-h(\E Q)
        \le
        \left(1+\frac{D}{|J|}\right)(v-u)\,\nu([u,v]).
\]
\end{proof}

Specializing Lemma~\ref{lem:hull-jensen} to the active ratio support converts a
pairwise active cap on the swept curvature into a bound on the Jensen gap.

\begin{lemma}[Jensen bound under a pairwise active cap]\label{lem:jensen}
Fix a type $k$. Let $h$ be convex on $[0,y]$, and suppose that the measure
$\mu_h$ defined above is finite, atomless, and supported on $S_k$. Let
${\cal Q}\subset[0,y]$ be nonempty, and let $Q$ be a random variable taking
values in ${\cal Q}$. Then
\[
        \E h(Q)-h(\E Q)
        \le
        \left(1+\frac{y}{|S_k|}\right)
        \sup_{q,\bar q\in{\cal Q}}
        \ell_k(q,\bar q)\,\mu_h\bigl(I_k^a(q,\bar q)\bigr).
\]
\end{lemma}

\begin{proof}
The curvature measure $\mu_h$ is finite, atomless, and supported on $S_k$, so
Lemma~\ref{lem:hull-jensen} applies with $J=S_k$, $\nu=\mu_h$, and ambient length
$D=y$ (since ${\cal Q}\cup S_k\subseteq[0,y]$). With
$u=\inf\{\Pi_k(q):q\in{\cal Q}\}$ and $v=\sup\{\Pi_k(q):q\in{\cal Q}\}$, it gives
\[
        \E h(Q)-h(\E Q)
        \le
        \left(1+\frac{y}{|S_k|}\right)(v-u)\,\mu_h([u,v]).
\]
By the active-hull closure lemma, Lemma~\ref{lem:hullcap}, applied with $B=S_k$
and $\nu=\mu_h$,
\[
        (v-u)\,\mu_h([u,v])
        \le
        \sup_{q,\bar q\in{\cal Q}}
        \ell_k(q,\bar q)\,\mu_h\bigl(I_k^a(q,\bar q)\bigr).
\]
Combining this with the previous display proves the lemma.
\end{proof}

\paragraph{Remainder of the proof.}
Lemma~\ref{lem:jensen} reduces the per-stage loss to a pairwise active cap. For
good pathwise cutoffs $q,\bar q$, we need to show that
\[
        \ell_k(q,\bar q)\,\Lambda_{k,z}(I_k^a(q,\bar q))
        \le
        C r .
\]
On a dominated neighborhood, Assumption~\ref{ass:regularity} gives
$\Lambda_{k,z}(I)\le C\mu_k(I)$ for every active interval $I$ contained in that
neighborhood and for almost every size $z$. Thus, in the dominated case, it is enough to bound
\[
        \ell_k(q,\bar q)\,\mu_k(I_k^a(q,\bar q))
\]
uniformly over good cutoff pairs. The next subsection proves this deterministic
active-mass bound. At a contact endpoint, domination by $\mu_k$ may fail; there
we use the endpoint Hardy estimate instead.

\subsection{Projected active-mass product stability}

We now prove the deterministic stability estimate that controls the active-mass
product swept by two empirical cutoff vectors. The estimate is deterministic:
concentration will later supply the required closeness of the empirical inputs.
Recall the finite-path fluid feasible set $K_m(\rho)$ from \eqref{eq:feasible-set};
it is the feasible region of the type-level finite-path fluid problem with type
masses $m$ and normalized capacity $\rho$.

The cutoffs enter only through their projected values on the active supports.
Recall from Section~\ref{subsec:sweep} the active interval $I_k^a(q,\bar q)$, its
length $\ell_k(q,\bar q)$, and its weighted mass $M_k(q,\bar q)$; the product
$\ell_k(q,\bar q)M_k(q,\bar q)$ is the active-mass product for type $k$.

We also use the normal cone notation
\[
        N_K(x)=\{z:z^\top(v-x)\le0\ \text{for all }v\in K\},
        \qquad x\in K .
\]
For each type $k$, define the population tail masses
\begin{equation*}
        C_k(r)=\pi_k\E[\beta\1\{R\ge r\}\mid J=k],
        \qquad
        C_k(r+)=\pi_k\E[\beta\1\{R>r\}\mid J=k],
\end{equation*}
and $m_k^0=C_k(0)=\pi_k\E[\beta\mid J=k]$, with $m^0=(m_1^0,\ldots,m_K^0)$.

Proposition~\ref{prop:active} shows that if two feasible type-level solutions
have nearby empirical tail data, nearby capacities, and compatible projected
normal directions, then their active-mass product is small.

The next proposition isolates the part of the active-mass stability argument that
does not use a lower mass exponent.

\begin{proposition}[Pre-Young active-mass stability]\label{prop:active-linear}
Assume that each active support \(S_k\) is a single interval and that each
measure \(\mu_k\) is atomless. Fix compact ranges for the tail vector \(m\) and
the capacity vector \(\rho\) on which Lemma~\ref{lem:projcomp} applies. For
every finite constant \(C_m\), there is a constant \(C<\infty\) such that the
following statement holds.

Let \(\widetilde u\in K_{\widetilde m}(\rho)\) and
\(\bar u\in K_{\bar m}(\bar\rho)\), where
\(\widetilde m,\bar m,\rho,\bar\rho\) lie in the fixed compact ranges. Let
\(\widetilde q,\bar q\in[0,y]^K\). Suppose that \(\eps>0\) and that, for
\(k=1,\ldots,K\),
\[
        C_k(\widetilde q_k+)-\eps
        \le \widetilde u_k
        \le C_k(\widetilde q_k)+\eps,
        \qquad
        C_k(\bar q_k+)-\eps
        \le \bar u_k
        \le C_k(\bar q_k)+\eps .
\]
Suppose also that
\[
        \Pi(\widetilde q)\in N_{K_{\widetilde m}(\rho)}(\widetilde u),
        \qquad
        \Pi(\bar q)\in N_{K_{\bar m}(\bar\rho)}(\bar u),
\]
and that
\[
        \|\widetilde m-\bar m\|_\infty\le C_m\eps,
        \qquad
        \|\rho-\bar\rho\|_\infty\le\tau .
\]
Then
\[
        \sum_{k=1}^K
        \ell_k(\widetilde q_k,\bar q_k)
        M_k(\widetilde q_k,\bar q_k)
        \le
        C\Bigl(\tau+\eps\sum_{k=1}^K
        \ell_k(\widetilde q_k,\bar q_k)\Bigr),
\]
where
\[
        \ell_k(a,b):=\Leb\bigl(I_k^a(a,b)\bigr),
        \qquad
        M_k(a,b):=\mu_k\bigl(I_k^a(a,b)\bigr).
\]
\end{proposition}

\begin{proof}
For each type \(k\), write
\[
        I_k:=I^a_k(\widetilde q_k,\bar q_k)
        =
        [\widetilde q_k\wedge\bar q_k,
         \widetilde q_k\vee\bar q_k]\cap S_k,
\]
and set
\[
        \ell_k:=\Leb(I_k),
        \qquad
        M_k:=\mu_k(I_k).
\]
Let
\[
        w:=\Pi(\widetilde q),
        \qquad
        \bar w:=\Pi(\bar q),
\]
so that \(w_k=\Pi_k(\widetilde q_k)\) and
\(\bar w_k=\Pi_k(\bar q_k)\). Since \(S_k\) is a single interval and
\(\Pi_k\) is projection onto \(S_k\), the active interval length is exactly the
distance between the projected cutoffs:
\begin{equation}
        \ell_k
        =
        |\Pi_k(\widetilde q_k)-\Pi_k(\bar q_k)|
        =
        |w_k-\bar w_k|.
        \label{eq:active-stab-proj-length}
\end{equation}

We first record the tail identities used below. Since \(\mu_k\) is atomless,
\(\mu_k(\{a\})=0\) for every point \(a\). By definition of the closed and strict
tails,
\[
        C_k(a)=\mu_k([a,\infty)),
        \qquad
        C_k(a{+})=\mu_k((a,\infty)).
\]
Thus, if \(a>b\), then
\begin{equation}
        C_k(a)-C_k(b{+})
        =
        -\mu_k([b,a]\cap S_k).
        \label{eq:active-stab-tail-desc}
\end{equation}
Similarly, if \(a<b\), then
\begin{equation}
        C_k(a{+})-C_k(b)
        =
        \mu_k([a,b]\cap S_k).
        \label{eq:active-stab-tail-asc}
\end{equation}

We now prove the coordinatewise active-mass product estimate. Fix a type \(k\).
If \(w_k=\bar w_k\), then \(\ell_k=0\) by
\eqref{eq:active-stab-proj-length}, and hence
\[
        (w_k-\bar w_k)(\widetilde u_k-\bar u_k)
        =
        0
        =
        -\ell_kM_k+2\eps\ell_k.
\]

Suppose next that \(w_k>\bar w_k\). Since \(\Pi_k\) is monotone, this inequality
implies that \(\widetilde q_k>\bar q_k\). The graph relations in the proposition
give
\[
\begin{aligned}
        \widetilde u_k-\bar u_k
        &\le
        \bigl(C_k(\widetilde q_k)+\eps\bigr)
        -
        \bigl(C_k(\bar q_k{+})-\eps\bigr)  \\
        &=
        C_k(\widetilde q_k)-C_k(\bar q_k{+})+2\eps .
\end{aligned}
\]
Applying \eqref{eq:active-stab-tail-desc} with
\(a=\widetilde q_k\) and \(b=\bar q_k\) gives
\[
        C_k(\widetilde q_k)-C_k(\bar q_k{+})=-M_k.
\]
Therefore
\[
        \widetilde u_k-\bar u_k\le -M_k+2\eps .
\]
Since \(w_k-\bar w_k=\ell_k>0\), we obtain
\[
        (w_k-\bar w_k)(\widetilde u_k-\bar u_k)
        \le
        -\ell_kM_k+2\eps\ell_k .
\]

Finally, suppose that \(w_k<\bar w_k\). Monotonicity of \(\Pi_k\) gives
\(\widetilde q_k<\bar q_k\). The graph relations now imply
\[
\begin{aligned}
        \widetilde u_k-\bar u_k
        &\ge
        \bigl(C_k(\widetilde q_k{+})-\eps\bigr)
        -
        \bigl(C_k(\bar q_k)+\eps\bigr)  \\
        &=
        C_k(\widetilde q_k{+})-C_k(\bar q_k)-2\eps .
\end{aligned}
\]
Applying \eqref{eq:active-stab-tail-asc} with
\(a=\widetilde q_k\) and \(b=\bar q_k\) gives
\[
        C_k(\widetilde q_k{+})-C_k(\bar q_k)=M_k.
\]
Hence
\[
        \widetilde u_k-\bar u_k\ge M_k-2\eps .
\]
Since \(w_k-\bar w_k<0\),
\[
\begin{aligned}
        (w_k-\bar w_k)(\widetilde u_k-\bar u_k)
        &\le
        (w_k-\bar w_k)(M_k-2\eps)  \\
        &=
        -|w_k-\bar w_k|M_k+2\eps |w_k-\bar w_k|  \\
        &=
        -\ell_kM_k+2\eps\ell_k .
\end{aligned}
\]

The three cases show that, for every \(k\),
\[
        (w_k-\bar w_k)(\widetilde u_k-\bar u_k)
        \le
        -\ell_kM_k+2\eps\ell_k.
\]
Summing over \(k\) yields
\begin{equation}
        (w-\bar w)^\top(\widetilde u-\bar u)
        \le
        -\sum_{k=1}^K\ell_kM_k
        +
        2\eps\sum_{k=1}^K\ell_k.
        \label{eq:active-stab-upper-proj}
\end{equation}

The same inner product is bounded from below by the projected-comparison lemma.
Indeed, Lemma~\ref{lem:projcomp} applies by hypothesis, so there is a constant
\(C_{\rm pc}<\infty\) such that
\begin{equation}
        (w-\bar w)^\top(\widetilde u-\bar u)
        \ge
        -C_{\rm pc}\tau
        -
        C_{\rm pc}\eps\sum_{k=1}^K\ell_k.
        \label{eq:active-stab-lower-proj}
\end{equation}
Combining \eqref{eq:active-stab-upper-proj} and
\eqref{eq:active-stab-lower-proj}, and increasing the constant if necessary,
gives
\[
        \sum_{k=1}^K\ell_kM_k
        \le
        C_1\tau
        +
        C_1\eps\sum_{k=1}^K\ell_k.
\]
This is the asserted bound.
\end{proof}

The pre-Young bound of Proposition~\ref{prop:active-linear} controls the active
product \(\sum_k\ell_kM_k\) up to the linear active-length term
\(\eps\sum_k\ell_k\). The next lemma absorbs that term into the product itself
whenever active mass has a polynomial lower bound. This is the only step in this
argument where the exponent enters.

\begin{lemma}[Local active-mass absorption]\label{lem:local-absorption}
Let \(p\ge1\) and \(c>0\). Let \((\ell_k)_{k=1}^K\) and
\((M_k)_{k=1}^K\) be nonnegative real numbers such that
\[
        M_k\ge c\,\ell_k^{p},
        \qquad k=1,\ldots,K.
\]
If, for some \(\tau\ge0\), \(\eps>0\), and \(C_1<\infty\),
\[
        \sum_{k=1}^K\ell_kM_k
        \le
        C_1\tau
        +
        C_1\eps\sum_{k=1}^K\ell_k,
\]
then there is a constant \(C<\infty\), depending only on \(C_1,c,p\), and \(K\),
such that
\[
        \sum_{k=1}^K\ell_kM_k
        \le
        C\bigl(\tau+\eps^{1+1/p}\bigr).
\]
\end{lemma}

\begin{proof}
The lower bound on \(M_k\) gives
\[
        \ell_kM_k\ge c\,\ell_k^{p+1}.
\]
Young's inequality with conjugate exponents \(p+1\) and \((p+1)/p\) gives, for
every \(\ell\ge0\),
\[
        C_1\eps\,\ell
        \le
        \frac{c}{2}\ell^{p+1}
        +
        C_2\eps^{1+1/p},
\]
where \(C_2\) depends only on \(C_1,c\), and \(p\). Applying this inequality with
\(\ell=\ell_k\), and using \(c\,\ell_k^{p+1}\le\ell_kM_k\), gives
\[
        C_1\eps\,\ell_k
        \le
        \frac12\,\ell_kM_k
        +
        C_2\eps^{1+1/p}.
\]
After summing over the \(K\) types, we obtain
\[
        C_1\eps\sum_{k=1}^K\ell_k
        \le
        \frac12\sum_{k=1}^K\ell_kM_k
        +
        C_3\eps^{1+1/p},
        \qquad C_3:=KC_2 .
\]
Substituting this estimate into the hypothesis yields
\[
        \sum_{k=1}^K\ell_kM_k
        \le
        C_1\tau
        +
        \frac12\sum_{k=1}^K\ell_kM_k
        +
        C_3\eps^{1+1/p}.
\]
Moving the half-product term to the left and increasing the constant proves the
claim.
\end{proof}

\begin{proposition}[Active-mass product stability]\label{prop:active}
Assume the active mass bound \eqref{eq:wr}, and assume that each active support
set $S_k$ is a single interval. Fix compact ranges for $m$ and $\rho$ on which
Lemma~\ref{lem:projcomp} applies. For every $C_m<\infty$, there are constants
$\eps_0,C<\infty$ such that the following holds.

Let $\widetilde u\in K_{\widetilde m}(\rho)$ and
$\bar u\in K_{\bar m}(\bar\rho)$, where
$\widetilde m,\bar m,\rho,\bar\rho$ lie in the fixed compact ranges. Let
$\widetilde q,\bar q\in[0,y]^K$. Suppose that $0<\eps\le\eps_0$,
\[
        C_k(\widetilde q_k+)-\eps
        \le \widetilde u_k
        \le C_k(\widetilde q_k)+\eps,\qquad
        C_k(\bar q_k+)-\eps
        \le \bar u_k
        \le C_k(\bar q_k)+\eps,
        \qquad k=1,\ldots,K,
\]
that
\[
        \Pi(\widetilde q)\in N_{K_{\widetilde m}(\rho)}(\widetilde u),
        \qquad
        \Pi(\bar q)\in N_{K_{\bar m}(\bar\rho)}(\bar u),
\]
and that
\[
        \|\widetilde m-\bar m\|_\infty\le C_m\eps,
        \qquad
        \|\rho-\bar\rho\|_\infty\le\tau .
\]
Then
\[
        \sum_{k=1}^K
        \ell_k(\widetilde q_k,\bar q_k)
        M_k(\widetilde q_k,\bar q_k)
        \le
        C\bigl(\tau+\eps^{1+1/\pp}\bigr),
\]
where $M_k(a,b):=\mu_k(I_k^a(a,b))$.
\end{proposition}

\begin{proof}
The upper bound in \eqref{eq:wr} implies that each $\mu_k$ is atomless, since
$\mu_k(\{a\})\le C\,\Leb(\{a\}\cap S_k)=0$ for every point $a$, and the active
supports are single intervals. Proposition~\ref{prop:active-linear} therefore
applies and gives
\[
        \sum_{k=1}^K\ell_kM_k
        \le
        C_1\tau+C_1\eps\sum_{k=1}^K\ell_k ,
\]
where $\ell_k=\ell_k(\widetilde q_k,\bar q_k)$,
$M_k=M_k(\widetilde q_k,\bar q_k)=\mu_k(I_k)$, and
$I_k=I_k^a(\widetilde q_k,\bar q_k)$. By the lower bound in \eqref{eq:wr} applied
to $I_k$,
\[
        M_k=\mu_k(I_k)\ge c\,\ell_k^{\pp}.
\]
Lemma~\ref{lem:local-absorption}, applied with $p=\pp$, then yields
\[
        \sum_{k=1}^K\ell_kM_k
        \le
        C\bigl(\tau+\eps^{1+1/\pp}\bigr).
\]
The estimate holds for every $\eps>0$; the smallness $\eps\le\eps_0$ in the
statement is not used.
\end{proof}

This is the only active-mass product stability estimate used below. It controls
the active-mass product $\sum_k \ell_k M_k$ swept by the projected cutoffs. The
projection discards inactive flat parts of the sweep, which carry neither
weighted-resource mass nor conditional curvature.

\subsection{Concentration, summation, and the bound at \texorpdfstring{$\mathfrak
p=1$}{p=1}}\label{subsec:sum}

It remains to combine the pathwise active-mass product cap of
Proposition~\ref{prop:active} with the Jensen bound in Lemma~\ref{lem:jensen},
and then to sum the resulting per-stage losses. The concentration step below compares the
empirical analogues of these tails across two finite future paths to their
population values.

\paragraph{Normalization reminder.}
Recall from \eqref{eq:emptail} that, for a future path
$W=(J_i,\beta_i,V_i)_{i=1}^n$, the empirical tail systems are normalized by the
future-path length $n$:
\[
        \widehat C_{k,W}(r)
        =
        \frac1n\sum_{i=1}^n
        \beta_i\1\{J_i=k,\ R_i\ge r\},
        \qquad
        \widehat C_{k,W}(r+)
        =
        \frac1n\sum_{i=1}^n
        \beta_i\1\{J_i=k,\ R_i>r\},
\]
where $R_i=V_i/\beta_i$. Thus
\[
        \widehat m_{k,W}:=\widehat C_{k,W}(0)
\]
is the normalized empirical type total. Likewise, along the capacity path
$b_\theta=b-\theta za^k$, the normalized right-hand side used in the deterministic
systems is
\[
        \rho_\theta:=\frac{b_\theta}{n}.
\]
The finite-path primal allocation selected in
Lemma~\ref{lem:bounded-cutoff-selection} is divided by $n$, so it lies in
$K_{\widehat m_W}(\rho_\theta)$, or equivalently in the clipped system
$K_{\widehat m_W}(\rho_\theta^{\rm eff})$ introduced below.

\paragraph{Concentration.}
The empirical closed and strict tails are generated by the weighted-threshold
classes
\[
        (J,\beta,V)\mapsto \beta\1\{J=k,\ V/\beta\ge r\},
        \qquad
        (J,\beta,V)\mapsto \beta\1\{J=k,\ V/\beta> r\}.
\]
These are bounded VC-subgraph classes with envelope $\beta$. Hence, after
increasing constants if necessary, there is an event $E_n(W)$ such that
$\Prob(E_n)\ge1-n^{-6}$ and, on $E_n(W)$,
\begin{equation}
        \max_k\sup_{r\in[0,y]}
        \max\left\{
        |\widehat C_{k,W}(r)-C_k(r)|,\,
        |\widehat C_{k,W}(r+)-C_k(r+)|
        \right\}
        \le \delta_n,
        \qquad
        \delta_n:=C\sqrt{\frac{\log(en)}{n}} .
        \label{eq:tailconc}
\end{equation}
Taking $r=0$ in \eqref{eq:tailconc} gives
\begin{equation}
        \|\widehat m_W-m^0\|_\infty\le\delta_n .
        \label{eq:mconc}
\end{equation}
On $E_n(W)$, every empirical closed/right cutoff relation is a population
closed/right relation with perturbation $\delta_n$. Indeed, if
\[
        \widehat C_{k,W}(q_k+)\le u_k\le \widehat C_{k,W}(q_k),
\]
then \eqref{eq:tailconc} implies
\begin{equation}
        C_k(q_k+)-\delta_n
        \le u_k
        \le C_k(q_k)+\delta_n .
        \label{eq:tailrelax}
\end{equation}
The upper bound in \eqref{eq:wr} implies that $\mu_k$ is atomless, since
$\mu_k(\{r\})\le C\Leb(\{r\}\cap S_k)=0$. Thus
$C_k(q_k+)=C_k(q_k)$, although the two-sided form \eqref{eq:tailrelax} is the
form used below.

\paragraph{Global domination when $\pp=1$.}
When $\pp=1$, Assumption~\ref{ass:regularity} requires the finite
active cover of each $S_k$ to consist only of dominated neighborhoods. Thus,
for each cover element $U$, for a.e. feasible size $z$,
\[
        \Lambda_{k,z}(I)
        \le
        C_U\mu_k(I)
        \qquad
        \text{for every interval }I\subset U.
\]
Since the cover is finite, we may intersect the corresponding full-measure sets
of $z$'s and work on one full-measure set where all local domination bounds
hold simultaneously.

Choose a finite partition
\[
        r_k^-=a_{k,0}<a_{k,1}<\cdots<a_{k,N_k}=r_k^+
\]
such that every cell $[a_{k,h-1},a_{k,h}]$ is contained in one dominated
neighborhood. This is possible by compactness of $S_k$ and the finite active
cover. If $I\subseteq S_k$ is an interval, then, up to endpoints,
\[
        I
        =
        \bigcup_{h=1}^{N_k}
        \bigl(I\cap[a_{k,h-1},a_{k,h}]\bigr).
\]
Endpoints do not matter for $\mu_k$, since $\mu_k$ is atomless. They also
do not matter for $\Lambda_{k,z}$ on the full-measure size set, because local
domination gives
\[
        \Lambda_{k,z}(\{a\})
        \le
        C_U\mu_k(\{a\})
        =
        0
\]
whenever $a$ lies in a dominated neighborhood $U$. Therefore, summing the
local domination bounds over the finitely many cells gives
\begin{equation}
        \Lambda_{k,z}(I)
        \le
        C\,\mu_k(I)
        \qquad
        \text{for every interval }I\subseteq S_k,
        \text{ for a.e. }z.
        \label{eq:pone-global-domination}
\end{equation}

Moreover, since
\[
        \pi_k\E_\beta[\Lambda_{k,\beta}(B)\mid J=k]=\mu_k(B)
\]
for every Borel set $B$, and since $\mu_k([0,y]\setminus S_k)=0$, we also
have
\[
        \Lambda_{k,z}([0,y]\setminus S_k)=0
        \qquad
        \text{for a.e. }z.
\]
Thus, on the same full-measure size set, $\Lambda_{k,z}$ is finite,
atomless, and supported on $S_k$. These are precisely the support and
atomlessness conditions needed to apply Lemma~\ref{lem:jensen} with
curvature measure $\Lambda_{k,z}$.

\paragraph{Per-stage bound.}
Fix $s\ge3$, set $n=s-1$, and define
\begin{equation*}
        r_n
        :=
        C\left(\delta_n^{\,1+1/\pp}+\frac1n\right).
\end{equation*}
In the case $\pp=1$, the concentration scale \eqref{eq:tailconc} satisfies
$\delta_n^2\le C\frac{\log(en)}{n}$, so
\begin{equation}
        r_n
        =
        C\left(\delta_n^2+\frac1n\right)
        \le
        C\frac{\log(en)}{n}.
        \label{eq:pone-rn}
\end{equation}

The deterministic active-mass product estimate, Proposition~\ref{prop:active},
requires $\delta_n\le\eps_0$. This holds for all sufficiently large $n$.
For the finitely many smaller values of $n$, we enlarge the final constant in
the regret bound. Hence, in the rest of this paragraph, assume
$
        \delta_n\le\eps_0.
$

Fix a feasible current pair $(k,z)$, so $za^k\le b$, and define
\[
        b_\theta:=b-\theta z a^k,
        \qquad
        \rho_\theta:=\frac{b_\theta}{n},
        \qquad
        \theta\in[0,1].
\]
\paragraph{Clipping the resource vector.}
We first place the empirical type masses and resource vectors in the compact
ranges required by Proposition~\ref{prop:active}. Choose
$m^{\max}\in\mathbb R_+^K$ such that $\widehat m_W\le m^{\max}$ for every normalized
empirical vector $\widehat m_W$, and set
\[
        \rho^{\max}:=Am^{\max},
        \qquad
        \rho_\theta^{\rm eff}:=\rho_\theta\wedge\rho^{\max} .
\]
Here the minimum is taken coordinatewise. Because
$0\le u\le\widehat m_W\le m^{\max}$ implies $Au\le Am^{\max}=\rho^{\max}$, clipping
nonbinding resource coordinates does not change the feasible set:
\begin{equation}
        K_{\widehat m_W}(\rho_\theta)
        =
        K_{\widehat m_W}(\rho_\theta^{\rm eff}).
        \label{eq:pone-clipping-no-change}
\end{equation}
Consequently, the normalized empirical primal allocation selected in
Lemma~\ref{lem:bounded-cutoff-selection} lies in
$K_{\widehat m_W}(\rho_\theta^{\rm eff})$, and the same projected normality
condition holds for this clipped feasible set. Since
$\widehat m_W\le m^{\max}$ and $0\le\rho_\theta^{\rm eff}\le\rho^{\max}$, both
arguments lie in the fixed compact ranges on which
Proposition~\ref{prop:active} applies.

The clipping map is coordinatewise $1$-Lipschitz. Thus, for any
$\theta,\theta'\in[0,1]$,
\begin{equation}
        \|\rho_\theta^{\rm eff}-\rho_{\theta'}^{\rm eff}\|_\infty
        \le
        \|\rho_\theta-\rho_{\theta'}\|_\infty
        \le
        \frac{z\|a^k\|_\infty}{n}
        \le
        \frac{C}{n}.
        \label{eq:pone-rhs-gap}
\end{equation}

Let $W,W'\in E_n$, and let $\theta,\theta'$ be differentiability points for the
corresponding pathwise value functions. By
\eqref{eq:tailrelax}, both empirical cutoff systems satisfy the
population closed/right graph relations with perturbation $\delta_n$. By
\eqref{eq:mconc}, the vectors $\widehat m_W$ and $\widehat m_{W'}$ are
within $\delta_n$ of $m^0$, and hence within $2\delta_n$ of each other. By
\eqref{eq:pone-rhs-gap}, their clipped right-hand sides differ by at most
$C/n$. Lemma~\ref{lem:bounded-cutoff-selection} supplies projected normality.
Therefore Proposition~\ref{prop:active}, applied with $C_m=2$, $\eps=\delta_n$, and $\tau=C/n$,
gives
\begin{equation*}
        \sum_{j=1}^K
        \ell_j(q_{j,\theta,W},q_{j,\theta',W'})
        \,
        \mu_j\!\left(
        I^a_j(q_{j,\theta,W},q_{j,\theta',W'})
        \right)
        \le
        r_n.
\end{equation*}
Each summand is nonnegative, so in particular, for the current type $k$,
\begin{equation}
        \ell_k(q_{k,\theta,W},q_{k,\theta',W'})
        \,
        \mu_k\!\left(
        I^a_k(q_{k,\theta,W},q_{k,\theta',W'})
        \right)
        \le
        r_n.
        \label{eq:pone-current-coordinate-cap}
\end{equation}

Let
\[
        Q_\Omega:=q_{k,\Theta,W},
        \qquad
        \Omega=(\Theta,W),
        \qquad
        \Theta\sim{\rm Unif}[0,1],
\]
with $\Theta$ independent of $W$. By
Lemma~\ref{lem:marginal-cutoff-reduction-new},
\begin{equation}
\begin{aligned}
        \E_W\!\left[H_{k,z}(Y_W)\right]
        -
        H_{k,z}\!\left(\E_W[Y_W]\right)
        \le
        \E_\Omega H_{k,z}(Q_\Omega)
        -
        H_{k,z}\!\left(\E_\Omega Q_\Omega\right).
\end{aligned}
        \label{eq:pone-reduce-to-Q}
\end{equation}

Now condition on the good event
$
        G:=E_n(W).
$
Since $Q_\Omega\in[0,y]$, and since $H_{k,z}$ is uniformly bounded and
uniformly Lipschitz on $[0,y]$, for $n$ large enough that
$\Prob(G)\ge1/2$,
\begin{equation}
        \left|
        \Big(
        \E H_{k,z}(Q_\Omega)
        -
        H_{k,z}(\E Q_\Omega)
        \Big)
        -
        \Big(
        \E[H_{k,z}(Q_\Omega)\mid G]
        -
        H_{k,z}(\E[Q_\Omega\mid G])
        \Big)
        \right|  
        \le
        C\Prob(G^c)
        \le
        Cn^{-6}.
        \label{eq:pone-bad-future-removal}
\end{equation}
Indeed, boundedness gives
\[
        \left|
        \E H_{k,z}(Q_\Omega)
        -
        \E[H_{k,z}(Q_\Omega)\mid G]
        \right|
        \le
        C\Prob(G^c),
\]
after increasing $C$, and $Q_\Omega\in[0,y]$ gives
\[
        \left|
        \E Q_\Omega
        -
        \E[Q_\Omega\mid G]
        \right|
        \le
        C\Prob(G^c).
\]
The Lipschitz property of $H_{k,z}$ then controls the difference between the
two $H_{k,z}$-of-mean terms.

Let $\mathcal D_W\subseteq[0,1]$ be the full-measure set of differentiability
points for the path $W$, and define
\[
        \mathcal Q_G
        :=
        \left\{
        q_{k,\theta,W}:
        W\in E_n,\ \theta\in\mathcal D_W
        \right\}.
\]
Changing $Q_\Omega$ on a null set if necessary, the conditional random
variable $Q_\Omega\mid G$ takes values in $\mathcal Q_G$. Moreover, for any
two values $q,\bar q\in\mathcal Q_G$, generated by good futures and
differentiability points, \eqref{eq:pone-current-coordinate-cap} gives
\begin{equation}
        \ell_k(q,\bar q)\,
        \mu_k\!\left(I^a_k(q,\bar q)\right)
        \le
        r_n.
        \label{eq:pone-good-pairwise-mu-cap}
\end{equation}
Using the global domination bound \eqref{eq:pone-global-domination}, we get, for
a.e. feasible size $z$,
\begin{equation}
        \ell_k(q,\bar q)\,
        \Lambda_{k,z}\!\left(I^a_k(q,\bar q)\right)
        \le
        C r_n
        \qquad
        \text{for all }q,\bar q\in\mathcal Q_G.
        \label{eq:pone-good-pairwise-lambda-cap}
\end{equation}
The cap \eqref{eq:pone-good-pairwise-lambda-cap} is the pairwise active-cap
hypothesis of Lemma~\ref{lem:jensen}, with $r=Cr_n$.  It follows from
\eqref{eq:pone-good-pairwise-mu-cap} and the global domination bound
\eqref{eq:pone-global-domination}.
Applying that lemma conditionally on $G$ with
\[
        h=H_{k,z},
        \qquad
        \mu_h=\Lambda_{k,z},
        \qquad
        \mathcal Q=\mathcal Q_G,
\]
gives
\begin{equation}
        \E[H_{k,z}(Q_\Omega)\mid G]
        -
        H_{k,z}\!\left(\E[Q_\Omega\mid G]\right)
        \le
        C r_n.
        \label{eq:pone-good-jensen}
\end{equation}
Combining \eqref{eq:pone-reduce-to-Q},
\eqref{eq:pone-bad-future-removal}, and \eqref{eq:pone-good-jensen}, we obtain
\begin{equation*}
        \E_W\!\left[H_{k,z}(Y_W)\right]
        -
        H_{k,z}\!\left(\E_W[Y_W]\right)
        \le
        C r_n+Cn^{-6}.
\end{equation*}
This bound holds for every feasible current pair $(k,z)$ with $z$ in the
full-measure size set described above. Since the exceptional set has zero
arrival probability, averaging over the current type and size in the definition
of $\Xi_s(b)$ gives
\begin{equation*}
        \Xi_s(b)
        \le
        C r_n+Cn^{-6}.
\end{equation*}
When $\pp=1$, using \eqref{eq:pone-rn} yields
\begin{equation*}
        \Xi_s(b)
        \le
        C\frac{\log(e(s-1))}{s-1}.
\end{equation*}

\paragraph{Summation.}
The preceding per-stage bound is uniform over feasible capacities $b$. Since
$n=s-1$, Proposition~\ref{prop:telescope} gives
\[
        \Reg_T(\SPM;b_T)
        \le
        C+
        C\sum_{s=3}^T
        \left(
        \frac{\log(e(s-1))}{s-1}
        +
        \frac1s
        \right).
\]
The harmonic term is bounded by $C\log(eT)$, and
\[
        \sum_{s=3}^T\frac{\log(e(s-1))}{s-1}
        \le
        C(\log(eT))^2 .
\]
Therefore,
\[
        \Reg_T(\SPM;b_T)
        \le
        C(\log(eT))^2 .
\]
This proves the $\pp=1$ part of Theorem~\ref{thm:main}.

\subsection{Endpoint-contact verification when \texorpdfstring{$\pp>1$}{p>1}}

For $\pp>1$, the proof follows the same reductions as in the case $\pp=1$, but
the last curvature step changes. On dominated neighborhoods, the conditional
curvature is controlled pointwise by the weighted ratio measure. Near an
endpoint-contact neighborhood, this domination can fail because the conditional
ratio support has a moving endpoint. The replacement is a Hardy-type estimate
that integrates the conditional curvature over the endpoint branch.

The next lemma is the per-stage estimate needed to complete the proof. It uses
the same active-mass product cap from Proposition~\ref{prop:active}, followed by
the endpoint Hardy estimate. The proof is deferred to
Appendix~\ref{app:endpoint-contact}.

\begin{lemma}[Per-stage loss under endpoint contact]\label{lem:epstability}
Assume Assumption~\ref{ass:regularity} with $\pp>1$. There exist
$s_0<\infty$ and $C<\infty$ such that, for all $s\ge s_0$ and every
capacity $b$, with $n=s-1$,
\[
        \Xi_s(b)=\E_{J,\beta}\!\Big[
        \big(\E_W\!\left[H_{J,\beta}(Y_W)\right]
        -H_{J,\beta}\!\left(\E_W\!\left[Y_W\right]\right)\big)
        \1\{\beta a^J\le b\}\Big]\ \le\ C\,r_s\log(e/r_s),
\]
where
\[
        r_s
        =
        \left(\frac{\log(es)}{s}\right)^{(\pp+1)/(2\pp)}
        +
        \frac1s .
\]
\end{lemma}

\paragraph{Completion of the proof of Theorem~\ref{thm:main}.}
The case $\pp=1$ was proved in Section~\ref{subsec:sum}. It remains to consider
$\pp>1$. By Proposition~\ref{prop:telescope} and
Lemma~\ref{lem:epstability}, the finitely many stages $s<s_0$ can be absorbed
into the constant, and
\[
        \Reg_T(\SPM;b_T)
        \le
        C+
        C\sum_{s=s_0}^T
        \left(
        r_s\log(e/r_s)+\frac1s
        \right),
\]
where
\[
        r_s=
        \left(\frac{\log(es)}{s}\right)^{(\pp+1)/(2\pp)}
        +\frac1s .
\]
Since $\pp>1$, the first term in $r_s$ dominates $1/s$ up to constants for
large $s$. Hence
\[
        r_s\log(e/r_s)
        \le
        C
        s^{-(\pp+1)/(2\pp)}
        (\log(es))^{(\pp+1)/(2\pp)+1}
        +
        C\frac{\log(es)}{s}.
\]
Therefore,
\[
\begin{aligned}
        \sum_{s=s_0}^T r_s\log(e/r_s)
        &\le
        C\sum_{s=s_0}^T
        s^{-(\pp+1)/(2\pp)}
        (\log(es))^{(\pp+1)/(2\pp)+1}
        +
        C(\log(eT))^2  \\
        &\le
        C
        T^{1/2-1/(2\pp)}
        (\log(eT))^{(\pp+1)/(2\pp)+1}
        +
        C(\log(eT))^2 .
\end{aligned}
\]
The polynomial term dominates the polylogarithmic term when $\pp>1$, after
increasing $C$ if necessary. Combining this estimate with the telescoping bound
gives
\[
        \Reg_T(\SPM;b_T)
        \le
        C
        T^{1/2-1/(2\pp)}
        (\log(eT))^{(\pp+1)/(2\pp)+1}
        +C .
\]
This proves the $\pp>1$ part of Theorem~\ref{thm:main}, and completes the proof.

\section{Lower bound}\label{sec:lower}

The polynomial exponent in the $\SPM$ upper bound is sharp in the worst case over the endpoint-contact families constructed below. We show that, for
every prescribed active weighted-mass exponent $\pp>1$, no online policy can
achieve regret smaller than order $T^{1/2-1/(2\pp)}$. The construction uses one
resource and one arrival type. Its weighted ratio measure has endpoint
exponent exactly $\pp$ at both endpoints of its support, so the smallest exponent
satisfying the active-mass condition \eqref{eq:wr} is $\pp$ itself.

The construction isolates the mechanism behind the lower bound. The operative
ratio cutoff lies at the lower edge of the ratio support. That edge is a corner:
it is reached only when consumption and value simultaneously approach their
extreme values. As a result, the resource mass just above the cutoff is governed
by the product of the value density and the size density near the boundary. We
choose the size density symmetrically at its two endpoints, so the same contact
exponent appears at both ratio endpoints. Thus the active weighted-mass exponent
used in the upper bound is exactly the exponent used in the lower-bound
construction.  The capacity is set at the critical, accept-all level, so the
operative cutoff is pinned at this endpoint; there the capacity-local exponent of
Appendix~\ref{app:local} coincides with the global exponent, and the polynomial
lower bound is a critical-capacity effect.

\subsection{A one-resource endpoint-contact family}

Fix $\pp>1$. The lower-bound instance has one resource, one arrival type, and
unit consumption direction. The size is $\beta=2-Y$, where
$Y\sim\mathrm{Beta}(\pp-1,\pp-1)$ has density
\[
        f_Y(y)=\frac{1}{B(\pp-1,\pp-1)}
        y^{\pp-2}(1-y)^{\pp-2},
        \qquad 0\le y\le1 .
\]
The reward is $V=1+S$, where $S\sim\Unif[0,1]$, and $S$ and $Y$ are independent.
Thus $\beta\in[1,2]$ and $V\in[1,2]$, and the mean size is
$\mu:=\E[\beta]=\tfrac32$. We set the capacity at the fluid scale
\[
        b_T=T\mu=\frac32 T .
\]

The value-to-size ratio is
\[
        R=\frac{V}{\beta}=\frac{1+S}{2-Y}\in[1/2,2].
\]
Let $r_0=1/2$ denote the lower endpoint of the ratio support. This endpoint is
reached only in the joint limit $S\downarrow0$ and $Y\downarrow0$. The upper
endpoint $2$ is reached only in the joint limit $S\uparrow1$ and $Y\uparrow1$.

This family is the independent value-and-size class of
Proposition~\ref{prop:independent-value-size} (the case behind
Corollary~\ref{cor:regret-beta}), with uniform value ($a_V^\pm=1$) and
$\beta=2-Y$, $Y\sim\mathrm{Beta}(\pp-1,\pp-1)$
($a_\beta^\pm=\pp-1$). By that proposition both ratio-endpoint exponents
equal $a_V^-+a_\beta^+=a_V^++a_\beta^-=\pp$, so the weighted ratio measure
satisfies the active-mass condition \eqref{eq:wr} with exponent $\pp$; and
since the endpoint mass is of order $x^{\pp}$, no smaller exponent is
admissible. Hence $\pp$ is exactly the smallest admissible active-mass
exponent of this family.

The lower-bound proof uses only the following two endpoint estimates.

\begin{lemma}[Endpoint mass]\label{lem:lb-mass}
For all sufficiently small $x>0$,
\[
        \E\!\left[\beta\,\1\{R\le r_0+x\}\right]\asymp x^{\pp},
        \qquad
        \E\!\left[\beta(R-r_0)\,\1\{R\le r_0+x\}\right]\asymp x^{\pp+1}.
\]
The constants depend only on $\pp$.
\end{lemma}

\begin{proof}
This family is the independent value-and-size class of
Proposition~\ref{prop:independent-value-size}, with uniform value
($a_V^\pm=1$) and $\beta=2-Y$, $Y\sim\mathrm{Beta}(\pp-1,\pp-1)$
($a_\beta^\pm=\pp-1$). Its lower ratio endpoint exponent is therefore
$a_V^-+a_\beta^+=\pp$, and the verification there, through
Lemma~\ref{lem:mass}, gives the weighted ratio measure
$\E[\beta\,\1\{R\in\cdot\}]$ the endpoint density
\[
        m(r_0+t)\asymp t^{\pp-1}
        \qquad\text{as }t\downarrow0 .
\]
Hence
\[
        \E\!\left[\beta\,\1\{R\le r_0+x\}\right]
        =\int_0^x m(r_0+t)\,\dd t\asymp x^{\pp},
\]
and, integrating the same density against the ratio premium $r-r_0=t$,
\[
        \E\!\left[\beta(R-r_0)\,\1\{R\le r_0+x\}\right]
        =\int_0^x t\,m(r_0+t)\,\dd t\asymp x^{\pp+1}.
\]
\end{proof}

\subsection{The matching lower bound}

The lower bound is a two-point dilemma at the lower ratio endpoint. An endpoint
layer of width $\eps_T$ carries $\Theta(\eps_T^{\pp}T)$ weighted resource
(Lemma~\ref{lem:lb-mass}), which an online policy cannot distinguish from the
$\Theta(\sqrt T)$ fluctuation of total demand exactly when
$\eps_T^{\pp}\sqrt T\asymp1$, i.e.\ $\eps_T\asymp T^{-1/(2\pp)}$.
Since the per-unit regret for misjudging the layer is the ratio premium $\eps_T$,
the unavoidable regret is $\eps_T\cdot\sqrt T\asymp T^{1/2-1/(2\pp)}$,
larger for a thinner endpoint (larger $\pp$).

\begin{theorem}[Endpoint-contact lower bound]\label{thm:lower}
For the family above, there is a constant $c_{\pp}>0$ such that, for all
sufficiently large even $T$,
\[
        \inf_{\pi}\Reg_T(\pi;b_T)\ \ge\ c_{\pp}\,T^{1/2-1/(2\pp)} .
\]
\end{theorem}

\begin{proof}
Write $T=2n$ and split the horizon into two halves of length $n$. Fix a small
constant $\eta>0$, to be chosen below, and set
\[
        \eps_T=\eta T^{-1/(2\pp)} .
\]
Define the two low-ratio layers
\[
        L_{\eps}=\{R\le r_0+\eps_T\},
        \qquad
        L_{2\eps}=\{R\le r_0+2\eps_T\},
\]
and, for the first half,
\[
        W=\sum_{t=1}^n \beta_t\1\{(\beta_t,V_t)\in L_\eps\},
        \qquad
        U_1=\sum_{t=1}^n \beta_t\1\{(\beta_t,V_t)\in L_{2\eps}\},
        \qquad
        S_1=\sum_{t=1}^n \beta_t .
\]
Thus $W$ and $U_1$ are the first-half resource in the thinner and thicker
endpoint layers, and $S_1$ is the total first-half resource.

\smallskip
\noindent\emph{Step 1: Construct a favorable first-half event.}
By Lemma~\ref{lem:lb-mass}, $\E W=n\,\Theta(\eps_T^{\pp})
=\Theta(\eta^{\pp}\sqrt T)$, and the summands of $W$ are bounded, so
$\Var(W)\le C\E W$. Chebyshev's inequality then gives
$\Prob\{W<c_W\eta^{\pp}\sqrt T\}\le C/(\eta^{\pp}\sqrt T)$ for a
sufficiently small constant $c_W>0$, and choosing $c_W$ small enough also gives
\begin{equation}
        \E\!\left[W\1\{W\ge c_W\eta^{\pp}\sqrt T\}\right]
        \ge c\,\eta^{\pp}\sqrt T .
        \label{eq:lb-thm-W-truncated}
\end{equation}
Since $\Var(S_1)=\Theta(T)$, Chebyshev's inequality gives
$\Prob\{|S_1-n\mu|\le K\sqrt T\}\ge\tfrac34$ for a fixed large constant $K$.
Lemma~\ref{lem:lb-mass} also gives $\E U_1\le C\eta^{\pp}\sqrt T$, so after
fixing $K$ and then choosing $\eta$ small enough that
$C\eta^{\pp}\le K/16$, Markov's inequality gives
$\Prob\{U_1\le K\sqrt T\}\ge\tfrac{15}{16}$.

Let $F=A_0\cap B_0\cap C_0$ be the intersection of the first-half events
\[
        A_0=\{W\ge c_W\eta^{\pp}\sqrt T\},
        \qquad
        B_0=\{|S_1-n\mu|\le K\sqrt T\},
        \qquad
        C_0=\{U_1\le K\sqrt T\}.
\]
We now show that intersecting with $B_0$ and $C_0$ does not destroy too much
$W$-mass. By Cauchy--Schwarz,
\[
        \E[W\1_{B_0^c}]\le(\E W^2)^{1/2}\,\Prob(B_0^c)^{1/2}.
\]
Since $\E W^2=(\E W)^2+\Var(W)\le(\E W)^2+C\E W$ and $\Prob(B_0^c)\le C/K^2$,
choosing $K$ large enough gives
\begin{equation}
        \E[W\1_{B_0^c}]\le\tfrac14 c\,\eta^{\pp}\sqrt T
        \label{eq:lb-thm-W-loss-B}
\end{equation}
for all large $T$, where $c$ is the constant in \eqref{eq:lb-thm-W-truncated}.
Second, since $W\le U_1$ and
$\Prob(C_0^c)=\Prob\{U_1>K\sqrt T\}\le C\eta^{\pp}/K$, the bound
$\E U_1^2\le(\E U_1)^2+C\E U_1$ and Cauchy--Schwarz give
\[
        \E[W\1_{C_0^c}]
        \le\E[U_1\1_{C_0^c}]
        \le(\E U_1^2)^{1/2}\,\Prob(C_0^c)^{1/2}
        \le C\eta^{\pp}\sqrt T
        \Bigl(\tfrac{\eta^{\pp}}{K}\Bigr)^{1/2}
        +o(\eta^{\pp}\sqrt T).
\]
After fixing $K$, choosing $\eta$ small enough and then $T$ large gives
\begin{equation}
        \E[W\1_{C_0^c}]\le\tfrac14 c\,\eta^{\pp}\sqrt T .
        \label{eq:lb-thm-W-loss-C}
\end{equation}
Combining \eqref{eq:lb-thm-W-truncated}, \eqref{eq:lb-thm-W-loss-B}, and
\eqref{eq:lb-thm-W-loss-C},
\begin{equation}
        \E[W\1_F]\ge c_F\,\eta^{\pp}\sqrt T
        \label{eq:lb-thm-W-on-F}
\end{equation}
for a primitive constant $c_F>0$; and on $F$ one has
$W\ge c_W\eta^{\pp}\sqrt T$, $|S_1-n\mu|\le K\sqrt T$, and
$U_1\le K\sqrt T$.

\smallskip
\noindent\emph{Step 2: The second half forces a dilemma.}
Fix an arbitrary first-half history in $F$ (and, for a randomized policy, its
internal randomization through the first half). Let $Y_{\rm rej}\in[0,W]$ be the
amount of first-half $L_\eps$-resource rejected by the online policy, and set
$L=W-Y_{\rm rej}$ for the accepted amount. Let
$S_2=\sum_{t=n+1}^{2n}\beta_t$ and
$U_2=\sum_{t=n+1}^{2n}\beta_t\1\{(\beta_t,V_t)\in L_{2\eps}\}$ be the second-half
analogues of $S_1$ and $U_1$; the second half is independent of the first-half
history and policy decisions.

Since $\beta$ is bounded and nondegenerate, the central limit theorem gives a
constant $p_->0$ such that $E^-=\{S_2-n\mu\le -2K\sqrt T\}$ satisfies
$\Prob(E^-)\ge p_-$ for all large $T$. Similarly
$\Prob\{S_2-n\mu\ge 4K\sqrt T\}\ge q_+$ for some $q_+>0$; since
$\E U_2\le C\eta^{\pp}\sqrt T$, choosing $\eta$ small enough that
$C\eta^{\pp}/K\le q_+/2$ gives $\Prob\{U_2>K\sqrt T\}\le q_+/2$ by Markov,
so $E^+=\{S_2-n\mu\ge4K\sqrt T,\ U_2\le K\sqrt T\}$ satisfies
$\Prob(E^+)\ge p_+:=q_+/2>0$ for all large $T$.

\smallskip
\noindent\emph{Resource-poor second half.}
On $F\cap E^-$ the bounds on $S_1$ and $S_2$ give
$S_1+S_2\le n\mu+K\sqrt T+n\mu-2K\sqrt T=b_T-K\sqrt T<b_T$, so accepting every
arrival is feasible; as all values are nonnegative, the hindsight optimum accepts
every arrival. Every unit of first-half $L_\eps$-resource the online policy
rejects is then lost relative to hindsight, and on $L_\eps$ each such unit has
ratio $R\ge r_0$. Hence
\begin{equation}
        \operatorname{regret}\ge r_0Y_{\rm rej}
        \qquad\text{on }F\cap E^- .
        \label{eq:lb-thm-regret-poor}
\end{equation}

\smallskip
\noindent\emph{Resource-rich second half.}
On $F\cap E^+$ we have
$S_1+S_2\ge n\mu-K\sqrt T+n\mu+4K\sqrt T=b_T+3K\sqrt T$. Also
$U_1+U_2\le2K\sqrt T$, so the total resource outside $L_{2\eps}$ is at least
$S_1+S_2-(U_1+U_2)\ge b_T+K\sqrt T>b_T$; thus there is a feasible hindsight
solution of total resource $b_T$ using only arrivals outside $L_{2\eps}$, every
unit of which has ratio at least $r_0+2\eps_T$.

Let $X$ be the online accepted fractional resource measure, with total mass
$M=|X|\le b_T$. Its first-half $L_\eps$ part has mass $L=W-Y_{\rm rej}$ and ratio
at most $r_0+\eps_T$; let $X_{\rm hi}$ be the remaining online resource, so
$|X_{\rm hi}|=M-L$. After removing $X_{\rm hi}$, at least $b_T-(M-L)$ units of
resource outside $L_{2\eps}$ remain, so the solution consisting of $X_{\rm hi}$
together with $b_T-M+L$ such units is feasible (its total mass is $b_T$).
Comparing it to the online solution, the common $X_{\rm hi}$ cancels: the
comparison gains $b_T-M+L$ units of ratio at least $r_0+2\eps_T$, while the online
solution has $L$ units of ratio at most $r_0+\eps_T$ and possibly unused capacity
$b_T-M$. Hence, using $L=W-Y_{\rm rej}$,
\begin{equation}
\begin{aligned}
        \operatorname{regret}
        &\ge
        (r_0+2\eps_T)(b_T-M+L)-(r_0+\eps_T)L  \\
        &=(r_0+2\eps_T)(b_T-M)+\eps_T L
        \ge\eps_T(W-Y_{\rm rej})
        \qquad\text{on }F\cap E^+ .
\end{aligned}
        \label{eq:lb-thm-regret-rich}
\end{equation}

\smallskip
\noindent\emph{Step 3: Optimize over the policy's first-half choice.}
Conditioning on the first-half history in $F$ and using $\Prob(E^-)\ge p_-$,
$\Prob(E^+)\ge p_+$, \eqref{eq:lb-thm-regret-poor}, and
\eqref{eq:lb-thm-regret-rich},
\begin{equation*}
\begin{aligned}
        \E[\operatorname{regret}\mid\text{first half}]
        &\ge
        p_- r_0 Y_{\rm rej}+p_+\eps_T(W-Y_{\rm rej})  \\
        &=
        p_+\eps_T W+(p_-r_0-p_+\eps_T)Y_{\rm rej}.
\end{aligned}
\end{equation*}
Since $\eps_T\to0$, for all large $T$ we have $p_+\eps_T\le p_-r_0$, so the right
side is nondecreasing in $Y_{\rm rej}$ and minimized at $Y_{\rm rej}=0$; hence
$\E[\operatorname{regret}\mid\text{first half}]\ge p_+\eps_T W$ on $F$. Taking
expectations and using \eqref{eq:lb-thm-W-on-F},
\[
        \E[\operatorname{regret}]
        \ge p_+\eps_T\,\E[W\1_F]
        \ge c\,\eta T^{-1/(2\pp)}\cdot\eta^{\pp}\sqrt T
        =c\,\eta^{\pp+1}T^{1/2-1/(2\pp)} .
\]
Absorbing the fixed $\eta^{\pp+1}$ into the constant gives
$\E[\operatorname{regret}]\ge c_{\pp}T^{1/2-1/(2\pp)}$. The online
policy was arbitrary, including randomized policies, so the same bound holds after
taking the infimum over all online policies.
\end{proof}

When the reward and size are independent and both uniform on $[1,2]$ --- the
construction above at $\pp=2$, where $Y\sim\Unif[0,1]$ --- with capacity
$b_T=\frac32 T$, the active weighted-mass exponent is $\pp=2$, and
Theorem~\ref{thm:lower} gives
\[
        \inf_{\pi}\Reg_T(\pi;b_T)\ge c\,T^{1/4}.
\]

This theorem isolates how joint reward-size randomness can create a thin ratio-support endpoint and force polynomial regret. When consumption is
deterministic ($\beta\equiv1$), the ratio is a one-dimensional function of the value
alone, its support has no corner, the relevant endpoint is regular ($\pp=1$),
and polylogarithmic regret is achievable; the same holds when the value is
deterministic and the size has a regular distribution. It is the joint randomness of consumption and value that turns the
ratio-support edge into a corner, raises the active weighted-mass exponent to $\pp=2$, and
makes $T^{1/4}$ the exact polynomial exponent for this instance. Together with the SPM upper bound of
Section~\ref{sec:spm}---which matches $T^{1/2-1/(2\pp)}$ up to a logarithmic factor for
every $\pp\ge1$ and gives polylogarithmic regret at $\pp=1$---the two
results pin down the polynomial exponent as a function of the single active
weighted-mass exponent $\pp$.

\section{Concluding remarks}\label{sec:conclusion}

This paper studies online allocation with random rewards and random consumption
sizes.  Under the paper's weighted-ratio regularity conditions, the main conclusion is that regret is controlled by the size-weighted value-to-size ratio measure near active acceptance cutoffs.  When this
size-weighted resource mass grows linearly, the sample-path marginal policy
attains \(O((\log T)^2)\) regret.  When the mass vanishes faster than linearly,
the problem becomes polynomially harder: for exponent \(\pp>1\), the
policy attains the rate \(T^{1/2-1/(2\pp)}\) up to logarithmic factors,
and the lower bound shows that the polynomial exponent is unavoidable.  The
mechanism is distributional.  Jointly random rewards and consumptions can make
the critical value-to-size ratio occur only on a thin part of the joint support,
such as a corner, even when the marginal reward and consumption distributions have bounded densities on compact supports.  This is one price of degeneracy: continuous random consumption can create thin cutoff mass.  This polynomial price is a critical-capacity phenomenon: a capacity-local refinement of the exponent shows that, at any fixed non-critical capacity, the active cutoff is interior and the same primitive distribution is polylogarithmic (Appendix~\ref{app:local}).

A companion note~\cite{Zhang2026b} studies the unknown-distribution setting.
When the arrival distribution is unknown but all arrivals are observed, and the
endpoint shape parameters---the local endpoint mass exponents of
Definition~\ref{def:exponent}---are known, a smoothed empirical version of the
expected-hindsight marginal rule is analyzed there through the same swept
active-mass stability estimate.  The companion note shows that this plug-in rule
preserves the rates of Theorem~\ref{thm:main}: \((\log T)^2\) when \(\pp=1\), and
\(T^{1/2-1/(2\pp)}\) up to logarithmic factors when \(\pp>1\).

Several questions remain open.  One is algorithmic: the sample-path marginal
policy uses the marginal value of the expected hindsight problem, while much of
the online allocation literature works with policies that periodically re-solve
fluid or empirical relaxations.  It would be useful to understand whether the
same rates can be achieved with substantially fewer re-solves, or with simpler
approximations to the marginal value rule.  A second question is structural.  In
this paper, a type has one random scalar size that scales a fixed resource
bundle.  A natural next step is to allow a request to have a genuinely random
consumption vector, with different random size distributions across resources.
Such a model would allow several resources to become critical simultaneously and
would require controlling multidimensional cutoff structure without imposing
fluid non-degeneracy assumptions such as unique dual prices or strict
complementarity.

\appendix

\section{Auxiliary proofs for Section~\ref{sec:spm}}

\subsection{Proof of Proposition~\ref{prop:telescope}}\label{app:telescope}

\begin{proof}[Proof of Proposition~\ref{prop:telescope}]
To analyze the total reward collected by the SPM algorithm, we denote,  at the 
decision epoch with $s$ periods remaining, the current arrival as
$Z_s=(J_s,\beta_s,V_s)$ and  $X_s\in\{0,1\}$ as the $\SPM$ decision. Recall that
$B_T=b_T$. We have
\begin{equation}
\begin{aligned}
        & \ \ \ \ \Reg_T(\SPM;b_T)\\
        &=
        \Phi_T(b_T)-\E\!\left[\sum_{s=1}^T V_sX_s\right]  \\
        &=
        \E\!\left[
        \Phi_T(B_T)-\Phi_0(B_0)
        -
        \sum_{s=1}^T V_sX_s
        \right] \\
        &=
        \E\!\left[
        \sum_{s=1}^T
        \bigl(\Phi_s(B_s)-\Phi_{s-1}(B_{s-1})\bigr)
        -
        \sum_{s=1}^T V_sX_s
        \right] \\
        &=
        \E\!\left[
        \sum_{s=1}^T
        \bigl(\Phi_s(B_s)-V_sX_s-\Phi_{s-1}(B_{s-1})\bigr)
        \right] \\
        &=
        \E\!\left[
        \sum_{s=1}^T
        \bigl(\Phi_s(B_s)-V_sX_s-\Phi_{s-1}(B_s-\beta_s a^{J_s}X_s)\bigr)
        \right]\\
        &=
        \E\!\left[
        \sum_{s=1}^T
        \left(
        \Phi_s(B_s)
        -
        \E\!\left[
        \begin{aligned}
        &\max\{\Phi_{s-1}(B_s),\,
        V_s+\Phi_{s-1}(B_s-\beta_s a^{J_s})\}
        \1\{\beta_s a^{J_s}\le B_s\}                                  \\
        &\qquad
        +
        \Phi_{s-1}(B_s)
        \1\{\beta_s a^{J_s}\not\le B_s\}
        \end{aligned}
        \,\middle|\, B_s
        \right]
        \right)
        \right].
\end{aligned}
        \label{eq:regret-telescope-expanded}
\end{equation}

The inner expectation is over the current arrival $Z_s$, conditionally on the
realized remaining capacity $B_s$. The last equality uses that $\SPM$ is the
one-step greedy rule. By
\eqref{eq:accept}, on the feasible event $\{\beta_s a^{J_s}\le B_s\}$ the policy
accepts ($X_s=1$) exactly when
$V_s+\Phi_{s-1}(B_s-\beta_s a^{J_s})\ge\Phi_{s-1}(B_s)$, so
$V_sX_s+\Phi_{s-1}(B_s-\beta_s a^{J_s}X_s)$ equals the displayed maximum; on the
infeasible event $X_s=0$, so this quantity equals $\Phi_{s-1}(B_s)$.

We now upper-bound the first term in the one-step gap, $\Phi_s(B_s)$. Recall
that $\Phi_s(B_s)$ is the expected value of the $s$-period offline fractional
hindsight LP with initial capacity $B_s$. The offline LP allows fractional
decisions. We first show that, for the current arrival $Z_s=(J_s,\beta_s,V_s)$,
any fractional use can be removed at expected cost $O(1/s)$.

For each realized $s$-tuple of arrivals and capacity $B_s$, choose a basic
optimal solution by a permutation-equivariant measurable rule. One way to obtain
such a rule is to attach independent continuous auxiliary labels to the $s$
arrivals, order the arrivals by these labels, and choose the first basic optimum
in the resulting finite, label-ordered list of bases. The labels are used only
to select among optimal basic solutions and do not change the offline value.

A basic optimum of a $d$-resource fractional knapsack LP has at most $d$
fractional variables. Since the selected basic optimum is permutation-equivariant
and the $s$ arrivals are exchangeable, the distinguished current arrival is
fractional with probability
\[
        \E\!\left[
        \frac1s\sum_{i=1}^s \1\{0<X_i^*<1\}
        \right]
        \le
        \frac ds .
\]
Rounding down the fractional use of this one arrival loses at most $\bar v$.
Let $\OPT_s^{\mathrm b}(B_s;Z_s,W)$ be the $s$-period offline optimum in which the
current arrival $Z_s$ is restricted to a binary decision, while the remaining
$s-1$ arrivals $W$ stay fractional. Rounding the distinguished arrival in the
selected basic optimum gives the pathwise bound
\[
        \OPT_s(B_s;Z_s,W)-\OPT_s^{\mathrm b}(B_s;Z_s,W)\le\bar v\,\1\{0<X_s^*<1\},
\]
so, taking expectations and using the fractional-probability bound above,
\[
        \Phi_s(B_s)\le\E\!\left[\OPT_s^{\mathrm b}(B_s;Z_s,W)\right]+\frac{d\bar v}{s} .
\]
It therefore suffices to bound $\E[\OPT_s^{\mathrm b}(B_s;Z_s,W)]$, at the cost of
the additive $C/s$.

Now separate the two cases for the binary offline decision on $Z_s$. Let $W$ be
any sample path of the remaining $s-1$ arrivals. If
$\beta_s a^{J_s}\not\le B_s$, then accepting $Z_s$ is infeasible, so the current
arrival must be rejected and the future value is at most
$\OPT_{s-1}(B_s;W)$. If $\beta_s a^{J_s}\le B_s$, then the offline solution may
either reject $Z_s$, leaving capacity $B_s$, or accept it, leaving capacity
$B_s-\beta_s a^{J_s}$. Thus, with the current arrival restricted to a binary decision,
\[
\begin{aligned}
&\OPT_s^{\mathrm b}(B_s;Z_s,W)\\
&\quad =
\max\{\OPT_{s-1}(B_s;W),
V_s+\OPT_{s-1}(B_s-\beta_s a^{J_s};W)\}
\mathbf 1\{\beta_s a^{J_s}\le B_s\}\\
&\qquad
+\OPT_{s-1}(B_s;W)\mathbf 1\{\beta_s a^{J_s}\not\le B_s\}.
\end{aligned}
\]

It remains to rewrite the feasible-acceptance term. On the event
$\{\beta_s a^{J_s}\le B_s\}$, define
\[
        Y_W
        = Y_W(B_s, J_s, \beta_s)=
        \frac{
        \OPT_{s-1}(B_s;W)
        -
        \OPT_{s-1}(B_s-\beta_s a^{J_s};W)}
        {\beta_s}.
\]
Then, conditional on $B_s,J_s,\beta_s,W$,
\[
        \max\{\OPT_{s-1}(B_s;W),
        V_s+\OPT_{s-1}(B_s-\beta_s a^{J_s};W)\}
        =
        \OPT_{s-1}(B_s;W)+\bigl(V_s-\beta_sY_W\bigr)^+.
\]
Averaging this pathwise identity first over $V_s$ and then over $W$ gives
\[
\begin{aligned}
&\E_W
\E_{V_s}\!\left[
        \max\{\OPT_{s-1}(B_s;W),
        V_s+\OPT_{s-1}(B_s-\beta_s a^{J_s};W)\}
        \mid B_s,J_s,\beta_s,W
\right] \\
&\quad =
\E_W\!\left[
        \OPT_{s-1}(B_s;W)
        +
        \E_{V_s}\!\left[
        \bigl(V_s-\beta_sY_W\bigr)^+
        \mid J_s,\beta_s
        \right]
\right] \\
&\quad =
\E_W\!\left[
        \OPT_{s-1}(B_s;W)
        +
        H_{J_s,\beta_s}(Y_W)
\right] \\
&\quad =
\Phi_{s-1}(B_s)
+
\E_W\!\left[H_{J_s,\beta_s}(Y_W)\right]\\
& \quad =\Phi_{s-1}(B_s)
+
H_{J_s,\beta_s}\!\left(\E_W\!\left[Y_W\right]\right) 
+
\Bigl[
\E_W\!\left[H_{J_s,\beta_s}(Y_W)\right]
-
H_{J_s,\beta_s}\!\left(\E_W\!\left[Y_W\right]\right)
\Bigr] \\ 
&\quad =
\mathbb E\!\left[
\left.
\max\{\Phi_{s-1}(B_s),
V_s+\Phi_{s-1}(B_s-\beta_s a^{J_s})\}
\right| B_s,J_s,\beta_s
\right]+\Bigl[
\E_W\!\left[H_{J_s,\beta_s}(Y_W)\right]
-
H_{J_s,\beta_s}\!\left(\E_W\!\left[Y_W\right]\right)
\Bigr]
\end{aligned}
\]
where the last equality follows from the definition of $H_{J_s, \beta_s}$ and the
identity
\[
        \beta_s\,\E_W\!\left[Y_W\right]
        =
        \Phi_{s-1}(B_s)
        -
        \Phi_{s-1}(B_s-\beta_s a^{J_s})
\].
The bracketed term is the Jensen loss from replacing the random future-path
marginal $Y_W$ by its expectation. Averaging this Jensen loss over the feasible
current type and size gives $\Xi_s(B_s)$. Therefore
\begin{equation}
\begin{aligned}
\Phi_s(B_s)
\le\;&
\mathbb E\!\left[
\left.
\max\{\Phi_{s-1}(B_s),
V_s+\Phi_{s-1}(B_s-\beta_s a^{J_s})\}
\mathbf 1\{\beta_s a^{J_s}\le B_s\}
\right| B_s
\right] \\
&\quad
+\Phi_{s-1}(B_s)
        \mathbb P(\beta_s a^{J_s}\not\le B_s\mid B_s)
+\Xi_s(B_s)
+\frac{C}{s}.
\end{aligned}
\label{eq:one-step-comparison}
\end{equation}

The first two terms on the right-hand side of
\eqref{eq:one-step-comparison} are precisely the one-step expression already
subtracted in \eqref{eq:regret-telescope-expanded}. Hence, for every $s\ge3$,
\[
\mathbb E\!\left[
\Phi_s(B_s)
-
V_sX_s
-
\Phi_{s-1}(B_s-\beta_s a^{J_s}X_s)
\right]
\le
\mathbb E[\Xi_s(B_s)]
+
\frac{C}{s}.
\]
Substituting this bound into \eqref{eq:regret-telescope-expanded} and summing over
$s=1,\ldots,T$ --- with $\Xi_s\equiv0$ for $s\le2$, whose one-step gaps are $O(1)$
and at most $C/s$ after enlarging $C$ --- gives
\[
        \Reg_T(\SPM;b_T)
        \le
        C+
        \sum_{s=1}^T
        \left(
        \mathbb E[\Xi_s(B_s)]
        +
        \frac{C}{s}
        \right).
\]
This proves the proposition.\end{proof}

\subsection{Cutoff selection and active-hull closure}\label{app:cutoff-details}

This appendix supplies two technical ingredients deferred from
Section~\ref{subsec:sweep}.  The first is the bounded cutoff selection lemma.
The second is the active-hull closure lemma, which turns a pairwise active cap
into a cap on the full projected hull.

\begin{proof}[Proof of Lemma~\ref{lem:bounded-cutoff-selection}]
The function \(g(\theta)\) is concave and Lipschitz, because
\(\OPT_n(\cdot;W)\) is the optimal value of a finite-dimensional linear program
as a function of its right-hand side.  Hence \(g\) is differentiable for almost
every \(\theta\in[0,1]\).  Fix such a value of \(\theta\).

We first bound the derivative of \(g\) along this capacity path by the
reward-ratio support.  For \(0<\varepsilon\le 1-\theta\),
\[
        0\le g(\theta)-g(\theta+\varepsilon)\le \varepsilon z y .
\]
The lower bound follows from monotonicity of the offline value in capacity.  To
prove the upper bound, we pass to the dual of the box-constrained offline LP.
Its upper-bound constraints \(x_i\le1\) carry nonnegative multipliers
\(\eta_i\), and the dual is
\[
        \min_{\lambda\ge0,\ \eta\ge0}\
        b_{\theta+\varepsilon}^\top\lambda+\sum_{i=1}^n\eta_i
        \qquad\text{subject to}\qquad
        \beta_i(a^{J_i})^\top\lambda+\eta_i\ge V_i,\quad i=1,\ldots,n .
\]
Let \((\lambda,\eta)\) be an optimal dual solution, and set
\(M_j:=\vbar/(\betamin\alpha_j)\).  We claim that clipping \(\lambda\) to
\(\lambda\wedge M\), while keeping \(\eta\) fixed, gives another optimal dual
solution.  Feasibility is preserved item by item.  Fix \(i\).  If every
coordinate \(j\) with \(a^{J_i}_j>0\) has \(\lambda_j\le M_j\), then
\((a^{J_i})^\top(\lambda\wedge M)=(a^{J_i})^\top\lambda\) on those coordinates and
the \(i\)-th constraint is unchanged.  Otherwise some such coordinate has
\(\lambda_j>M_j\), and since \(a^{J_i}\neq0\) gives \(a^{J_i}_j\ge\alpha_j\),
\[
        \beta_i(a^{J_i})^\top(\lambda\wedge M)
        \ge \beta_i a^{J_i}_j M_j
        \ge \betamin\,\alpha_j\cdot\frac{\vbar}{\betamin\alpha_j}
        =\vbar\ge V_i ,
\]
so the \(i\)-th constraint holds even with \(\eta_i=0\).  Because
\(b_{\theta+\varepsilon}\ge0\), which holds since \(za^k\le b\), and
\(\lambda\wedge M\le\lambda\), the objective
\(b_{\theta+\varepsilon}^\top(\lambda\wedge M)+\sum_i\eta_i\) does not increase,
so \((\lambda\wedge M,\eta)\) is again optimal.  Therefore the offline LP at
\(b_{\theta+\varepsilon}\) has an optimal resource-dual vector \(\lambda\) with
\(\lambda_j\le M_j\) for all \(j\).

Since the offline value is concave in the right-hand side, this dual vector is a
supergradient at \(b_{\theta+\varepsilon}\).  Taking
\(b_\theta=b_{\theta+\varepsilon}+\varepsilon z a^k\), we obtain
\[
        g(\theta)
        =\OPT_n(b_\theta;W)
        \le
        \OPT_n(b_{\theta+\varepsilon};W)
        +\lambda^\top(b_\theta-b_{\theta+\varepsilon})
        =g(\theta+\varepsilon)+\varepsilon z\,\lambda^\top a^k .
\]
The coordinatewise bound on \(\lambda\) gives
\[
        \lambda^\top a^k\le \sum_{j=1}^d M_j a^k_j=y_k\le y .
\]
Hence \(0\le g(\theta)-g(\theta+\varepsilon)\le\varepsilon z y\).  Dividing by
\(\varepsilon z\) and letting \(\varepsilon\downarrow0\) gives
\begin{equation}
        0\le -\frac{g'(\theta)}{z}\le y .
        \label{eq:derivative-cutoff-bounded}
\end{equation}

Let \(x^\theta\) be an optimal solution of the pathwise fractional problem with
capacity \(b_\theta\).  Define the normalized type allocation by
\[
        u_{\theta,W,\ell}
        =
        \frac1n\sum_{i:J_i=\ell}\beta_i x_i^\theta,
        \qquad \ell=1,\ldots,K .
\]
Then
\[
        0\le u_{\theta,W}\le\widehat m_W,
        \qquad
        Au_{\theta,W}\le\rho_\theta .
\]
Thus \(u_{\theta,W}\in K_{\widehat m_W}(\rho_\theta)\).

Let \(\lambda_\theta\) be an optimal dual vector for the resource constraints at
capacity \(b_\theta\) such that
\[
        g'(\theta)=-z (a^k)^\top\lambda_\theta .
\]
Such a choice exists by standard LP sensitivity analysis; see
\citet[Sec.~5.6.3, Eq.~(5.58)]{boyd2004convex}.  We construct
\(q_{\theta,W}\) from this dual vector.

Complementary slackness gives, for every item \(i\) with \(J_i=\ell\),
\[
        V_i/\beta_i>(a^\ell)^\top\lambda_\theta \Rightarrow x_i^\theta=1,
        \qquad
        V_i/\beta_i<(a^\ell)^\top\lambda_\theta \Rightarrow x_i^\theta=0 .
\]
Consequently, whenever \((a^\ell)^\top\lambda_\theta\le y\),
\[
        \widehat C_{\ell,W}\bigl(((a^\ell)^\top\lambda_\theta)+\bigr)
        \le u_{\theta,W,\ell}
        \le \widehat C_{\ell,W}((a^\ell)^\top\lambda_\theta).
\]
Set
\[
        q_{\theta,W,\ell}=\min\{(a^\ell)^\top\lambda_\theta,y\},
        \qquad \ell=1,\ldots,K .
\]
Then \(q_{\theta,W}\in[0,y]^K\).  Since
\((a^k)^\top\lambda_\theta=-g'(\theta)/z\), the bound
\eqref{eq:derivative-cutoff-bounded} gives
\[
        q_{\theta,W,k}=-\frac{g'(\theta)}{z},
\]
which proves \eqref{eq:selection-derivative}.

It remains to verify the tail condition.  If
\((a^\ell)^\top\lambda_\theta\le y\), then
\eqref{eq:selection-tail} follows from the preceding display.  If
\((a^\ell)^\top\lambda_\theta>y\), then
\(V_i/\beta_i<(a^\ell)^\top\lambda_\theta\) for every item \(i\) with
\(J_i=\ell\).  Complementary slackness therefore gives \(x_i^\theta=0\) for all
such items, and hence \(u_{\theta,W,\ell}=0\).  Since
\(q_{\theta,W,\ell}=y\),
\[
        \widehat C_{\ell,W}(q_{\theta,W,\ell}+)
        =
        \widehat C_{\ell,W}(y+)
        =0
        \le u_{\theta,W,\ell}
        \le \widehat C_{\ell,W}(y)
        =
        \widehat C_{\ell,W}(q_{\theta,W,\ell}) .
\]
This proves \eqref{eq:selection-tail} for every \(\ell\).

We next prove \eqref{eq:selection-normal-ineq}.  Let
\[
        U_{\theta,W,\ell}:=
        \sum_{i:J_i=\ell}\beta_i x_i^\theta
        =
        n u_{\theta,W,\ell}.
\]
For every \(u'\in K_{\widehat m_W}(\rho_\theta)\), feasibility of \(u'\) and
nonnegativity of \(\lambda_\theta\) imply
\[
        \lambda_\theta^\top A(u'-u_{\theta,W})
        \le
        \lambda_\theta^\top \rho_\theta-\lambda_\theta^\top A u_{\theta,W}
        =
        \frac1n\lambda_\theta^\top(b_\theta-AU_{\theta,W}) .
\]
The last term is zero by complementary slackness in the unnormalized pathwise
LP.  Equivalently,
\begin{equation}
        \sum_{\ell=1}^K ((a^\ell)^\top\lambda_\theta)(u'_\ell-u_{\theta,W,\ell})\le0 .
        \label{eq:dual-ineq-before-cap-v4}
\end{equation}
By definition,
\[
        q_{\theta,W,\ell}=\min\{(a^\ell)^\top\lambda_\theta,y\},
        \qquad \ell=1,\ldots,K .
\]
Thus \(q_{\theta,W,\ell}\le (a^\ell)^\top\lambda_\theta\) for every \(\ell\).
Strict inequality can occur only when \((a^\ell)^\top\lambda_\theta>y\).  In
that case every type-\(\ell\) item has negative reduced cost, so
\(u_{\theta,W,\ell}=0\).  Therefore, for every feasible \(u'\) and every
\(\ell\),
\[
        q_{\theta,W,\ell}(u'_\ell-u_{\theta,W,\ell})
        \le
        ((a^\ell)^\top\lambda_\theta)(u'_\ell-u_{\theta,W,\ell}) .
\]
Combining this inequality with \eqref{eq:dual-ineq-before-cap-v4} yields
\begin{equation}
        \sum_{\ell=1}^K q_{\theta,W,\ell}(u'_\ell-u_{\theta,W,\ell})\le0,
        \qquad u'\in K_{\widehat m_W}(\rho_\theta).
        \label{eq:ineq-after-cap-v4}
\end{equation}

Finally, we replace each coordinate of \(q_{\theta,W}\) by its projection onto
the corresponding active support.  Fix \(\ell\).  If
\(q_{\theta,W,\ell}>r_\ell^+\), then \((a^\ell)^\top\lambda_\theta\ge
q_{\theta,W,\ell}>r_\ell^+\); since every realized type-\(\ell\) ratio lies in
\([r_\ell^-,r_\ell^+]\), each type-\(\ell\) item has
\(V_i/\beta_i\le r_\ell^+<(a^\ell)^\top\lambda_\theta\), so complementary
slackness gives \(x_i^\theta=0\) and \(u_{\theta,W,\ell}=0\).  Hence
\(u'_\ell-u_{\theta,W,\ell}\ge0\) for every feasible \(u'\), and since
\(\Proj_\ell(q_{\theta,W,\ell})\le q_{\theta,W,\ell}\),
\[
        \Proj_\ell(q_{\theta,W,\ell})(u'_\ell-u_{\theta,W,\ell})
        \le
        q_{\theta,W,\ell}(u'_\ell-u_{\theta,W,\ell}) .
\]
If \(q_{\theta,W,\ell}<r_\ell^-\), then
\((a^\ell)^\top\lambda_\theta=q_{\theta,W,\ell}<r_\ell^-\); since every realized
type-\(\ell\) ratio lies in \([r_\ell^-,r_\ell^+]\), each type-\(\ell\) item has
\(V_i/\beta_i\ge r_\ell^->(a^\ell)^\top\lambda_\theta\), so complementary
slackness gives \(x_i^\theta=1\) and
\(u_{\theta,W,\ell}=\widehat m_{\ell,W}=\widehat C_{\ell,W}(r_\ell^-)\).  Every
feasible \(u'\) then satisfies \(u'_\ell\le\widehat m_{\ell,W}=u_{\theta,W,\ell}\),
and since \(\Proj_\ell(q_{\theta,W,\ell})\ge q_{\theta,W,\ell}\),
\[
        \Proj_\ell(q_{\theta,W,\ell})(u'_\ell-u_{\theta,W,\ell})
        \le
        q_{\theta,W,\ell}(u'_\ell-u_{\theta,W,\ell}) .
\]
If \(q_{\theta,W,\ell}\in S_\ell\), the two sides are equal.  Summing over
\(\ell\) and using \eqref{eq:ineq-after-cap-v4} proves
\eqref{eq:selection-normal-ineq}.
\end{proof}

We now prove the active-hull closure lemma used in Lemma~\ref{lem:jensen}.  For
\(q,\bar q\in[0,y]\) and a closed interval \(B\subset[0,y]\), define
\[
        I_B(q,\bar q):=[q\wedge \bar q,q\vee \bar q]\cap B
\]
and
\[
        d_B(q,\bar q):=\Leb(I_B(q,\bar q)).
\]

\begin{lemma}[Active-hull closure of a pairwise cap]\label{lem:hullcap}
Let \(\nu\) be a finite atomless measure on a closed interval \(B \subset[0,y]\),
and let \({\cal Q}\subset[0,y]\).  Suppose that, for some \(r>0\),
\begin{equation}
        d_B(q,\bar q)\nu(I_B(q,\bar q))\le r,
        \qquad q,\bar q\in{\cal Q}.
        \label{eq:pairwise-cap-abstract}
\end{equation}
Let
\[
        u:=\inf\{\Pi_B(q):q\in{\cal Q}\},
        \qquad
        v:=\sup\{\Pi_B(q):q\in{\cal Q}\}.
\]
Then
\begin{equation}
        (v-u)\nu([u,v])\le r .
        \label{eq:hull-cap-abstract}
\end{equation}
\end{lemma}

\begin{proof}
If \(u=v\), the claim is immediate.  Suppose that \(u<v\).

First consider the case in which both extremes are attained.  Thus there exist
\(q^-,q^+\in{\cal Q}\) such that
\[
        \Pi_B(q^-)=u,
        \qquad
        \Pi_B(q^+)=v .
\]
Then
\[
        I_B(q^-,q^+)=[u,v],
        \qquad
        d_B(q^-,q^+)=v-u .
\]
The pairwise cap \eqref{eq:pairwise-cap-abstract} gives
\[
        (v-u)\nu([u,v])
        =
        d_B(q^-,q^+)\nu(I_B(q^-,q^+))
        \le r .
\]

It remains to handle the case in which at least one extreme is not attained.
For every \(0<\eta<(v-u)/2\), the definitions of \(u\) and \(v\) give
\(q_\eta,\bar q_\eta\in{\cal Q}\) such that
\[
        \Pi_B(q_\eta)\le u+\eta,
        \qquad
        \Pi_B(\bar q_\eta)\ge v-\eta .
\]
Therefore
\[
        [u+\eta,v-\eta]\subseteq I_B(q_\eta,\bar q_\eta).
\]
It follows that
\[
        d_B(q_\eta,\bar q_\eta)\ge v-u-2\eta
\]
and
\[
        \nu(I_B(q_\eta,\bar q_\eta))
        \ge
        \nu([u+\eta,v-\eta]).
\]
Using \eqref{eq:pairwise-cap-abstract}, we obtain
\[
        (v-u-2\eta)\nu([u+\eta,v-\eta])\le r .
\]
Letting \(\eta\downarrow0\), and using that \(\nu\) is finite and atomless,
gives
\[
        (v-u)\nu([u,v])\le r .
\]
This proves \eqref{eq:hull-cap-abstract}.
\end{proof}

\subsection{Projected stability estimates}

This appendix proves the three fixed-polytope estimates used in the proof of
Proposition~\ref{prop:active}: a uniform optimal-face Hoffman bound, an
optimal value and solution sensitivity estimate, and the projected comparison for single-interval
systems.

\noindent\emph{Credit.}
The finite-switching path idea used below is inspired by the well-connected
polyhedral mapping viewpoint of \citet{CamachoCanovasGfrererParra2026} for
right-hand-side perturbations of linear-program argmin maps. Their paper
develops such perturbations through finite unions of convex polyhedral graph
pieces; the argument below uses only the classical Hoffman bound.

\begin{lemma}[Uniform optimal-face Hoffman bound]\label{lem:hoffman}
Let \(B\) be an \(m\times n\) matrix, and let
\(\mathcal C\subset\mathbb R^m\) and \(\mathcal S\subset\mathbb R^n\) be
compact.  Assume that, for every \(a\in\operatorname{conv}(\mathcal C)\),
\[
        P(a):=\{v\in\mathbb R^n:Bv\le a\}
\]
is nonempty, and that the family
\[
        \{P(a):a\in\operatorname{conv}(\mathcal C)\}
\]
is uniformly bounded.  Then there is a constant \(C<\infty\) such that the
following holds.  If \(c,c'\in\mathcal C\), \(s\in\mathcal S\),
\(\|c-c'\|_\infty\le\Delta\), and
\[
        x'\in\arg\max\{s^\top v:v\in P(c')\},
\]
then there exists
\[
        x\in\arg\max\{s^\top v:v\in P(c)\}
\]
such that
\[
        \|x-x'\|_1\le C\Delta .
\]
\end{lemma}

\begin{proof}
We apply Hoffman's bound to the optimal face induced by the price vector \(s\).
The main point is that the resulting constant can be chosen independently of
\(s\).  We use the classical Hoffman error bound for systems of linear
inequalities and equalities; see \cite{Hoffman1952}.

Write \(B_i\) for the \(i\)-th row of \(B\).  For each subset
\(D\subseteq\{1,\ldots,m\}\), define
\[
        P_D(a):=\{v:Bv\le a,\ B_Dv=a_D\},
\]
where the equality constraints are vacuous when \(D=\emptyset\).  Hoffman's
bound gives a constant \(H_D<\infty\), depending only on \(B\) and \(D\), such
that, whenever \(P_D(a)\neq\emptyset\),
\begin{equation}\label{eq:hoffman-bound}
        \operatorname{dist}_1(y,P_D(a))
        \le
        H_D
        \max\Bigl\{
              \|(By-a)^+\|_\infty,\,
              \|B_Dy-a_D\|_\infty
            \Bigr\}.
\end{equation}
Let
\[
        H:=\max_{D\subseteq\{1,\ldots,m\}}H_D<\infty .
\]

Fix \(s\in\mathcal S\).  For
\(a\in\operatorname{conv}(\mathcal C)\), write
\[
        \Phi(a,s):=\arg\max\{s^\top v:v\in P(a)\}.
\]
The defining linear program and its dual are
\[
        \max\{s^\top v:Bv\le a\}
        \qquad\text{and}\qquad
        \min\{a^\top\lambda:B^\top\lambda=s,\ \lambda\ge0\}.
\]
Let \(\mathcal D(s)\) be the finite family of subsets
\(D\subseteq\{1,\ldots,m\}\) for which there exist coefficients
\(\lambda_i>0\), \(i\in D\), satisfying
\[
        s=\sum_{i\in D}\lambda_i B_i^\top .
\]
These coefficients depend on \(s\) and \(B\), but not on the right-hand side
\(a\).  When \(s=0\), the empty set belongs to \(\mathcal D(s)\).

We claim that, for every \(D\in\mathcal D(s)\) and every
\(a\in\operatorname{conv}(\mathcal C)\) such that \(P_D(a)\neq\emptyset\),
\begin{equation}\label{eq:hoffman-face}
        P_D(a)=\Phi(a,s).
\end{equation}
Fix such \(D\) and \(a\), and let \(\lambda\) be the nonnegative vector
supported on \(D\) with \(B^\top\lambda=s\).  Then \(\lambda\) is feasible for
the dual problem.  If \(z\in P_D(a)\), then
\[
        s^\top z
        =
        \lambda^\top Bz
        =
        \lambda^\top a .
\]
By weak duality, \(z\) and \(\lambda\) are primal and dual optimal.  Hence
\(P_D(a)\subseteq\Phi(a,s)\).

Conversely, let \(v\in\Phi(a,s)\).  Since \(\lambda\) is dual optimal,
\[
        0
        =
        \lambda^\top a-s^\top v
        =
        \sum_{i\in D}\lambda_i(a_i-B_iv).
\]
Each term in the final sum is nonnegative, and each coefficient
\(\lambda_i\), \(i\in D\), is strictly positive.  Therefore \(B_iv=a_i\) for
every \(i\in D\), so \(v\in P_D(a)\).  This proves
\eqref{eq:hoffman-face}.

Now fix \(c,c'\in\mathcal C\), and define
\[
        \delta:=\|c-c'\|_\infty,
        \qquad
        \gamma(t):=c'+t(c-c'),\quad t\in[0,1].
\]
The sets \(P(\gamma(t))\), \(0\le t\le1\), are nonempty and lie in a common
compact set.  Indeed, \(\gamma(t)\in\operatorname{conv}(\mathcal C)\) for every
\(t\in[0,1]\), and the family \(P(a)\) is uniformly bounded over
\(\operatorname{conv}(\mathcal C)\).

For each \(D\in\mathcal D(s)\), define
\[
        I_D:=\{t\in[0,1]:P_D(\gamma(t))\neq\emptyset\}.
\]
We claim that each nonempty \(I_D\) is a closed interval.  Convexity follows by
taking convex combinations of feasible witnesses.  If
\(v_0\in P_D(\gamma(t_0))\) and \(v_1\in P_D(\gamma(t_1))\), then, for
\(\theta\in[0,1]\), the point
\[
        (1-\theta)v_0+\theta v_1
\]
lies in
\[
        P_D(\gamma((1-\theta)t_0+\theta t_1)).
\]
This is because both the inequalities \(Bv\le\gamma(t)\) and the equalities
\(B_Dv=\gamma_D(t)\) are preserved under convex combination, while
\(\gamma(t)\) varies affinely in \(t\).  To prove closedness, take
\(t_r\in I_D\) with \(t_r\to t\), and choose
\(v_r\in P_D(\gamma(t_r))\).  The points \(v_r\) lie in the common compact set,
so a subsequence converges to some \(v\).  Passing to the limit in the
inequalities and equalities gives \(v\in P_D(\gamma(t))\).  Hence \(t\in I_D\).

The finitely many intervals \(\{I_D:D\in\mathcal D(s)\}\) cover \([0,1]\).
To see this, fix \(t\in[0,1]\), and choose
\(v_t\in\Phi(\gamma(t),s)\).  By strong linear-program duality, there is a dual
optimal solution \(\lambda(t)\ge0\) such that
\[
        B^\top\lambda(t)=s,
        \qquad
        \lambda_i(t)\bigl(\gamma_i(t)-B_i v_t\bigr)=0
        \quad\text{for all }i .
\]
Let
\[
        D(t):=\{i:\lambda_i(t)>0\}.
\]
Then \(D(t)\in\mathcal D(s)\).  Complementary slackness gives
\(B_iv_t=\gamma_i(t)\) for every \(i\in D(t)\), and therefore
\(v_t\in P_{D(t)}(\gamma(t))\).  Thus \(t\in I_{D(t)}\).

Since a finite family of closed intervals covers \([0,1]\), we may choose
points
\[
        0=t_0<t_1<\cdots<t_N=1
\]
and subsets \(D_1,\ldots,D_N\in\mathcal D(s)\) such that
\[
        [t_{j-1},t_j]\subseteq I_{D_j},
        \qquad j=1,\ldots,N .
\]

Set \(x_0:=x'\).  We construct points
\(x_j\in\Phi(\gamma(t_j),s)\) inductively.  Suppose that
\(x_{j-1}\in\Phi(\gamma(t_{j-1}),s)\).  Since
\([t_{j-1},t_j]\subseteq I_{D_j}\), \eqref{eq:hoffman-face} gives
\[
        P_{D_j}(\gamma(t_{j-1}))=\Phi(\gamma(t_{j-1}),s),
        \qquad
        P_{D_j}(\gamma(t_j))\neq\emptyset .
\]
Thus \(x_{j-1}\in P_{D_j}(\gamma(t_{j-1}))\).  Applying
\eqref{eq:hoffman-bound} to the nonempty target set
\(P_{D_j}(\gamma(t_j))\), we obtain a point
\(x_j\in P_{D_j}(\gamma(t_j))\) such that
\[
        \|x_j-x_{j-1}\|_1
        \le
        H\|\gamma(t_j)-\gamma(t_{j-1})\|_\infty
        =
        H(t_j-t_{j-1})\delta .
\]
Because \(t_j\in I_{D_j}\), this point also satisfies
\(x_j\in\Phi(\gamma(t_j),s)\) by \eqref{eq:hoffman-face}.

Summing over \(j=1,\ldots,N\), we get
\[
        \|x_N-x'\|_1
        \le
        H\delta\sum_{j=1}^N(t_j-t_{j-1})
        =
        H\delta .
\]
Since \(\gamma(1)=c\), the point \(x:=x_N\) belongs to \(\Phi(c,s)\).  Taking
\(C:=H\) yields
\[
        \|x-x'\|_1\le C\|c-c'\|_\infty\le C\Delta,
\]
which completes the proof.
\end{proof}

\begin{lemma}[Optimal value and solution sensitivity estimate]\label{lem:fourpoint}
Fix compact coordinatewise ranges for \(m\) and \(\rho\) such that
\(K_m(\rho)\neq\emptyset\) for every \((m,\rho)\) in these ranges.  For
\(r\in[0,y]^K\), define
\[
        h_{m,\rho}(r):=\max\{r^\top v:v\in K_m(\rho)\}.
\]
Then there is a constant \(C<\infty\) such that, for all admissible
\(m,\bar m,\rho,\bar\rho\) and all \(r,\bar r\in[0,y]^K\),
\[
        h_{m,\rho}(r)+h_{\bar m,\bar\rho}(\bar r)
        -h_{\bar m,\bar\rho}(r)-h_{m,\rho}(\bar r)
        \ge
        -C\|\rho-\bar\rho\|_\infty
        -C\|m-\bar m\|_\infty\,\|r-\bar r\|_1 .
\]
Consequently, if \(u\) maximizes \(r^\top v\) over \(K_m(\rho)\) and
\(\bar u\) maximizes \(\bar r^\top v\) over \(K_{\bar m}(\bar\rho)\), then
\[
        (r-\bar r)^\top(u-\bar u)
        \ge
        -C\|\rho-\bar\rho\|_\infty
        -C\|m-\bar m\|_\infty\,\|r-\bar r\|_1 .
\]
\end{lemma}

\begin{proof}
Let
\[
        B_0:=
        \begin{pmatrix}
        I\\ -I\\ A
        \end{pmatrix},
        \qquad
        c(m,\rho):=
        \begin{pmatrix}
        m\\ 0\\ \rho
        \end{pmatrix}.
\]
Then \(K_m(\rho)=\{v:B_0v\le c(m,\rho)\}\).  Applying
Lemma~\ref{lem:hoffman} to the fixed matrix \(B_0\) and to the compact ranges
under consideration gives a constant \(L<\infty\) with the following property.
If \(\|c(m,\rho)-c(m',\rho')\|_\infty\le\Delta\), and if \(z'\) maximizes
\(s^\top x\) over \(K_{m'}(\rho')\) for some \(s\in[0,y]^K\), then there is a
maximizer \(z\) of \(s^\top x\) over \(K_m(\rho)\) such that
\begin{equation}\label{eq:fourpoint-optface}
        \|z-z'\|_1\le L\Delta .
\end{equation}

We first vary only the resource right-hand side.  Fix \(m\) and
\(s\in[0,y]^K\).  If \(w'\in\arg\max_{K_m(\rho')}s^\top x\), then
\eqref{eq:fourpoint-optface} gives
\(w\in\arg\max_{K_m(\rho)}s^\top x\) with
\(\|w-w'\|_1\le L\|\rho-\rho'\|_\infty\).  Therefore
\[
        h_{m,\rho}(s)
        =s^\top w
        \ge
        s^\top w'-\|s\|_\infty\|w-w'\|_1
        \ge
        h_{m,\rho'}(s)-yL\|\rho-\rho'\|_\infty .
\]
Interchanging \(\rho\) and \(\rho'\) yields
\begin{equation}\label{eq:fourpoint-rhs-lipschitz}
        |h_{m,\rho}(s)-h_{m,\rho'}(s)|
        \le yL\|\rho-\rho'\|_\infty .
\end{equation}

Define
\[
        F(m,\rho;\bar m,\bar\rho;r,\bar r)
        :=
        h_{m,\rho}(r)+h_{\bar m,\bar\rho}(\bar r)
        -h_{\bar m,\bar\rho}(r)-h_{m,\rho}(\bar r).
\]
Applying \eqref{eq:fourpoint-rhs-lipschitz} twice gives
\begin{equation}\label{eq:fourpoint-rhs-reduction}
        F(m,\rho;\bar m,\bar\rho;r,\bar r)
        \ge
        F(m,\rho;\bar m,\rho;r,\bar r)
        -2yL\|\rho-\bar\rho\|_\infty .
\end{equation}
It remains to bound \(F(m,\rho;\bar m,\rho;r,\bar r)\).

Consider first a one-coordinate change in \(m\).  Suppose that
\(m^+=m+\eta e_j\), where \(\eta\ge0\), and set
\(G_j(s):=h_{m^+,\rho}(s)-h_{m,\rho}(s)\).  We claim that
\begin{equation}\label{eq:fourpoint-onecoord}
        |G_j(r)-G_j(\bar r)|
        \le L\eta\|r-\bar r\|_1 .
\end{equation}
Let \(d:=r-\bar r\) and \(s_t:=\bar r+td\), \(0\le t\le1\).  The functions
\(g^+(t):=h_{m^+,\rho}(s_t)\) and \(g(t):=h_{m,\rho}(s_t)\) are Lipschitz in
\(t\), and hence are differentiable for almost every \(t\).

We use the standard derivative formula for support functions.  If
\(h_K(s)=\max_{v\in K}s^\top v\) and \(X_K(s)\) is its optimizer set, then, at
every \(t\) where \(h_K(s_t)\) is differentiable,
\begin{equation}\label{eq:fourpoint-support-derivative}
        \frac{d}{dt}h_K(s_t)=d^\top x
        \qquad\text{for every }x\in X_K(s_t).
\end{equation}
Indeed, the right derivative is
\(\max_{x\in X_K(s_t)}d^\top x\), the left derivative is
\(\min_{x\in X_K(s_t)}d^\top x\), and these two values coincide at points of
differentiability.

Fix a point \(t\) at which both \(g^+\) and \(g\) are differentiable, and choose
\(x^+(t)\in\arg\max_{v\in K_{m^+}(\rho)}s_t^\top v\).  The right-hand sides
\(c(m^+,\rho)\) and \(c(m,\rho)\) differ only in the \(j\)-th upper-bound
coordinate, by \(\eta\).  Thus \eqref{eq:fourpoint-optface} gives a point
\(x(t)\in\arg\max_{v\in K_m(\rho)}s_t^\top v\) with
\(\|x^+(t)-x(t)\|_1\le L\eta\).  Using
\eqref{eq:fourpoint-support-derivative}, we obtain, for almost every \(t\),
\[
        \left|\frac{d}{dt}\{g^+(t)-g(t)\}\right|
        =
        |d^\top(x^+(t)-x(t))|
        \le
        \|d\|_\infty\|x^+(t)-x(t)\|_1
        \le
        L\eta\|d\|_1 .
\]
Integrating over \(t\in[0,1]\) proves \eqref{eq:fourpoint-onecoord}.

For this one-coordinate increase,
\[
\begin{aligned}
        &h_{m,\rho}(r)+h_{m^+,\rho}(\bar r)
        -h_{m^+,\rho}(r)-h_{m,\rho}(\bar r)  \\
        &\qquad =
        G_j(\bar r)-G_j(r)
        \ge
        -L\eta\|r-\bar r\|_1 .
\end{aligned}
\]
The same bound holds for a coordinate decrease, after interchanging the two
vectors.

Now connect \(m\) to \(\bar m\) one coordinate at a time.  Let \(m^0=m\) and
\(m^K=\bar m\), where
\[
        m^j_i=
        \begin{cases}
        \bar m_i, & i\le j,\\
        m_i, & i>j,
        \end{cases}
        \qquad j=1,\ldots,K .
\]
The vectors \(m^{j-1}\) and \(m^j\) differ only in coordinate \(j\), by
\(\eta_j:=|\bar m_j-m_j|\).  Summing the one-coordinate estimate gives
\[
\begin{aligned}
        &h_{m,\rho}(r)+h_{\bar m,\rho}(\bar r)
        -h_{\bar m,\rho}(r)-h_{m,\rho}(\bar r)  \\
        &\qquad\ge
        -L\|r-\bar r\|_1\sum_{j=1}^K|\bar m_j-m_j|
        \ge
        -LK\|m-\bar m\|_\infty\|r-\bar r\|_1 .
\end{aligned}
\]
Combining this estimate with \eqref{eq:fourpoint-rhs-reduction}, and absorbing
\(2yL\) and \(LK\) into a single constant \(C\), gives the desired four-point
estimate.

It remains to prove the final displayed inequality in the lemma.  By the
definitions of \(u\) and \(\bar u\), we have
\(r^\top u=h_{m,\rho}(r)\) and
\(\bar r^\top\bar u=h_{\bar m,\bar\rho}(\bar r)\).  Since
\(\bar u\in K_{\bar m}(\bar\rho)\) and \(u\in K_m(\rho)\),
\[
        r^\top\bar u\le h_{\bar m,\bar\rho}(r),
        \qquad
        \bar r^\top u\le h_{m,\rho}(\bar r).
\]
Therefore
\[
\begin{aligned}
        (r-\bar r)^\top(u-\bar u)
        &=
        r^\top u+\bar r^\top\bar u-r^\top\bar u-\bar r^\top u \\
        &\ge
        h_{m,\rho}(r)+h_{\bar m,\bar\rho}(\bar r)
        -h_{\bar m,\bar\rho}(r)-h_{m,\rho}(\bar r).
\end{aligned}
\]
The four-point estimate gives the claimed bound.
\end{proof}

\begin{lemma}[Projected comparison for single-interval systems]\label{lem:projcomp}
Assume that each active support set \(S_k\) is a single interval in \([0,y]\).
For every \(C_m<\infty\), there is a constant \(C<\infty\) such that the
following holds.  Let \(m,\bar m,\rho,\bar\rho\) lie in the compact ranges of
Lemma~\ref{lem:fourpoint}.  Let \(u\in K_m(\rho)\),
\(\bar u\in K_{\bar m}(\bar\rho)\), and let \(q,\bar q\in[0,y]^K\).  Set
\(p_k=\Pi_k(q_k)\) and \(\bar p_k=\Pi_k(\bar q_k)\) for
\(k=1,\ldots,K\).  If
\[
        p\in N_{K_m(\rho)}(u),
        \qquad
        \bar p\in N_{K_{\bar m}(\bar\rho)}(\bar u),
\]
and
\[
        \|m-\bar m\|_\infty\le C_m\eps,
        \qquad
        \|\rho-\bar\rho\|_\infty\le\tau,
\]
then
\[
        (p-\bar p)^\top(u-\bar u)
        \ge
        -C\tau
        -C\eps\sum_{k=1}^K\ell_k(q_k,\bar q_k),
\]
where
\[
        \ell_k(a,b):=
        \Leb\bigl([a\wedge b,a\vee b]\cap S_k\bigr).
\]
\end{lemma}

\begin{proof}
For a convex set \(K\) and a point \(x\in K\), write
\(N_K(x):=\{z:z^\top(v-x)\le0\ \text{for all }v\in K\}\).  Thus
\(p\in N_{K_m(\rho)}(u)\) means that \(u\) maximizes \(p^\top v\) over
\(K_m(\rho)\).  Similarly, \(\bar u\) maximizes \(\bar p^\top v\) over
\(K_{\bar m}(\bar\rho)\).  Since each projection \(\Pi_k\) maps into
\(S_k\subseteq[0,y]\), both \(p\) and \(\bar p\) belong to \([0,y]^K\).

Apply Lemma~\ref{lem:fourpoint} with \(r=p\) and \(\bar r=\bar p\).  If
\(C_4\) is the constant in that lemma, then
\[
        (p-\bar p)^\top(u-\bar u)
        \ge
        -C_4\|\rho-\bar\rho\|_\infty
        -C_4\|m-\bar m\|_\infty\|p-\bar p\|_1 .
\]
Using the assumed bounds on \(m-\bar m\) and \(\rho-\bar\rho\), we get
\begin{equation}\label{eq:projcomp-fourpoint-bound}
        (p-\bar p)^\top(u-\bar u)
        \ge
        -C_4\tau
        -C_4C_m\eps\|p-\bar p\|_1 .
\end{equation}

It remains to express \(\|p-\bar p\|_1\) in terms of the portions of the
intervals between \(q_k\) and \(\bar q_k\) that lie inside \(S_k\).  Write
\(S_k=[s_k^-,s_k^+]\).  For \(a\le b\), the projection \(\Pi_k\) is constant on
\((-\infty,s_k^-]\) and on \([s_k^+,\infty)\), and it equals the identity on
\([s_k^-,s_k^+]\).  Therefore its variation over \([a,b]\) is exactly
\(\Leb([a,b]\cap S_k)\).  By symmetry, for all \(a,b\in[0,y]\),
\begin{equation}\label{eq:projcomp-projection-length}
        |\Pi_k(a)-\Pi_k(b)|
        =
        \Leb\bigl([a\wedge b,a\vee b]\cap S_k\bigr)
        =
        \ell_k(a,b).
\end{equation}
Applying \eqref{eq:projcomp-projection-length} with \(a=q_k\) and
\(b=\bar q_k\), and summing over \(k\), gives
\[
        \|p-\bar p\|_1
        =
        \sum_{k=1}^K \ell_k(q_k,\bar q_k).
\]
Substituting this identity into \eqref{eq:projcomp-fourpoint-bound}, and then
renaming constants, proves the claim.
\end{proof}

\subsection{Endpoint-contact verification}\label{app:endpoint-contact}

This part completes the proof of the main bound when \(\pp>1\).  The
argument for \(\pp=1\) reduced the per-stage marginal loss to an active
conditional-curvature product.  On a dominated neighborhood, pointwise domination
turns that product into the active weighted-resource product controlled by
Proposition~\ref{prop:active}.  On an endpoint-contact neighborhood, pointwise
domination can fail: for sizes near the size boundary, the conditional curvature
\(\Lambda_{k,z}\) may concentrate near a receding edge
\(e(\omega)\asymp \omega^\tau\).  A short ratio interval can then have
conditional curvature of lower order than its size-averaged mass.

The active-mass product cap is still available, because
Proposition~\ref{prop:active} is deterministic and uses only projected active
intervals.  The remaining task is to integrate the conditional curvature along
the endpoint branch.  The Hardy estimate below does this.  Under the active-mass
product cap, the size-integrated curvature is at most
\(C r_s\log(e/r_s)\), so the endpoint branch costs only a logarithmic factor
relative to the dominated case.

Assumption~\ref{ass:regularity} is in force throughout this subsection.  The
deterministic active-mass estimate has already used only the active mass bound
\eqref{eq:wr} and the single-interval active supports.  The endpoint step below
uses the endpoint-contact neighborhoods in Assumption~\ref{ass:regularity}.

On an endpoint-contact neighborhood, the branch has local exponent
\(\theta=\gamma+\alpha/\tau\).  Lemma~\ref{lem:mass} gives the corresponding
local branch mass growth.  The dominated remainder has, by
\eqref{eq:dominated-component}, no larger endpoint growth than this branch.
Thus the full weighted-ratio measure has local endpoint exponent \(\theta\) at
that contact endpoint.  The global active-mass exponent \(\pp\) in
\eqref{eq:wr} need not equal this particular \(\theta\).  Any admissible global
exponent must satisfy \(\theta\le\pp\), and equality holds only at endpoints that
realize the worst local exponent when \(\pp\) is chosen as the smallest
admissible exponent.  The Hardy bound below is stated for a general \(\theta\),
so the proof does not require the identification \(\theta=\pp\).

Throughout this subsection, \(C\) denotes a constant that depends only on the
primitive endpoint constants, the number of types, the direction matrix, and the
boundedness constants.  It never depends on \(T\), \(s\), or the current
capacity.

\subsubsection*{A. The endpoint Hardy bound}

We now convert the active-mass product cap into a bound on the size-integrated
conditional curvature.  First we isolate the branch disintegration supplied by
the endpoint-contact assumption.  Then we record the two one-branch estimates
used in the Hardy step: the mass contributed to \(\mu_k\), and the Jensen
estimate for the branchwise surplus.

\begin{lemma}[Kernel branch disintegration]\label{lem:branchdis}
Fix a type \(k\) and one endpoint-contact neighborhood \(U\) from
Assumption~\ref{ass:regularity}.  For every nonnegative measurable functional
\(\mathcal A(z,\Lambda)=\int_U a(z,r)\,\Lambda(\dd r)\) that is linear in the
curvature measure \(\Lambda\), there are constants \(0<c\le C<\infty\), depending
only on the primitive bounds on the branch weight \(w\), such that
\[
        c
        \int_0^{\omega_0}
        \mathcal A\bigl(\beta_k(\omega),\Lambda^{\mathrm{br}}_{k,\omega}\bigr)
        f(\omega)\,\dd\omega
        \le
        \pi_k
        \int_{\mathcal B_k}
        \mathcal A\bigl(z,\Lambda^{\mathrm{br}}_{k,z}\bigr)
        P_k^\beta(\dd z)
        \le
        C
        \int_0^{\omega_0}
        \mathcal A\bigl(\beta_k(\omega),\Lambda^{\mathrm{br}}_{k,\omega}\bigr)
        f(\omega)\,\dd\omega .
\]
\end{lemma}

\begin{proof}
The endpoint-contact representation decomposes the kernel measure on
\(\mathcal B_k\times U\) as
\(\mathfrak M_k|_{\mathcal B_k\times U}=\mathfrak M_k^D+\mathfrak M_k^{\mathrm{br}}\),
with subkernels \(\Lambda^D_{k,z}\) and \(\Lambda^{\mathrm{br}}_{k,z}\) as in
\eqref{eq:endpoint-kernel-split}.  Applying the branch identity
\eqref{eq:branch-disint-def} to the nonnegative kernel \(g(z,r)=a(z,r)\) gives the
exact disintegration
\[
        \pi_k
        \int_{\mathcal B_k}
        \mathcal A\bigl(z,\Lambda^{\mathrm{br}}_{k,z}\bigr)
        P_k^\beta(\dd z)
        =
        \int_0^{\omega_0}
        w(\omega)\,
        \mathcal A\bigl(\beta_k(\omega),\Lambda^{\mathrm{br}}_{k,\omega}\bigr)
        f(\omega)\,\dd\omega ,
\]
and the comparison follows from the primitive upper and lower bounds on \(w\).
\end{proof}

\begin{lemma}[Endpoint-contact mass bounds]\label{lem:mass}
For the endpoint branch of type \(k\), recall
\(\theta=\gamma+\alpha/\tau\) from Definition~\ref{def:branch}.  Orient the
local ratio coordinate into the support: \(x=R-r_{k,0}\) at a lower endpoint and
\(x=r_{k,0}-R\) at an upper endpoint.  Then the branch contribution
\(\mu^{\mathrm{br}}_k\) has Lebesgue density
\(m^{\mathrm{br}}_k(x)\asymp x^{\theta-1}\) on \((0,x_0)\).  Hence, for every
interval \([a,b]\subseteq[0,x_0]\),
\[
        c(b-a)b^{\theta-1}
        \le
        \mu^{\mathrm{br}}_k(\{a\le x\le b\})
        \le
        C(b-a)b^{\theta-1},
        \qquad
        \mu^{\mathrm{br}}_k(\{0\le x\le x_1\})\asymp x_1^\theta .
\]
\end{lemma}

\begin{proof}
By the branch formula in Definition~\ref{def:branch},
\[
        \mu_k^{\mathrm{br}}(I)
        =
        \int_0^{\omega_0}
        w(\omega)\Lambda^{\mathrm{br}}_{k,\omega}(I)
        f(\omega)\dd\omega .
\]
Since \(w\) is bounded above and below by primitive constants, it can be
absorbed into the comparability constants.  Thus the Lebesgue density of
\(\mu_k^{\mathrm{br}}\) satisfies
\[
        m^{\mathrm{br}}_k(x)\asymp
        \int_0^{\omega_0}
        \lambda^{\mathrm{br}}_{k,\omega}(x)f(\omega)\dd\omega,
        \qquad
        f(\omega)\asymp \omega^{\alpha-1}.
\]

The integrand is zero unless \(e(\omega)\le x\).  Since
\(e(\omega)\asymp\omega^\tau\), the relevant range is
\(0\le\omega\le Cx^{1/\tau}\).  The upper density bound therefore gives
\[
        m^{\mathrm{br}}_k(x)
        \le
        Cx^{\gamma-1}
        \int_0^{Cx^{1/\tau}}\omega^{\alpha-1}\dd\omega
        \le
        Cx^{\theta-1}.
\]
For the lower bound, restrict to \(\omega\le c\,x^{1/\tau}\).  Since
\(e(\omega)\le C_e\,\omega^\tau\) for a primitive constant \(C_e\), and
\(x_0,\omega_0\) are primitive, the constant \(c\) may be chosen depending only on
the primitive constants---small enough that \(e(\omega)\le x/2\) and
\(c\,x^{1/\tau}\le\omega_0\) for every \(x\in(0,x_0]\).  Since \(\gamma\ge1\),
\(\lambda^{\mathrm{br}}_{k,\omega}(x)\ge c(x-e(\omega))^{\gamma-1}
\ge c x^{\gamma-1}\).  The same integration gives
\(m^{\mathrm{br}}_k(x)\ge c x^{\theta-1}\).

Integrating the density over \([a,b]\) gives the interval bound.  Indeed, since
\(\theta\ge1\), \(\int_a^b x^{\theta-1}\dd x\asymp (b-a)b^{\theta-1}\).  Taking
\(a=0\) and \(b=x_1\) gives
\(\mu^{\mathrm{br}}_k(\{0\le x\le x_1\})\asymp x_1^\theta\).
\end{proof}

\begin{lemma}[Endpoint Hardy bound under an active-mass product cap]\label{lem:hardy}
For the endpoint branch of type \(k\), recall that
\(\theta=\gamma+\alpha/\tau\).  There are constants \(r_0'>0\) and
\(C<\infty\), depending only on the primitive endpoint constants, such that the
following holds.  Fix \(0<r\le r_0'\).  For each size coordinate \(\omega\), let
\(I_\omega\) be a local-coordinate ratio interval, and write
\(d_\omega:=\Leb(I_\omega\cap[e(\omega),x_0])\),
\(\Lambda_\omega(I_\omega):=\Lambda^{\mathrm{br}}_{k,\omega}(I_\omega)\), and
\(\mu(I_\omega):=\mu^{\mathrm{br}}_k(I_\omega)\).  If the active-mass product cap
\begin{equation}
        d_\omega\mu(I_\omega)\le r
        \qquad
        \text{for a.e. }\omega\in[0,\omega_0]
        \label{eq:hardy-active-cap}
\end{equation}
holds, then
\begin{equation}
        \int_0^{\omega_0}
        d_\omega\Lambda_\omega(I_\omega)f(\omega)\,\dd\omega
        \le
        C r\log(e/r).
        \label{eq:hardy-final}
\end{equation}
\end{lemma}

\begin{proof}
Work in the oriented local coordinate, so the moving active interval is
\([e(\omega),x_0]\), with \(e(\omega)\asymp\omega^\tau\).  If
\(I_\omega\cap[e(\omega),x_0]=\emptyset\), then \(d_\omega=0\), and the
corresponding integrand is zero.  Hence we may restrict to \(\omega\)'s for
which this intersection is nonempty.

For such an \(\omega\), set
\(B_\omega:=\sup(I_\omega\cap[e(\omega),x_0])\).  Since \(I_\omega\) is an
interval, the active intersection is, up to endpoints, an interval
\([a_\omega,B_\omega]\).  Endpoint conventions do not matter because the branch
measures have densities.  Thus \(d_\omega=B_\omega-a_\omega\), so
\(d_\omega\le B_\omega\).  Also, because \(B_\omega\ge e(\omega)\) and
\(e(\omega)\asymp\omega^\tau\),
\begin{equation}
        B_\omega\ge c\,\omega^\tau .
        \label{eq:hardy-B-lower}
\end{equation}

Lemma~\ref{lem:mass}, applied to \([a_\omega,B_\omega]\), gives
\(\mu^{\mathrm{br}}_k([a_\omega,B_\omega])\ge
c\,d_\omega B_\omega^{\theta-1}\).  Since
\([a_\omega,B_\omega]\subseteq I_\omega\) and
\(\mu=\mu^{\mathrm{br}}_k\) is positive, we also have
\(\mu(I_\omega)\ge c\,d_\omega B_\omega^{\theta-1}\).  Combining this lower
bound with \eqref{eq:hardy-active-cap} yields
\begin{equation}
        d_\omega^2 B_\omega^{\theta-1}\le C r .
        \label{eq:hardy-length-constraint}
\end{equation}

We next bound the conditional curvature term.  By the endpoint-branch density
assumption, the density of \(\Lambda^{\mathrm{br}}_{k,\omega}\) satisfies
\(\lambda^{\mathrm{br}}_{k,\omega}(x)\le
C(x-e(\omega))^{\gamma-1}\) for \(x\ge e(\omega)\).  Since \(\gamma\ge1\), we
have \((x-e(\omega))^{\gamma-1}\le B_\omega^{\gamma-1}\) for
\(x\in[a_\omega,B_\omega]\).  Therefore
\(\Lambda_\omega(I_\omega)\le C d_\omega B_\omega^{\gamma-1}\).  Multiplying by
\(d_\omega\) and using \eqref{eq:hardy-length-constraint}, we get
\[
        d_\omega\Lambda_\omega(I_\omega)
        \le
        C d_\omega^2 B_\omega^{\gamma-1}
        =
        C d_\omega^2 B_\omega^{\theta-1}B_\omega^{\gamma-\theta}
        \le
        C r B_\omega^{\gamma-\theta}.
\]
Since \(\theta=\gamma+\alpha/\tau\), this becomes
\begin{equation}
        d_\omega\Lambda_\omega(I_\omega)
        \le
        C r B_\omega^{-\alpha/\tau}.
        \label{eq:hardy-decreasing}
\end{equation}

We also need a crude increasing bound.  Since \(d_\omega\le B_\omega\),
\begin{equation}
        d_\omega\Lambda_\omega(I_\omega)
        \le
        C d_\omega^2 B_\omega^{\gamma-1}
        \le
        C B_\omega^{\gamma+1}.
        \label{eq:hardy-increasing}
\end{equation}
Combining \eqref{eq:hardy-decreasing} and \eqref{eq:hardy-increasing}, and then
using \eqref{eq:hardy-B-lower}, gives the pointwise bound
\begin{equation}
        d_\omega\Lambda_\omega(I_\omega)
        \le
        C\sup_{B\ge c\omega^\tau}
        \min\{B^{\gamma+1},\,rB^{-\alpha/\tau}\}.
        \label{eq:hardy-pointwise-sup}
\end{equation}

Let \(B_*:=r^{1/(\theta+1)}\).  This is the balancing scale, since
\(B_*^{\gamma+1}=rB_*^{-\alpha/\tau}\).  Choose \(r_0'>0\) small enough that
\(B_*\le x_0\) whenever \(0<r\le r_0'\).  Define
\begin{equation}
        \omega_*:=c_* B_*^{1/\tau},
        \label{eq:hardy-omegastar}
\end{equation}
where \(c_*>0\) is a sufficiently small primitive constant.  Reducing \(r_0'\)
again if necessary, we may assume that \(\omega_*\le\omega_0\).

First consider \(0\le\omega\le\omega_*\).  By the choice of \(c_*\), we have
\(c\omega^\tau\le B_*\).  Since \(B\mapsto B^{\gamma+1}\) is increasing and
\(B\mapsto rB^{-\alpha/\tau}\) is decreasing, the supremum in
\eqref{eq:hardy-pointwise-sup} is attained, up to primitive constants, at the
balancing scale \(B_*\).  Thus
\(d_\omega\Lambda_\omega(I_\omega)\le C B_*^{\gamma+1}\) for
\(0\le\omega\le\omega_*\).  Using \(f(\omega)\le C\omega^{\alpha-1}\), we obtain
\[
        \int_0^{\omega_*}
        d_\omega\Lambda_\omega(I_\omega)f(\omega)\,\dd\omega
        \le
        C B_*^{\gamma+1}\int_0^{\omega_*}\omega^{\alpha-1}\,\dd\omega
        \le
        C B_*^{\gamma+1}\omega_*^\alpha .
\]
By \eqref{eq:hardy-omegastar}, \(\omega_*^\alpha\le C B_*^{\alpha/\tau}\).
Hence
\begin{equation}
        \int_0^{\omega_*}
        d_\omega\Lambda_\omega(I_\omega)f(\omega)\,\dd\omega
        \le
        C B_*^{\gamma+1+\alpha/\tau}
        =
        C B_*^{\theta+1}
        =
        C r .
        \label{eq:hardy-small-omega-integral}
\end{equation}

Now consider \(\omega>\omega_*\).  Then \(c\omega^\tau\ge c'B_*\), after
adjusting primitive constants.  The supremum in
\eqref{eq:hardy-pointwise-sup} is therefore bounded by the decreasing branch:
\[
        \sup_{B\ge c\omega^\tau}\min\{B^{\gamma+1},\,rB^{-\alpha/\tau}\}
        \le
        C r(c\omega^\tau)^{-\alpha/\tau}
        \le
        C r\omega^{-\alpha}.
\]
Together with \eqref{eq:hardy-pointwise-sup}, this gives
\(d_\omega\Lambda_\omega(I_\omega)\le C r\omega^{-\alpha}\) for
\(\omega_*<\omega\le\omega_0\).  Since
\(f(\omega)\le C\omega^{\alpha-1}\),
\[
        \int_{\omega_*}^{\omega_0}
        d_\omega\Lambda_\omega(I_\omega)f(\omega)\,\dd\omega
        \le
        C r\int_{\omega_*}^{\omega_0}\omega^{-1}\,\dd\omega
        \le
        C r\log(e/\omega_*).
\]
Because \(\omega_*=c_*B_*^{1/\tau}
=c_*r^{1/(\tau(\theta+1))}\), we have
\(\log(e/\omega_*)\le C\log(e/r)\).  Therefore
\begin{equation}
        \int_{\omega_*}^{\omega_0}
        d_\omega\Lambda_\omega(I_\omega)f(\omega)\,\dd\omega
        \le
        C r\log(e/r).
        \label{eq:hardy-large-omega-final}
\end{equation}

Combining \eqref{eq:hardy-small-omega-integral} and
\eqref{eq:hardy-large-omega-final} yields \eqref{eq:hardy-final}, because
\(0<r\le r_0'\) and \(r_0'\) is fixed.
\end{proof}

The scale \(B_*=r^{1/(\theta+1)}\) balances the increasing and decreasing
branches of the size integral.  The boundary term
\(\log(e/\omega_*)\asymp\log(e/r)\) is the only logarithmic loss; it is the price
of integrating the conditional curvature against the size distribution rather
than dominating it pointwise.

\subsubsection*{Proof of Lemma~\ref{lem:epstability}}

\begin{proof}[Proof of Lemma~\ref{lem:epstability}]
We prove the bound in four steps.  First, Proposition~\ref{prop:active} gives a
pathwise active-mass product cap for all good future paths.  Second, the cap is
extended from pairs of cutoffs to the active hull of the good cutoff support.
Third, dominated neighborhoods and endpoint-contact neighborhoods are treated
separately.  Finally, we average over the current size and remove the
conditioning on the good future event.

Let \(n=s-1\), and set
\(\delta_n:=C\sqrt{\log(en)/n}\), as in the concentration event
\eqref{eq:tailconc}.  Define the deterministic cap scale
\[
        \bar r_s
        :=
        C_{\rm cap}\left(\delta_n^{\,1+1/\pp}+\frac1n\right),
\]
where \(C_{\rm cap}\) is large enough to dominate the constant obtained when
Proposition~\ref{prop:active} is applied below.  Since
\(\delta_n^{\,1+1/\pp}
=C(\log(en)/n)^{(\pp+1)/(2\pp)}\), after increasing constants we have
\begin{equation}
        \bar r_s\le C r_s .
        \label{eq:epstage-rbar-rs}
\end{equation}

Let \(\lambda_{\rm ch}>0\) be a Lebesgue number for the finite cover of the
active supports \(S_k\) by dominated neighborhoods and endpoint-contact
neighborhoods, chosen uniformly over all types and cover elements.  By the lower
active-mass bound in \eqref{eq:wr}, every active interval \(I\) with active
length at least \(\lambda_{\rm ch}\) satisfies
\begin{equation}
        \Leb(I\cap S_k)\,\mu_k(I)
        \ge
        c\,\lambda_{\rm ch}^{\pp+1}.
        \label{eq:epstage-large-interval-product}
\end{equation}
Let \(r_{\rm H}>0\) be the minimum of the small-radius constants in
Lemma~\ref{lem:hardy} over the finitely many endpoint-contact neighborhoods.  If
there are no endpoint-contact neighborhoods, set \(r_{\rm H}:=1\); then the
contact-branch part of the proof is vacuous.  Reduce \(r_{\rm H}\), if needed,
so that
\begin{equation}
        r_{\rm H}\le \frac12 c\,\lambda_{\rm ch}^{\pp+1}.
        \label{eq:epstage-hardy-radius}
\end{equation}
Choose \(s_0\) so large that, for all \(s\ge s_0\), we have
\(\delta_n\le\eps_0\), \(\bar r_s\le r_{\rm H}\), and
\(\Prob(E_n)\ge1/2\), where \(\eps_0\) is the smallness threshold in
Proposition~\ref{prop:active}.

Fix a current type and size pair \((k,z)\) with \(z a^k\le b\).  For
\(\theta\in[0,1]\), set \(b_\theta:=b-\theta z a^k\) and
\(\rho_\theta:=b_\theta/n\).  Let \(m^{\max}\) dominate every normalized empirical
vector \(\widehat m_W\), and set \(\rho^{\max}:=Am^{\max}\) and
\(\rho_\theta^{\rm eff}:=\rho_\theta\wedge\rho^{\max}\).  As in
\eqref{eq:pone-clipping-no-change}, clipping nonbinding resource coordinates
does not change the feasible set:
\begin{equation}
        K_{\widehat m_W}(\rho_\theta)
        =
        K_{\widehat m_W}(\rho_\theta^{\rm eff}).
        \label{eq:epstage-clipping-no-change}
\end{equation}
The clipping map is coordinatewise \(1\)-Lipschitz.  Hence, uniformly over
\(\theta,\theta'\in[0,1]\),
\begin{equation}
        \|\rho_\theta^{\rm eff}-\rho_{\theta'}^{\rm eff}\|_\infty
        \le
        \|\rho_\theta-\rho_{\theta'}\|_\infty
        \le
        \frac{z\|a^k\|_\infty}{n}
        \le
        \frac{C}{n}.
        \label{eq:epstage-rhs-gap}
\end{equation}

By Lemma~\ref{lem:bounded-cutoff-selection}, for a.e. \(\theta\) and almost every
future path \(W\), we may choose a normalized empirical primal--cutoff pair
\((\widehat u_{\theta,W}^{(z)},q_{\theta,W}^{(z)})\).  It satisfies
\(\widehat u_{\theta,W}^{(z)}
\in K_{\widehat m_W}(\rho_\theta^{\rm eff})\), the empirical closed/right tail
relations, and
\[
        \Pi(q_{\theta,W}^{(z)})
        \in
        N_{K_{\widehat m_W}(\rho_\theta^{\rm eff})}
        (\widehat u_{\theta,W}^{(z)}).
\]
Here \eqref{eq:epstage-clipping-no-change} transfers feasibility and projected
normality to the clipped feasible set.  The superscript \((z)\) records that the
capacity path depends on the current size \(z\).  We choose jointly measurable
representatives on the full-measure differentiability sets; null sets in
\((\theta,z,W)\) are irrelevant by Fubini.

Take \(W,W'\in E_n\), and let \(\theta,\theta'\) be differentiability points for
the corresponding pathwise value functions.  On \(E_n\), the empirical
closed/right graph relations imply the population relaxed graph relations with
tolerance \(\delta_n\), as in \eqref{eq:tailrelax}.  Also,
\(\|\widehat m_W-m^0\|_\infty\le\delta_n\) and
\(\|\widehat m_{W'}-m^0\|_\infty\le\delta_n\), so the two empirical mass vectors
are within \(2\delta_n\) of each other.  Together with
\eqref{eq:epstage-rhs-gap}, these are exactly the operative hypotheses needed
for Proposition~\ref{prop:active}: feasibility, projected normality, relaxed
closed/right graph relations, closeness of \(m\), right-hand-side closeness, and
containment of swept active intervals in the active supports.

Applying Proposition~\ref{prop:active} with \(\eps=\delta_n\) and \(\tau=C/n\)
gives a sum bound over all types.  Since every summand is nonnegative, the
current type \(k\) satisfies
\begin{equation}
        \ell_k
        \bigl(
        q^{(z)}_{k,\theta,W},
        q^{(z)}_{k,\theta',W'}
        \bigr)
        \,
        \mu_k
        \!\left(
        I^a_k
        \bigl(
        q^{(z)}_{k,\theta,W},
        q^{(z)}_{k,\theta',W'}
        \bigr)
        \right)
        \le
        \bar r_s .
        \label{eq:epstage-current-product}
\end{equation}

Let \(\Omega=(\Theta,W)\), where \(\Theta\sim{\rm Unif}[0,1]\) is independent of
\(W\), and set \(Q^{(z)}_\Omega:=q^{(z)}_{k,\Theta,W}\).  The
marginal-to-cutoff reduction, Lemma~\ref{lem:marginal-cutoff-reduction-new},
gives
\begin{equation}
        \E_W[H_{k,z}(Y_W)]
        -
        H_{k,z}(\E_W[Y_W])
        \le
        \E_\Omega H_{k,z}(Q^{(z)}_\Omega)
        -
        H_{k,z}(\E_\Omega Q^{(z)}_\Omega).
        \label{eq:epstage-reduce-to-cutoff}
\end{equation}

We now condition on the good future event \(G:=E_n(W)\).  Since
\(Q^{(z)}_\Omega\in[0,y]\), and \(H_{k,z}\) is uniformly bounded and uniformly
Lipschitz on \([0,y]\), \eqref{eq:pone-bad-future-removal} gives
\begin{equation}
\begin{aligned}
        &\left|
        \Big(
        \E H_{k,z}(Q^{(z)}_\Omega)
        -
        H_{k,z}(\E Q^{(z)}_\Omega)
        \Big)
        -
        \Big(
        \E[H_{k,z}(Q^{(z)}_\Omega)\mid G]
        -
        H_{k,z}(\E[Q^{(z)}_\Omega\mid G])
        \Big)
        \right|  \\
        &\hspace{35mm}\le C\Prob(G^c)\le Cn^{-6}.
\end{aligned}
        \label{eq:epstage-good-conditioning}
\end{equation}

For each good future \(W\), let \(\mathcal D_{W,z}\subseteq[0,1]\) be the
full-measure set of differentiability points, and let \(\mathcal Q_{G,z}\) be the
essential range of the cutoff \(q^{(z)}_{k,\Theta,W}\) under the conditional law
of \((\Theta,W)\) given \(G\)---the smallest closed subset of \([0,y]\) that
\(Q^{(z)}_\Omega\) enters with conditional probability one.  After changing
\(Q^{(z)}_\Omega\) on a null set, \(Q^{(z)}_\Omega\mid G\) takes values in
\(\mathcal Q_{G,z}\).  Each point of \(\mathcal Q_{G,z}\) is a limit of cutoffs
\(q^{(z)}_{k,\theta,W}\) with \(W\in E_n\) and \(\theta\in\mathcal D_{W,z}\);
since \eqref{eq:epstage-current-product} caps every pair of such cutoffs, and
both \(\ell_k\) and the atomless measure \(\mu_k\) vary continuously under
convergence of interval endpoints, the cap passes to the limit.  Hence every
pair \(q,\bar q\in\mathcal Q_{G,z}\) satisfies
\begin{equation}
        \ell_k(q,\bar q)\,
        \mu_k(I^a_k(q,\bar q))
        \le
        \bar r_s .
        \label{eq:epstage-good-pairwise-cap}
\end{equation}

Let \(J_z\) be the closed active hull of \(\mathcal Q_{G,z}\) in \(S_k\), namely
\[
        J_z
        :=
        \overline{
        \bigcup_{q,\bar q\in\mathcal Q_{G,z}}
        [q\wedge\bar q,q\vee\bar q]\cap S_k
        } .
\]
If \(J_z\) has zero active length, then the good cutoff distribution straddles
no curvature of \(H_{k,z}\), and the good-support Jensen gap is zero.  We
therefore assume below that \(J_z\) has positive active length.  By
Lemma~\ref{lem:hullcap}, applied with \(B=S_k\) and \(\nu=\mu_k\), the pairwise
cap \eqref{eq:epstage-good-pairwise-cap} extends to the closed active hull:
\begin{equation}
        \Leb(J_z)\,\mu_k(J_z)\le \bar r_s .
        \label{eq:epstage-hull-cap}
\end{equation}
Because \(\bar r_s\le r_{\rm H}\), \eqref{eq:epstage-large-interval-product} and
\eqref{eq:epstage-hardy-radius} imply that \(J_z\) is contained in a single
member of the finite active cover.  Otherwise the Lebesgue-number property would
force \(\Leb(J_z)\ge\lambda_{\rm ch}\), contradicting
\eqref{eq:epstage-hull-cap}.

We next choose this cover element measurably.  Write the finite active cover for
type \(k\) as \(\mathcal U_k=\{U_{k,1},\ldots,U_{k,N_k}\}\).  Using the jointly
measurable cutoff representatives, define
\[
        L_z
        :=
        \operatorname*{ess\,inf}_{(\theta,W)\in G}
        \Pi_k(q^{(z)}_{k,\theta,W}),
        \qquad
        R_z
        :=
        \operatorname*{ess\,sup}_{(\theta,W)\in G}
        \Pi_k(q^{(z)}_{k,\theta,W}).
\]
Because \(\mathcal Q_{G,z}\) is the essential range, its projected infimum and
supremum are \(L_z\) and \(R_z\), so \(J_z=[L_z,R_z]\cap S_k\).  Since each
\(U_{k,i}\) is an interval neighborhood, the event \(\{J_z\subset U_{k,i}\}\) is
measurable in \(z\).  The first-index rule
\(i_k(z):=\min\{i:J_z\subset U_{k,i}\}\) is therefore measurable on the set
where \(J_z\) has positive active length.  On the zero-active-length set, set
\(i_k(z)=0\).  This decomposes the size space into measurable sets
\(A_{k,i}:=\{z:i_k(z)=i\}\), \(i=1,\ldots,N_k\), plus a zero-contribution set
\(A_{k,0}\).  On each \(A_{k,i}\), all good active hulls lie in the fixed cover
element \(U_{k,i}\).

We now bound the good-support Jensen contribution on each cover element.

\emph{Dominated neighborhoods.}
Suppose that \(J_z\subset U\), where \(U\) is dominated.  Every active interval
generated by two good cutoffs lies in \(J_z\subset U\).  Hence the local
domination assumption gives
\(\Lambda_{k,z}(I^a_k(q,\bar q))\le C\mu_k(I^a_k(q,\bar q))\) for
\(P_k^\beta\)-a.e. \(z\) and all \(q,\bar q\in\mathcal Q_{G,z}\).  Combining
this with \eqref{eq:epstage-good-pairwise-cap}, we get
\[
        \ell_k(q,\bar q)\,
        \Lambda_{k,z}(I^a_k(q,\bar q))
        \le C\bar r_s,
        \qquad q,\bar q\in\mathcal Q_{G,z}.
\]

We verify the measure hypotheses needed for Lemma~\ref{lem:jensen}.  The measure
\(\mu_k\) is atomless, because \(\mu_k(\{t\})\le C\Leb(\{t\}\cap S_k)=0\).  Also,
\(\pi_k\int\Lambda_{k,z}([0,y]\setminus S_k)\,P_k^\beta(\dd z)
=\mu_k([0,y]\setminus S_k)=0\), so \(\Lambda_{k,z}\) is supported on \(S_k\) for
\(P_k^\beta\)-a.e. \(z\).  Since \(J_z\subset U\), only curvature inside \(U\)
can contribute to the good-support Jensen gap.  We may therefore replace
\(H_{k,z}\) by a convex representative \(H^U_{k,z}\) whose curvature measure is
\(\Lambda_{k,z}|_U\).  This restricted measure is finite, because
\(z\le\betamax_k\), and atomless by local domination.

Lemma~\ref{lem:jensen}, applied conditionally on \(G\) with
\(h=H^U_{k,z}\), \(\mu_h=\Lambda_{k,z}|_U\), and
\(\mathcal Q=\mathcal Q_{G,z}\), yields
\[
        \E[H_{k,z}(Q^{(z)}_\Omega)\mid G]
        -
        H_{k,z}(\E[Q^{(z)}_\Omega\mid G])
        \le
        C\bar r_s
        \le
        C\bar r_s\log(e/\bar r_s).
\]
The Jensen gap of \(H^U_{k,z}\) equals that of \(H_{k,z}\), because the gap
kernel vanishes outside the projected hull \(J_z\): outside \(J_z\), all good
cutoffs lie on one side of the curvature point, as in the proof of
Lemma~\ref{lem:hull-jensen}.

\emph{Endpoint-contact neighborhoods.}
Now suppose that \(U=U_{k,i}\) is an endpoint-contact cover element, with
assigned size set \(A_{k,i}\).  For a size \(z\) and a curvature measure
\(\Lambda\), define the good-support Jensen functional, after subtracting affine
parts, by
\[
        \Delta_{z,i}^U(\Lambda)
        :=
        \1\{z\in A_{k,i}\}
        \int_U
        \left\{
        \E\big[(Q^{(z)}_\Omega-t)^+\mid G\big]
        -
        \big(\E[Q^{(z)}_\Omega\mid G]-t\big)^+
        \right\}
        \Lambda(\dd t).
\]
The integrand is nonnegative and independent of \(\Lambda\), so
\(\Lambda\mapsto\Delta_{z,i}^U(\Lambda)\) is a nonnegative curvature-linear
functional.  The kernel split \eqref{eq:endpoint-kernel-split} gives
\(\Delta_{z,i}^U(\Lambda_{k,z}|_U)
=\Delta_{z,i}^U(\Lambda^D_{k,z})
+\Delta_{z,i}^U(\Lambda^{\mathrm{br}}_{k,z})\) for \(P_k^\beta\)-a.e. \(z\).

The dominated component is handled exactly as in the dominated-neighborhood
case, using \(\Lambda^D_{k,z}(I)\le C\mu_k(I)\) from
\eqref{eq:dominated-component}.  Its good-support contribution is therefore at
most \(C\bar r_s\), and hence at most \(C\bar r_s\log(e/\bar r_s)\).

It remains to control the contact branch after averaging over size.  By
Lemma~\ref{lem:branchdis}, applied to the functional
\(\Delta_{z,i}^U\), the branch contribution is bounded, up to a primitive
constant, by
\begin{equation}
        \int_0^{\omega_0}
        \Delta_{\beta_k(\omega),i}^U
        \bigl(\Lambda^{\mathrm{br}}_{k,\omega}\bigr)
        f(\omega)\,\dd\omega .
        \label{eq:epstage-branch-integrated-gap}
\end{equation}

Use the local endpoint coordinate \(\zeta\), oriented into the support as in
Lemma~\ref{lem:mass}.  This affine change of variables has slope \(1\) or
\(-1\), so it preserves lengths and Jensen gaps after the branch surplus is
written in local coordinates.  If the current item is infeasible, or if the good
local cutoff set is empty, set \(I_\omega=\emptyset\) and \(d_\omega=0\), and
define the branchwise contribution to be zero.  Otherwise, for
\(z=\beta_k(\omega)\), define
\[
        \mathcal Q_\omega
        :=
        \left\{
        \zeta\!\left(q^{(\beta_k(\omega))}_{k,\theta,W}\right):
        W\in E_n,\ \theta\in\mathcal D_{W,\beta_k(\omega)}
        \right\}.
\]
Let \(I_\omega\) be the closed active hull of \(\mathcal Q_\omega\) inside the
moving branch support \(B_\omega=[e(\omega),x_0]\), and set
\(d_\omega:=\Leb(I_\omega)\).

For two local cutoffs \(q,\bar q\in\mathcal Q_\omega\), write
\(I^\omega(q,\bar q):=[q\wedge\bar q,q\vee\bar q]\cap B_\omega\).  If
\(\tilde q=\zeta^{-1}(q)\) and
\(\tilde{\bar q}=\zeta^{-1}(\bar q)\), then the local coordinate map preserves
active lengths and sends \(I^\omega(q,\bar q)\) into the original active
interval \(I^a_k(\tilde q,\tilde{\bar q})\).  Hence
\[
        \Leb(I^\omega(q,\bar q))\,
        \mu_k^{\mathrm{br}}(I^\omega(q,\bar q))
        \le
        \ell_k(\tilde q,\tilde{\bar q})\,
        \mu_k(I^a_k(\tilde q,\tilde{\bar q}))
        \le
        \bar r_s .
\]
Since Lemma~\ref{lem:mass} implies that \(\mu_k^{\mathrm{br}}\) is atomless,
Lemma~\ref{lem:hullcap} extends this pairwise cap to the branch active hull:
\begin{equation}
        d_\omega\,\mu_k^{\mathrm{br}}(I_\omega)
        \le
        \bar r_s
        \qquad\text{for a.e. }\omega.
        \label{eq:epstage-branch-hull-cap}
\end{equation}
The passage from \(P_k^\beta\)-a.e. size \(z\) to a.e. branch coordinate
\(\omega\) uses the injectivity of \(\beta_k(\omega)\) and the lower Jacobian
bound in Definition~\ref{def:branch}.

The branch contribution \eqref{eq:epstage-branch-integrated-gap} is estimated
through Lemma~\ref{lem:hull-jensen}, whose geometric factor \(1+y/|J|\) is
governed by the length of the ambient interval \(J\).  The branch support
\(B_\omega=[e(\omega),x_0]\) has length \(x_0-e(\omega)\), which may be
arbitrarily small, so we split the branch coordinates according to
\[
        \Omega_{\mathrm L}=\{\omega:e(\omega)\le x_0/2\},
        \qquad
        \Omega_{\mathrm T}=\{\omega:e(\omega)>x_0/2\},
\]
and bound the two contributions separately.

Fix first a feasible coordinate \(\omega\in\Omega_{\mathrm L}\) with
\(\mathcal Q_\omega\neq\emptyset\) and \(\beta_k(\omega)\in A_{k,i}\).  The
branchwise good-support Jensen gap is the Jensen gap of the branch surplus,
whose curvature measure \(\Lambda^{\mathrm{br}}_{k,\omega}\) is atomless and
supported on \(B_\omega\).  Because \(|B_\omega|=x_0-e(\omega)\ge x_0/2\), the
factor \(1+y/|B_\omega|\) is at most \(1+2y/x_0\), and
Lemma~\ref{lem:hull-jensen}, applied with \(J=B_\omega\),
\(\nu=\Lambda^{\mathrm{br}}_{k,\omega}\), and ambient length at most \(y\), gives
\begin{equation}
        \Delta_{\beta_k(\omega),i}^U
        \bigl(\Lambda^{\mathrm{br}}_{k,\omega}\bigr)
        \le
        C\,d_\omega\,
        \Lambda^{\mathrm{br}}_{k,\omega}(I_\omega).
        \label{eq:epstage-branchwise-jensen}
\end{equation}
Since \(\bar r_s\le r_{\mathrm H}\), the cap \eqref{eq:epstage-branch-hull-cap}
lies in the range of Lemma~\ref{lem:hardy}.  Applying that lemma, and bounding
the integrand over \(\Omega_{\mathrm L}\) by \eqref{eq:epstage-branchwise-jensen}
and elsewhere by its nonnegativity, we obtain
\begin{equation}
        \int_{\Omega_{\mathrm L}}
        \Delta_{\beta_k(\omega),i}^U
        \bigl(\Lambda^{\mathrm{br}}_{k,\omega}\bigr)
        f(\omega)\,\dd\omega
        \le
        C\int_0^{\omega_0}
        d_\omega\,\Lambda^{\mathrm{br}}_{k,\omega}(I_\omega)\,f(\omega)\,\dd\omega
        \le
        C\bar r_s\log(e/\bar r_s).
        \label{eq:epstage-hardy-applied}
\end{equation}

Fix next a feasible coordinate \(\omega\in\Omega_{\mathrm T}\), so that
\(e(\omega)>x_0/2\) and \(B_\omega\subseteq[x_0/2,x_0]\) is bounded away from the
contact point; on this region the conditional curvature is dominated by
\(\mu_k\), and the branch is treated as a dominated neighborhood.  Since
\(\gamma\ge1\), the branch density satisfies
\(\lambda^{\mathrm{br}}_{k,\omega}(x)\le C(x-e(\omega))^{\gamma-1}\le C\) for
\(x\in B_\omega\), so
\[
        \Lambda^{\mathrm{br}}_{k,\omega}(I)
        =\Lambda^{\mathrm{br}}_{k,\omega}(I\cap B_\omega)
        \le C\,\Leb(I\cap B_\omega)
        \qquad\text{for every interval }I.
\]
By Lemma~\ref{lem:mass}, the branch marginal \(\mu_k^{\mathrm{br}}\) has density
comparable to \(x^{\theta-1}\), which is bounded below by a positive constant on
\([x_0/2,x_0]\); since \(\mu_k\ge\mu_k^{\mathrm{br}}\) and
\(B_\omega\subseteq[x_0/2,x_0]\),
\[
        \mu_k(I)\ge \mu_k^{\mathrm{br}}(I\cap B_\omega)\ge c\,\Leb(I\cap B_\omega)
        \qquad\text{for every interval }I.
\]
Combining the two estimates,
\[
        \Lambda^{\mathrm{br}}_{k,\omega}\bigl(I^a_k(q,\bar q)\bigr)
        \le
        C\,\mu_k\bigl(I^a_k(q,\bar q)\bigr)
        \qquad\text{for all }q,\bar q\in\mathcal Q_{G,z},
\]
which is the domination hypothesis of the dominated-neighborhood case.  Hence,
exactly as for the dominated component above, the pairwise cap
\eqref{eq:epstage-good-pairwise-cap} and Lemma~\ref{lem:jensen}, applied on the
full active support \(S_k\) with curvature measure
\(\mu_h=\Lambda^{\mathrm{br}}_{k,\omega}\), give
\(\Delta_{\beta_k(\omega),i}^U(\Lambda^{\mathrm{br}}_{k,\omega})\le C\bar r_s\);
the factor \(1+y/|S_k|\) from that lemma is bounded, and the interval
\([x_0/2,x_0]\) enters only through the density domination just established.
Because \(f\) is integrable,
\begin{equation}
        \int_{\Omega_{\mathrm T}}
        \Delta_{\beta_k(\omega),i}^U
        \bigl(\Lambda^{\mathrm{br}}_{k,\omega}\bigr)
        f(\omega)\,\dd\omega
        \le
        C\bar r_s\int_0^{\omega_0}f(\omega)\,\dd\omega
        \le
        C\bar r_s.
        \label{eq:epstage-tail-dominated}
\end{equation}

Adding \eqref{eq:epstage-hardy-applied} and \eqref{eq:epstage-tail-dominated},
and using \(\bar r_s\le1\) in the latter, the branch part of
\eqref{eq:epstage-branch-integrated-gap} is at most
\(C\bar r_s\log(e/\bar r_s)\).  Adding the dominated component gives the same
bound for endpoint-contact neighborhoods.

We have now bounded the good-support contribution on every measurable set
\(A_{k,i}\).  Summing over the finite cover and over the finitely many types
absorbs only primitive constants, so the total good-support contribution to
\(\Xi_s(b)\) is at most \(C\bar r_s\log(e/\bar r_s)\).  Using the cutoff
reduction \eqref{eq:epstage-reduce-to-cutoff} and the conditioning estimate
\eqref{eq:epstage-good-conditioning}, we get
\[
        \Xi_s(b)
        \le
        C\bar r_s\log(e/\bar r_s)+Cn^{-6}.
\]
Finally, \eqref{eq:epstage-rbar-rs} and \(\bar r_s,r_s\in(0,1)\) imply
\(\bar r_s\log(e/\bar r_s)\le C r_s\log(e/r_s)\).  After increasing \(s_0\), we
also have \(n^{-6}\le C r_s\log(e/r_s)\) for all \(s\ge s_0\).  Hence
\(\Xi_s(b)\le C r_s\log(e/r_s)\), as claimed.
\end{proof}

\section{A local exponent refinement}\label{app:local}

The exponent \(\pp\) in Assumption~\ref{ass:regularity} is a worst-case quantity
over capacities: it must control active mass at every cutoff the support can
produce. At a fixed capacity, however, only a local part of the ratio support is
active, and the mass growth there may be milder. This appendix isolates that
local part, defines a capacity-local exponent \(\pp_{\mathcal R}\), and proves
that, at an interior locally non-degenerate capacity, the regret of \(\SPM\) is
governed by \(\pp_{\mathcal R}\) rather than by the global exponent \(\pp\).

We use notation from Section~\ref{sec:spm}. The type-wise tails are
\(C_k(r)=\mu_k([r,\infty))\) and \(C_k(r+)=\mu_k((r,\infty))\), with
\(m_k^0=C_k(0)\) and \(m^0=(m_1^0,\ldots,m_K^0)\). The type-level fluid feasible
set is \(K_m(\rho)=\{u\in\R_+^K:0\le u\le m,\ Au\le\rho\}\), as in
\eqref{eq:feasible-set}. We also use the projection \(\Pi\) onto the active
supports, the per-stage Jensen loss \(\Xi_s\), the telescoping bound
\eqref{eq:telescope}, the pre-Young estimate of
Proposition~\ref{prop:active-linear}, and the absorption
Lemma~\ref{lem:local-absorption}. For a compact convex set \(K\) and a point
\(x\in K\), the normal cone is \(N_K(x)=\{z:z^\top(v-x)\le0\ \text{for all }v\in K\}\).

\subsection{The capacity-local exponent}

The local theorem needs to know which cutoffs can arise from small perturbations
of a given capacity region. The next definition records those cutoffs in a
tube.

\begin{definition}[Capacity-local active tube]\label{def:local-tube}
Fix a compact normalized-capacity set \(\mathcal R\subset\R_+^d\) and a scale
\(\eta>0\). For each type \(k\), let \(\mathcal A_k(\mathcal R,\eta)\) be the set
of projected cutoffs \(\Pi_k(q_k)\) for which there exist
\[
        m\in\R_+^K,\qquad
        \rho\in\R_+^d,\qquad
        u\in\R_+^K,\qquad
        q\in[0,y]^K
\]
such that
\[
        \|m-m^0\|_\infty\le\eta,
        \qquad
        \operatorname{dist}_\infty(\rho,\mathcal R)\le\eta,
        \qquad
        u\in K_m(\rho),
        \qquad
        \Pi(q)\in N_{K_m(\rho)}(u),
\]
and, for \(j=1,\ldots,K\),
\[
        C_j(q_j+)-\eta\le u_j\le C_j(q_j)+\eta .
\]
The type-\(k\) cutoff tube over \(\mathcal R\) at scale \(\eta\) is
\[
        \mathcal J_k(\mathcal R,\eta)
        :=
        \overline{\operatorname{conv}\mathcal A_k(\mathcal R,\eta)}\cap S_k .
\]
\end{definition}

The local exponent is an admissible exponent on this tube. We do not require it
to be globally sharp.

\begin{assumption}[Capacity-local active regularity]\label{ass:local-active}
The arrival distribution satisfies \emph{capacity-local active regularity} on a
compact set \(\mathcal R\subset\R_+^d\) with admissible exponent
\(\pp_{\mathcal R}\ge1\) if there exist a scale \(\eta_{\mathcal R}>0\),
relatively open neighborhoods \(U_k\subseteq S_k\) such that
\[
        \mathcal J_k(\mathcal R,\eta_{\mathcal R})\subseteq U_k,
        \qquad k=1,\ldots,K,
\]
and constants \(0<c\le C<\infty\) such that, for every type \(k\) and every
interval \(I\subseteq[0,y]\) with \(I\cap S_k\subseteq U_k\),
\[
        c\,\ell_k(I)^{\pp_{\mathcal R}}
        \le
        \mu_k(I)
        \le
        C\,\ell_k(I),
        \qquad
        \ell_k(I):=\Leb(I\cap S_k).
\]
The sets \(U_1,\ldots,U_K\) are the local active neighborhoods for
\(\mathcal R\).
\end{assumption}

If several admissible exponents satisfy Assumption~\ref{ass:local-active}, the
sharpest bound below is obtained by using the smallest available one.

\begin{assumption}[Local fluid non-degeneracy and binding balance]
\label{ass:local-nondeg}
Fix compact normalized-capacity regions
\[
        \mathcal R_0\Subset\mathcal R\subset\R_+^d,
\]
and suppose that Assumption~\ref{ass:local-active} holds on \(\mathcal R\) with
admissible exponent \(\pp_{\mathcal R}\) and local active neighborhoods
\(U_1,\ldots,U_K\). Assume also that this tube is interior: the neighborhoods
\(U_k\) are separated from the endpoints of \(S_k\). The following two
conditions hold.

\begin{enumerate}[label=\textup{(\alph*)},leftmargin=2.2em]
\item \textup{(interior cutoff and binding balance)}
For every \(\rho\in\mathcal R\), the population fluid problem admits a
cutoff--allocation pair \((q^\rho,u^\rho)\), with
\[
        \Pi_k(q_k^\rho)\in U_k,
        \qquad k=1,\ldots,K,
\]
such that
\[
        u^\rho\in K_{m^0}(\rho),
        \qquad
        C_k(q_k^\rho+)\le u_k^\rho\le C_k(q_k^\rho),
        \qquad k=1,\ldots,K,
\]
and
\[
        \Pi(q^\rho)\in N_{K_{m^0}(\rho)}(u^\rho).
\]
The locally relevant resources bind:
\[
        Au^\rho=\rho .
\]

\item \textup{(dominated tube)}
For each type \(k\), there is a constant \(C_{\rm dom}<\infty\) such that, for
\(P_k^\beta\)-a.e.\ \(z\),
\[
        \Lambda_{k,z}(I)\le C_{\rm dom}\,\mu_k(I)
\]
for every interval \(I\subseteq[0,y]\) with \(I\cap S_k\subseteq U_k\).
\end{enumerate}
\end{assumption}

The domination in part~(b) is automatic on an interior tube under
Assumption~\ref{ass:regularity}, once the endpoint-contact neighborhoods are
chosen small enough. Away from the endpoints of \(S_k\), the finite
conditional-curvature cover consists of dominated neighborhoods. At an
endpoint-contact cutoff, this domination may fail: the size-conditioned measure
\(\Lambda_{k,z}\) can concentrate near a size-dependent edge even where the
aggregate measure \(\mu_k\) is small. The endpoint Hardy estimate of
Section~\ref{sec:spm} is designed precisely for that non-dominated passage from
\(\Lambda_{k,z}\) to \(\mu_k\). Thus
Assumption~\ref{ass:local-nondeg} is an interior-capacity condition. It is the
condition that holds in the interior regimes considered below, and it fails in
the endpoint-pinned critical regime.

\begin{remark}[Slack resources]\label{rem:slack}
For a multi-resource instance in which some resources are slack throughout
\(\mathcal R\), the binding condition \(Au^\rho=\rho\) should be replaced by the
inward-pointing condition
\[
        n^\top(\rho-Au^\rho)\le0
        \qquad
        \text{for every outward normal }
        n\in N_{\mathcal R}(\rho),\ \rho\in\partial\mathcal R .
\]
With this replacement, the state-localization proof below uses a
distance-to-\(\mathcal R\) supermartingale instead of the ball estimate around
the initial normalized capacity. In the one-resource examples of
Section~\ref{subsec:local-examples}, the binding formulation is the relevant
one.
\end{remark}

The next corollary is the local version of the active-mass product bound. The
only change from the global argument is that the lower mass exponent is used
only on the local tube.

\begin{corollary}[Local active-mass stability]\label{cor:active-local}
Suppose that Assumption~\ref{ass:local-active} holds on \(\mathcal R\) with
admissible exponent \(\pp_{\mathcal R}\) and local active neighborhoods
\(U_1,\ldots,U_K\). In the setting of
Proposition~\ref{prop:active-linear}, suppose in addition that
\[
        I_k^a(\widetilde q_k,\bar q_k)\subseteq U_k,
        \qquad k=1,\ldots,K.
\]
Then
\[
        \sum_{k=1}^K
        \ell_k(\widetilde q_k,\bar q_k)
        M_k(\widetilde q_k,\bar q_k)
        \le
        C\bigl(\tau+\eps^{1+1/\pp_{\mathcal R}}\bigr).
\]
\end{corollary}

\begin{proof}
Proposition~\ref{prop:active-linear} gives
\(\sum_{k=1}^K\ell_kM_k\le C_1(\tau+\eps\sum_{k=1}^K\ell_k)\). On \(U_k\),
Assumption~\ref{ass:local-active} gives \(M_k=\mu_k(I_k)\ge c\,\ell_k^{\pp_{\mathcal R}}\).
Lemma~\ref{lem:local-absorption}, applied with \(p=\pp_{\mathcal R}\), yields the
claim.
\end{proof}

\subsection{State localization and the local balance of \texorpdfstring{\(\SPM\)}{SPM}}

We now prove that the \(\SPM\) state remains in the interior capacity region with
high probability. With \(s\) periods remaining, set \(n=s-1\). For \(s\ge2\),
let \(\delta_n=\sqrt{\log(en)/n}\) be the concentration scale from
Section~\ref{sec:spm}, and define the local per-stage scale
\[
        r_s^{\rm loc}
        :=
        \left(\frac{\log(es)}{s}\right)^{(\pp_{\mathcal R}+1)/(2\pp_{\mathcal R})}
        +
        \frac1s .
\]
Thus \(r_s^{\rm loc}\asymp\delta_n^{1+1/\pp_{\mathcal R}}+1/s\).
It is useful to separate the per-stage quantity that will actually be summed:
\[
        h_s^{\rm loc}
        :=
        \begin{cases}
        r_s^{\rm loc}, & \pp_{\mathcal R}=1,\\[1mm]
        r_s^{\rm loc}\log(e/r_s^{\rm loc}), & \pp_{\mathcal R}>1.
        \end{cases}
\]
Let \(D_s:=\beta_s a^{J_s}X_s^{\SPM}\) be the resource consumed by \(\SPM\) at a
decision epoch with \(s\) periods remaining. Also write
\(\bar d_s(b):=\E[D_s\mid B_s=b]\) and \(\rho_s:=B_s/s\). Finally, let
\(d_{\max}:=\max_k\beta_k^{\max}\|a^k\|_\infty\), where \(\beta_k^{\max}\) is the
essential supremum of the size for type \(k\).

For \(s\ge2\) and a compact region \(\mathcal R\), define the sweep-localization
event
\[
        \mathcal L_s(\mathcal R)
        =
        \left\{
        \frac{B_s-\theta za^k}{s-1}\in\mathcal R
        \text{ for every feasible }za^k\le B_s
        \text{ and every }\theta\in[0,1]
        \right\}.
\]
For \(s=1\), set \(\mathcal L_1(\mathcal R)=\Omega\). The event
\(\mathcal L_s(\mathcal R)\) is \(\sigma(B_s)\)-measurable.

\begin{lemma}[Local balance of \(\SPM\)]\label{lem:local-balance-spm}
Suppose that Assumptions~\ref{ass:local-active} and
\ref{ass:local-nondeg} hold on \(\mathcal R\), with
\(\mathcal R_0\Subset\mathcal R\). There are constants \(C<\infty\) and
\(s_0<\infty\) such that, for all \(s\ge s_0\) and all \(b\) with
\(b/s\in\mathcal R_0\),
\[
        \left\|\bar d_s(b)-\frac bs\right\|_\infty
        \le
        \kappa_s,
        \qquad
        \kappa_s
        :=
        C\bigl(r_s^{\rm loc}\bigr)^{1/(\pp_{\mathcal R}+1)}
        +
        Cs^{-6}.
\]
\end{lemma}

\begin{proof}
Fix \(b\) with \(\rho:=b/s\in\mathcal R_0\), and let \((q^\rho,u^\rho)\) be the
population pair from Assumption~\ref{ass:local-nondeg}(a). The binding-balance
condition gives \(Au^\rho=\rho=b/s\).

Consider a feasible current pair \((k,z)\), and write
\(q_{\theta,W}^{k,z}\) for the vector of pathwise cutoffs generated by the
marginal comparison at sweep point \(\rho_\theta=(b-\theta za^k)/n\),
\(0\le\theta\le1\). On the good event \(E_n(W)\) of Section~\ref{sec:spm}, whose complement has
probability at most \(Cn^{-6}\), these pathwise cutoffs satisfy the relaxed
graph and normal-cone relations of Proposition~\ref{prop:active-linear} with
error \(\delta_n\). Since \(\rho\in\mathcal R_0\Subset\mathcal R\), the sweep
capacities \(\rho_\theta\) lie in \(\mathcal R\) for all large \(s\), and
\(\|\rho_\theta-\rho\|_\infty\le C/n\).

The definition of the local tube now ensures that the active intervals remain
inside the local neighborhoods. Indeed, for \(n\) large enough that
\(\delta_n\le\eta_{\mathcal R}\), the projected cutoffs
\(\Pi_j((q_{\theta,W}^{k,z})_j)\) belong to
\(\mathcal A_j(\mathcal R,\eta_{\mathcal R})\). The population cutoff
\(\Pi_j(q_j^\rho)\) also belongs to this set. Hence the active interval between
the two projected cutoffs is contained in
\(\mathcal J_j(\mathcal R,\eta_{\mathcal R})\subseteq U_j\).

Applying Corollary~\ref{cor:active-local} to
\(q_{\theta,W}^{k,z}\) and \(q^\rho\), with
\(\eps=\delta_n\) and \(\tau=\|\rho_\theta-\rho\|_\infty\), gives
\[
        \sum_{j=1}^K
        \ell_j\bigl((q_{\theta,W}^{k,z})_j,q_j^\rho\bigr)\,
        M_j\bigl((q_{\theta,W}^{k,z})_j,q_j^\rho\bigr)
        \le
        C\left(\frac1n+\delta_n^{1+1/\pp_{\mathcal R}}\right)
        \le
        C r_s^{\rm loc}.
\]
The local lower bound gives
\[
        M_j\bigl((q_{\theta,W}^{k,z})_j,q_j^\rho\bigr)
        \ge
        c\,\ell_j\bigl((q_{\theta,W}^{k,z})_j,q_j^\rho\bigr)^{\pp_{\mathcal R}},
\]
so each swept length satisfies
\[
        \ell_j\bigl((q_{\theta,W}^{k,z})_j,q_j^\rho\bigr)
        \le
        C\bigl(r_s^{\rm loc}\bigr)^{1/(\pp_{\mathcal R}+1)}.
\]
Because the tube is interior, a cutoff whose projection lies in \(U_j\) must
itself lie in \(U_j\). Therefore
\[
        \left|(q_{\theta,W}^{k,z})_j-q_j^\rho\right|
        \le
        C\bigl(r_s^{\rm loc}\bigr)^{1/(\pp_{\mathcal R}+1)}
\]
on the good event.

Let \(Y_{s,k,z}(b)\) be the ratio cutoff used by \(\SPM\) for type \(k\) and
size \(z\). This cutoff is the average, over \((\theta,W)\), of the corresponding
pathwise cutoffs. Jensen's inequality is applied to the cutoff, not to the
measure:
\[
        \left|Y_{s,k,z}(b)-q_k^\rho\right|
        \le
        \E_{W,\theta}
        \left|(q_{\theta,W}^{k,z})_k-q_k^\rho\right|
        \le
        C\bigl(r_s^{\rm loc}\bigr)^{1/(\pp_{\mathcal R}+1)}
        +
        Cn^{-6},
\]
where the \(n^{-6}\) term comes from \(E_n(W)^c\), on which cutoffs are bounded.

The comparison of consumed resource must be made before integrating over sizes,
because the \(\SPM\) cutoff depends on \(z\). For type \(k\) and size \(z\), the
\(\SPM\) decision and the population cutoff decision can differ only on the
ratio band
\[
        I_{k,z}(b)
        :=
        \left[
        Y_{s,k,z}(b)\wedge q_k^\rho,\,
        Y_{s,k,z}(b)\vee q_k^\rho
        \right].
\]
This band is contained in \(U_k\). The dominated-tube condition in
Assumption~\ref{ass:local-nondeg}(b) gives
\(\Lambda_{k,z}(I_{k,z}(b))\le C_{\rm dom}\,\mu_k(I_{k,z}(b))\). The local linear
upper bound in Assumption~\ref{ass:local-active} then gives
\(\mu_k(I_{k,z}(b))\le C\,|Y_{s,k,z}(b)-q_k^\rho|\).
Using bounded sizes and resource vectors, and then integrating over \(z\), we
obtain
\[
\begin{aligned}
        \left\|\bar d_s(b)-Au^\rho\right\|_\infty
        &\le
        C\sum_{k=1}^K
        \int
        \Lambda_{k,z}\bigl(I_{k,z}(b)\bigr)\,
        P_k^\beta(\dd z)
        +
        Cn^{-6}  \\
        &\le
        C\sum_{k=1}^K
        \int
        \left|Y_{s,k,z}(b)-q_k^\rho\right|\,
        P_k^\beta(\dd z)
        +
        Cn^{-6}  \\
        &\le
        C\bigl(r_s^{\rm loc}\bigr)^{1/(\pp_{\mathcal R}+1)}
        +
        Cn^{-6}.
\end{aligned}
\]
Since \(Au^\rho=b/s\) and \(n=s-1\), the claim follows.
\end{proof}

\begin{lemma}[State localization]\label{lem:state-localization}
Suppose that Assumptions~\ref{ass:local-active} and
\ref{ass:local-nondeg} hold on \(\mathcal R\), with
\(\mathcal R_0\Subset\mathcal R\) and
\[
        \rho_0:=\frac{b_T}{T}\in\operatorname{int}(\mathcal R_0).
\]
Then, for every \(A_0>0\), there is a constant \(C<\infty\) such that
\[
        \sum_{s=1}^T\Prob\bigl(\mathcal L_s(\mathcal R)^c\bigr)
        \le
        C\log(eT)+CT^{-A_0}.
\]
In particular,
\[
        \sum_{s=1}^T\Prob\bigl(\mathcal L_s(\mathcal R)^c\bigr)
        =
        O(\log T).
\]
\end{lemma}

\begin{proof}
It is enough to prove the claim for large \(T\); bounded values of \(T\) are
absorbed by increasing the constant. Choose \(\eta>0\) such that \(\{\rho:\|\rho-\rho_0\|_\infty\le4\eta\}\subseteq\mathcal R_0\).
Choose \(S_\eta\) large enough that, whenever \(s_\star\ge S_\eta\),
\(\sum_{u\ge s_\star}\kappa_u/u\le\eta\) and
\((\|\rho_0\|_\infty+4\eta+d_{\max})/(s_\star-1)\le\eta\). Such an \(S_\eta\)
exists because \(\kappa_u\le C(\log(eu)/u)^{1/(2\pp_{\mathcal R})}+Cu^{-6}\), and
therefore \(\sum_u \kappa_u/u<\infty\). Now set
\(s_\star=\max\{S_\eta,\lceil C_\eta\log(eT)\rceil\}\), where \(C_\eta\) will be
chosen below.

Read the trajectory backward from \(s=T\) to \(s=1\), and define the first exit
time in backward time by
\[
        \tau
        =
        \max\Bigl\{
        s\in\{s_\star,\ldots,T\}:
        \|\rho_s-\rho_0\|_\infty>2\eta,\ 
        \|\rho_u-\rho_0\|_\infty\le2\eta
        \text{ for all }u>s
        \Bigr\},
\]
with the convention that \(\tau=0\) if no such exit occurs.

The capacity recursion \(B_{u-1}=B_u-D_u\) gives
\(\rho_{u-1}-\rho_u=(\rho_u-D_u)/(u-1)\). Splitting the increment into its conditional mean and martingale parts gives
\[
        \rho_{u-1}-\rho_u
        =
        \frac{\rho_u-\bar d_u(B_u)}{u-1}
        +
        \frac{\bar d_u(B_u)-D_u}{u-1}.
\]
On the event \(\{\tau>0\}\), all states \(\rho_u\) with \(u>\tau\) lie in
\(\mathcal R_0\). Lemma~\ref{lem:local-balance-spm} therefore gives
\(\|\rho_u-\bar d_u(B_u)\|_\infty\le\kappa_u\) for \(u>\tau\). The cumulative
drift before the exit is bounded by \(\sum_{u>\tau}\kappa_u/(u-1)\le\eta\).

It remains to control the martingale term. After the standard re-indexing from
backward time to forward time, the stopped increments
\(\1\{u>\tau\}(\bar d_u(B_u)-D_u)/(u-1)\) are martingale differences. Each coordinate is bounded by \(C/u\), and the
predictable quadratic variation over \(u\ge s_\star\) is at most \(C/s_\star\).
Freedman's maximal inequality, followed by a union bound over the \(d\) resource
coordinates, gives
\[
\begin{aligned}
        &\Prob\left(
        \sup_{s_\star\le s\le T}
        \left\|
        \sum_{u>s}
        \1\{u>\tau\}
        \frac{\bar d_u(B_u)-D_u}{u-1}
        \right\|_\infty
        >
        \eta
        \right)  \\
        &\hspace{3cm}
        \le
        2d\,e^{-c\eta^2s_\star}.
\end{aligned}
\]
Choosing \(C_\eta\) large enough makes the last expression at most
\(T^{-(A_0+2)}\).

On the complement of this martingale event, the exit cannot occur. Indeed, if
\(\tau>0\), then telescoping from \(T\) down to \(\tau\) gives
\(\rho_\tau-\rho_0=\sum_{u=\tau+1}^{T}(\rho_{u-1}-\rho_u)\). The drift part has
norm at most \(\eta\), and the stopped martingale part has norm at most \(\eta\).
Hence \(\|\rho_\tau-\rho_0\|_\infty\le2\eta\), which contradicts the definition of
\(\tau\). Therefore
\(\Prob(\exists s\ge s_\star:\|\rho_s-\rho_0\|_\infty>2\eta)\le T^{-(A_0+2)}\).

It remains to pass from state localization to sweep localization. Suppose that
\(s\ge s_\star\) and \(\|\rho_s-\rho_0\|_\infty\le2\eta\). For every feasible
\(za^k\le B_s\) and every \(\theta\in[0,1]\),
\[
\begin{aligned}
        \left\|
        \frac{B_s-\theta za^k}{s-1}-\rho_s
        \right\|_\infty
        &\le
        \left\|
        \frac{B_s}{s-1}-\frac{B_s}{s}
        \right\|_\infty
        +
        \frac{\|za^k\|_\infty}{s-1}  \\
        &\le
        \frac{\|\rho_s\|_\infty+d_{\max}}{s-1}
        \le
        \eta .
\end{aligned}
\]
Thus the swept capacity lies within \(3\eta\) of \(\rho_0\), hence inside
\(\mathcal R_0\subseteq\mathcal R\). Therefore \(\mathcal L_s(\mathcal R)\)
holds for every \(s\ge s_\star\), except on an event of probability at most
\(T^{-(A_0+2)}\). Consequently,
\[
        \sum_{s=1}^T\Prob(\mathcal L_s(\mathcal R)^c)
        \le
        s_\star+T\cdot T^{-(A_0+2)}
        \le
        C\log(eT)+CT^{-(A_0+1)} .
\]
Renaming the exponent proves the claim.
\end{proof}

\subsection{The localized regret bound}

The next lemma is the local analogue of the per-stage estimate in
Section~\ref{sec:spm}. In the linear-mass case, the dominated-tube argument gives
a direct \(r_s^{\rm loc}\) bound. For larger local exponents, we keep the
logarithmic factor from the general per-stage estimate.

\begin{lemma}[Localized per-stage loss]\label{lem:local-perstage}
Suppose that Assumption~\ref{ass:local-active} holds on \(\mathcal R\) with
admissible exponent \(\pp_{\mathcal R}\) and local active neighborhoods
\(U_1,\ldots,U_K\). Suppose also that the dominated-tube condition in
Assumption~\ref{ass:local-nondeg}(b) holds for these neighborhoods. There are
constants \(C<\infty\) and \(s_0<\infty\) such that, for all \(s\ge s_0\), on the
event \(\mathcal L_s(\mathcal R)\),
\[
        \Xi_s(B_s)
        \le
        C h_s^{\rm loc}
        +
        Cs^{-6}.
\]
\end{lemma}

\begin{proof}
On \(\mathcal L_s(\mathcal R)\), all sweep capacities lie in \(\mathcal R\). On
the good event \(E_n(W)\), the relaxed graph and normal-cone relations place all
projected pathwise cutoffs in
\(\mathcal A_k(\mathcal R,\eta_{\mathcal R})\). Hence every active interval
appearing in the per-stage argument is contained in
\(\mathcal J_k(\mathcal R,\eta_{\mathcal R})\subseteq U_k\).

The reduction of the per-stage Jensen loss to an active-mass product cap
(Lemma~\ref{lem:jensen} and the per-stage estimate of Section~\ref{subsec:sweep})
therefore applies with Corollary~\ref{cor:active-local} in place of the global
active-mass product bound. This gives the local active-mass cap
\(C r_s^{\rm loc}\). The dominated-tube condition gives the required passage from the
size-conditioned curvature \(\Lambda_{k,z}\) to the aggregate measure \(\mu_k\)
on the local tube. In the case \(\pp_{\mathcal R}=1\), this dominated passage
bounds the per-stage loss directly by \(C r_s^{\rm loc}+Cs^{-6}\), as in the
\(\pp=1\) analysis of Section~\ref{subsec:sum} (see \eqref{eq:pone-rn}). In the
case \(\pp_{\mathcal R}>1\), the same local cap inserted into the general
per-stage estimate of Lemma~\ref{lem:epstability} gives
\(C r_s^{\rm loc}\log(e/r_s^{\rm loc})+Cs^{-6}\). These two cases are exactly the bound by \(C h_s^{\rm loc}+Cs^{-6}\).
\end{proof}

\begin{theorem}[Localized regret at interior capacities]\label{thm:local}
Suppose that
\[
        \rho_0:=\frac{b_T}{T}\in\operatorname{int}(\mathcal R_0)
\]
for compact regions
\[
        \mathcal R_0\Subset\mathcal R\subset\R_+^d.
\]
Suppose also that Assumptions~\ref{ass:local-active} and
\ref{ass:local-nondeg} hold on \(\mathcal R\) with admissible exponent
\(\pp_{\mathcal R}\). Then there is a constant \(C<\infty\), depending on the
local data but not on \(T\), such that
\[
        \Reg_T(\SPM;b_T)
        \le
        C
        \begin{cases}
        (\log(eT))^2,
        & \pp_{\mathcal R}=1,\\[1mm]
        T^{1/2-1/(2\pp_{\mathcal R})}
        (\log(eT))^{(\pp_{\mathcal R}+1)/(2\pp_{\mathcal R})+1},
        & \pp_{\mathcal R}>1.
        \end{cases}
\]
\end{theorem}

\begin{proof}
The telescoping bound \eqref{eq:telescope} gives
\[
        \Reg_T(\SPM;b_T)
        \le
        C+\sum_{s=1}^T\left(\E[\Xi_s(B_s)]+\frac{C}{s}\right).
\]
Split the expected per-stage loss according to the localization event:
\[
        \E[\Xi_s(B_s)]
        =
        \E[\Xi_s(B_s)\1_{\mathcal L_s(\mathcal R)}]
        +
        \E[\Xi_s(B_s)\1_{\mathcal L_s(\mathcal R)^c}] .
\]
By Lemma~\ref{lem:local-perstage}, the first term is at most
\(C h_s^{\rm loc}+Cs^{-6}\) for \(s\ge s_0\). Since the one-step gaps are
\(O(1)\) (Section~\ref{sec:spm}), \(\Xi_s\le C\) uniformly, so the second term is
at most \(C\,\Prob(\mathcal L_s(\mathcal R)^c)\). After summing over \(s\) and applying Lemma~\ref{lem:state-localization}, we get
\[
        \sum_{s=1}^T\E[\Xi_s(B_s)]
        \le
        C\sum_{s=1}^T h_s^{\rm loc}
        +
        O(\log T).
\]
The harmonic remainder in the telescoping bound is also \(O(\log T)\).

It remains only to sum the local scale. If \(\pp_{\mathcal R}=1\), then
\(r_s^{\rm loc}\asymp\log(es)/s\), and
\[
        \sum_{s=1}^T h_s^{\rm loc}
        =
        \sum_{s=1}^T r_s^{\rm loc}
        \le
        C(\log(eT))^2.
\]
If \(\pp_{\mathcal R}>1\), summing \(r_s^{\rm loc}\log(e/r_s^{\rm loc})\) as in
the proof of Theorem~\ref{thm:main} (Lemma~\ref{lem:epstability}), with \(\pp\)
replaced by \(\pp_{\mathcal R}\), gives
\[
        \sum_{s=1}^T h_s^{\rm loc}
        \le
        C
        T^{1/2-1/(2\pp_{\mathcal R})}
        (\log(eT))^{(\pp_{\mathcal R}+1)/(2\pp_{\mathcal R})+1}.
\]
Combining these estimates proves the theorem.
\end{proof}

The constants in Theorem~\ref{thm:local} depend on the margin of \(\rho_0\)
inside \(\mathcal R_0\). In the proof of Lemma~\ref{lem:state-localization},
this margin determines \(\eta\) and the coefficient in
\(s_\star=O(\log T)\). The constants therefore deteriorate as \(\rho_0\)
approaches the boundary of the local interior regime, consistent with the
recovery of the global rate at a critical capacity.

Without the non-degeneracy and binding-balance conditions of
Assumption~\ref{ass:local-nondeg}, the same per-stage argument gives a
conditional statement, provided that the localization leakage is controlled by
some other argument.

\begin{proposition}[Conditional localized bound]\label{prop:local-conditional}
Suppose that Assumption~\ref{ass:local-active} holds on \(\mathcal R\) with
admissible exponent \(\pp_{\mathcal R}\), and suppose that the dominated-tube
condition in Assumption~\ref{ass:local-nondeg}(b) holds for the same local active
neighborhoods. If
\[
        \sum_{s=1}^T\Prob(\mathcal L_s(\mathcal R)^c)
\]
is at most the local rate in Theorem~\ref{thm:local}, then
\(\Reg_T(\SPM;b_T)\) satisfies the bound in Theorem~\ref{thm:local}.
\end{proposition}

\begin{proof}
The proof is identical to the proof of Theorem~\ref{thm:local}, except that the
leakage term
\[
        \sum_{s=1}^T\Prob(\mathcal L_s(\mathcal R)^c)
\]
is bounded by hypothesis rather than by Lemma~\ref{lem:state-localization}.
\end{proof}

\subsection{The uniform example, revisited}\label{subsec:local-examples}

Let \(V\) and \(\beta\) be independent and uniform on \([1,2]\), and set
\(R=V/\beta\). Then \(R\in[1/2,2]\) and \(\E[\beta]=\tfrac32\). Globally, the
weighted ratio measure has a corner at \(r=1/2\) with quadratic mass. Thus the
global exponent is \(\pp=2\) by Corollary~\ref{cor:regret-bounded}. The global theorem therefore gives the
\(T^{1/4}\) polynomial rate, up to logarithmic factors, at the critical capacity
\(b_T=\tfrac32 T\).

Now fix a binding non-critical capacity \(c\in(0,3/2)\), and choose a compact
region \(\mathcal R\Subset(0,3/2)\) that contains \(c\). For every
\(\rho\in\mathcal R\), the fluid cutoff \(q^\rho\) is the unique solution of
\(\E[\beta\,\1\{R>q^\rho\}]=\rho\), and this cutoff is interior,
\(q^\rho\in(1/2,2)\). The weighted ratio measure has a density bounded above and
below on a neighborhood of the corresponding cutoff tube. The tube is therefore
interior and dominated, and the capacity-local exponent is \(\pp_{\mathcal R}=1\).
Assumptions~\ref{ass:local-active} and \ref{ass:local-nondeg} hold, and
Theorem~\ref{thm:local} gives \(O((\log(eT))^2)\) regret at \(b_T=cT\). The
constant depends on the compact region \(\mathcal R\), and hence on the distance
of \(c\) from the boundary values \(0\) and \(\tfrac32\); the statement is not
uniform as \(c\) approaches either boundary.

At the critical capacity \(c=3/2\), the cutoff is pinned at the endpoint
\(1/2\). The dual interval is flat, and the local endpoint exponent is
\(\pp_{\mathcal R}=2\) (though the interior condition of
Assumption~\ref{ass:local-nondeg} no longer holds).
State localization to an interior tube is unavailable, and the global
\(T^{1/4}\) polynomial rate is the genuine critical-capacity behavior. Thus, for
this primitive distribution, polynomial regret is a critical-capacity
phenomenon: every fixed binding capacity \(c\in(0,3/2)\) falls in the
polylogarithmic local regime, whereas the accept-all boundary \(c=3/2\)
activates the quadratic corner.

\section{Regularity verification and the regret corollaries}\label{app:examples}

This appendix verifies Assumption~\ref{ass:regularity} for the structured
distribution classes in Section~\ref{subsec:mainbound} and then proves the
regret corollaries.  In these classes, endpoint behavior is determined by two
primitive mechanisms: the mass of the ratio distribution itself, and the way
conditional ratio supports recede as the size approaches a boundary.  The
propositions below compute a per-type exponent \(\pp_k\) for each type \(k\).
The global exponent used in Theorem~\ref{thm:main} is
\(\pp=\max_k\pp_k\).

\begin{proposition}[Independent size and ratio]\label{prop:independent-ratio}
Fix a type \(k\).  Suppose that, conditional on \(J=k\), the size \(\beta\) and
the ratio \(R=V/\beta\) are independent, that
\(\beta\in[\betamin_k,\betamax_k]\), and that \(R\) has single-interval support
\([r_k^-,r_k^+]\).  Suppose also that \(R\) is regular in the interior---\(c|I|\le
\Prob(R\in I\mid J=k)\le C|I|\) for every compact interval \(I\) in the interior
of the support---and that, at each endpoint \(r\), its distribution satisfies the
one-sided interval bound
\(c|I|^{\theta_{k,r}}\le \Prob(R\in I\mid J=k)\le C|I|\) for some
\(\theta_{k,r}\ge1\) and every interval \(I\) in a one-sided endpoint
neighborhood, with endpoint intervals having mass of order
\(x^{\theta_{k,r}}\).  Then the type-\(k\) distribution satisfies
Assumption~\ref{ass:regularity}.  Its local endpoint exponents are
\(\pp_{k,r}=\theta_{k,r}\), and its active mass condition holds with per-type
exponent
\[
        \pp_k=\max\{1,\theta_{k,r_k^-},\theta_{k,r_k^+}\}.
\]
\end{proposition}

\begin{proof}
For every interval \(I\),
\[
        \mu_k(I)
        =
        \pi_k\,\E[\beta\mid J=k]\,
        \Prob(R\in I\mid J=k).
\]
Thus \(\mu_k\) has exactly the same local endpoint exponents and interval bounds
as the ratio distribution.  Cover the compact support \(S_k\) by the two
endpoint neighborhoods and finitely many interior neighborhoods, taken relatively
open in \(S_k\) and overlapping, and let
\(\lambda_0\) be a Lebesgue number of this cover.  These local bounds patch to
\eqref{eq:wr} with exponent
\(\pp_k=\max\{1,\theta_{k,r_k^-},\theta_{k,r_k^+}\}\).  The upper bound follows
by summing the local upper bounds.  For the lower bound, if
\(\ell_k(I)\le\lambda_0\), then the active part of \(I\) lies in one cover
element, so the corresponding local lower bound applies.  If
\(\ell_k(I)>\lambda_0\), then \(I\cap S_k\) contains an active subinterval of
length \(\lambda_0\), and hence
\(\mu_k(I)\ge c\lambda_0^{\pp_k}\ge c'\ell_k(I)^{\pp_k}\), since
\(\ell_k(I)\le |S_k|\).

Independence gives, for almost every \(z\),
\[
        \Lambda_{k,z}(I)
        =
        z\,\Prob(R\in I\mid J=k)
        \le
        \frac{\betamax_k}{\pi_k\E[\beta\mid J=k]}\,\mu_k(I).
\]
Interior regularity of \(R\) gives dominated neighborhoods away from the
endpoints.  At each endpoint, the displayed domination gives a dominated
one-sided neighborhood.  Therefore the finite local cover in
Assumption~\ref{ass:regularity} is fully dominated.
\end{proof}

\begin{proposition}[Independent value and size]\label{prop:independent-value-size}
Fix a type \(k\).  Suppose that, conditional on \(J=k\), the reward \(V\) and
the size \(\beta\) are independent and have compact interval supports bounded
away from zero.  Suppose that their densities are locally bounded above and
below in the interiors, that the density of \(V\) is bounded above on its
support, and that both densities are comparable near each endpoint to a power of
the distance to that endpoint.  Then the type-\(k\) distribution satisfies
Assumption~\ref{ass:regularity}.  If \(a_V^\pm\) and \(a_\beta^\pm\) denote the
endpoint exponents of the distributions of \(V\) and \(\beta\), then the local
endpoint exponents of \(\mu_k\) are
\[
        \pp_{k,r_k^-}=a_V^-+a_\beta^+,
        \qquad
        \pp_{k,r_k^+}=a_V^+ + a_\beta^- .
\]
The active mass condition holds with per-type exponent
\[
        \pp_k=\max\{a_V^-+a_\beta^+,a_V^+ + a_\beta^-\}.
\]
\end{proposition}

\begin{proof}
Write the supports of \(V\) and \(\beta\) as
\([v^-,v^+]\) and \([\betamin_k,\betamax_k]\), with \(v^->0\).  The ratio
support is
\[
        [r_k^-,r_k^+]
        =
        \left[\frac{v^-}{\betamax_k},\frac{v^+}{\betamin_k}\right].
\]
Let \(f_V\) and \(f_\beta\) be the conditional densities.  The weighted ratio
measure has density
\[
        m_k(r)
        =
        \pi_k\int_{\betamin_k}^{\betamax_k}
        z^2 f_\beta(z)f_V(rz)\1\{v^-\le rz\le v^+\}\,\dd z .
\]
Since \(f_V\) is bounded above and \(f_\beta\) is integrable, \(m_k\) is bounded
above.  Hence the upper inequality in \eqref{eq:wr} holds locally, and then
globally after compact patching.

Now let \(r\) be an interior point of \([r_k^-,r_k^+]\).  Then there is
\(z_0\in(\betamin_k,\betamax_k)\) such that \(rz_0\in(v^-,v^+)\).  On small
neighborhoods of \(z_0\) and \(r\), both densities are bounded below, so
\(m_k\) is bounded below.  Thus \(\mu_k(I)\asymp |I|\) locally at every interior
point.  Moreover, on each such local neighborhood,
\[
        \Lambda_{k,z}(I)
        =
        z\,\Prob(V/z\in I\mid J=k)
        \le C |I|
        \le C\mu_k(I),
\]
so the interior neighborhoods are dominated.

It remains to verify the endpoint neighborhoods.  Since \(V\) and \(\beta\) are
independent conditional on \(J=k\), the conditional curvature for a fixed size
\(z\) is
\[
        \Lambda_{k,z}(\dd r)
        =
        z\,\Prob(V/z\in\dd r\mid J=k)
        =
        z^2 f_V(zr)\,\1\{v^-\le zr\le v^+\}\,\dd r .
\]
Thus the kernel measure in Definition~\ref{def:endpoint-contact} has product
density
\begin{equation}
        \mathfrak M_k(\dd z,\dd r)
        =
        \pi_k\,f_\beta(z)\,z^2 f_V(zr)\,
        \1\{v^-\le zr\le v^+\}\,\dd z\,\dd r .
        \label{eq:ivs-kernel-density}
\end{equation}

We use the endpoint power comparisons
\[
        f_V(v^-+s)\asymp s^{a_V^- -1},
        \quad
        f_V(v^+-s)\asymp s^{a_V^+ -1},
        \quad
        f_\beta(\betamin_k+s)\asymp s^{a_\beta^- -1},
        \quad
        f_\beta(\betamax_k-s)\asymp s^{a_\beta^+ -1}.
\]
The bounded-above assumption on \(f_V\) implies \(a_V^-,a_V^+\ge1\); otherwise
the corresponding endpoint density would blow up.  Therefore the branch
curvature exponents below satisfy the requirement \(\gamma\ge1\) in
Definition~\ref{def:branch}.  The endpoint exponents of \(\beta\) may be any
positive numbers, and they enter through the size-coordinate exponent
\(\alpha\).

\emph{Lower endpoint.}
The lower ratio endpoint is \(r_k^-=v^-/\betamax_k\).  Use the local coordinates
\(\omega=\betamax_k-z\) and \(x=r-r_k^-\), so that
\(z=\betamax_k-\omega\).  The conditional lower edge of the ratio support
recedes from \(r_k^-\) by
\[
        e(\omega)
        =
        \frac{v^-}{\betamax_k-\omega}-\frac{v^-}{\betamax_k}
        \asymp \omega .
\]
For \(r=r_k^-+x\), we have
\[
        zr
        =
        (\betamax_k-\omega)(r_k^-+x)
        =
        v^-+(\betamax_k-\omega)\bigl(x-e(\omega)\bigr).
\]
Hence, on a sufficiently small lower endpoint neighborhood, the branch
curvature is
\[
        \Lambda^{\mathrm{br}}_{k,\omega}(\dd x)
        =
        (\betamax_k-\omega)^2
        f_V\!\bigl(v^-+(\betamax_k-\omega)(x-e(\omega))\bigr)
        \1\{x\ge e(\omega)\}\,\dd x .
\]
Its density is comparable to \((x-e(\omega))^{a_V^- -1}\).

Substituting \(z=\betamax_k-\omega\) and \(r=r_k^-+x\) in
\eqref{eq:ivs-kernel-density} gives the exact local disintegration
\[
        \int g(z,r)\,\mathfrak M_k(\dd z,\dd r)
        =
        \int_0^{\omega_0}
        \pi_k f_\beta(\betamax_k-\omega)
        \int g(\betamax_k-\omega,r_k^-+x)
        \Lambda^{\mathrm{br}}_{k,\omega}(\dd x)\,\dd\omega .
\]
This is the branch identity \eqref{eq:branch-disint-def} with
\(\beta_k(\omega)=\betamax_k-\omega\), \(w\equiv1\), and
\(f(\omega)=\pi_k f_\beta(\betamax_k-\omega)\asymp
\omega^{a_\beta^+-1}\).  Fix the neighborhood width by \(x_0=e(\omega_0)\); since
\(e\) is increasing, \(e(\omega)>x_0\) for \(\omega>\omega_0\), so
\(\Lambda_{k,z}|_U=0\) for every size \(z\le\betamax_k-\omega_0\).  The
disintegration over \(\omega\in(0,\omega_0)\) therefore captures all of
\(\mathfrak M_k\) on \(\mathcal B_k\times U\), and the whole lower endpoint kernel
is one branch with no dominated remainder, so \(\Lambda^D_{k,z}=0\) and
\(\Lambda^{\mathrm{br}}_{k,z}=\Lambda_{k,z}|_U\).  Therefore the lower endpoint
has an endpoint-contact representation with
\[
        \gamma=a_V^-,
        \qquad
        \alpha=a_\beta^+,
        \qquad
        \tau=1,
        \qquad
        \pp_{k,r_k^-}=a_V^-+a_\beta^+ .
\]
Lemma~\ref{lem:mass} gives the lower endpoint mass order.  Together with the
dominated interior cover and the global upper density bound, this verifies
\eqref{eq:wr} near \(r_k^-\).

\emph{Upper endpoint.}
The upper ratio endpoint is \(r_k^+=v^+/\betamin_k\).  Use the local coordinates
\(\omega=z-\betamin_k\) and \(x=r_k^+-r\).  The conditional upper edge
\(v^+/(\betamin_k+\omega)\) recedes from \(r_k^+\) by
\(e(\omega)\asymp\omega\).  For \(r=r_k^+-x\),
\[
        zr
        =
        (\betamin_k+\omega)(r_k^+-x)
        =
        v^+-(\betamin_k+\omega)\bigl(x-e(\omega)\bigr).
\]
The same computation, again with \(x_0=e(\omega_0)\), applied to
\eqref{eq:ivs-kernel-density} yields a single branch with no dominated remainder.  Its branch curvature density is comparable
to \((x-e(\omega))^{a_V^+ -1}\), and its size density is
\(f(\omega)=\pi_k f_\beta(\betamin_k+\omega)\asymp
\omega^{a_\beta^- -1}\).  Thus the upper endpoint has an endpoint-contact
representation with
\[
        \gamma=a_V^+,
        \qquad
        \alpha=a_\beta^-,
        \qquad
        \tau=1,
        \qquad
        \pp_{k,r_k^+}=a_V^+ + a_\beta^- .
\]
Patching the two endpoint neighborhoods with the dominated interior
neighborhoods, all taken relatively open in \(S_k\) and overlapping, gives
\eqref{eq:wr} with exponent
\(\max\{a_V^-+a_\beta^+,a_V^+ + a_\beta^-\}\), by the same finite-cover argument
used in Proposition~\ref{prop:independent-ratio}.
\end{proof}

We now prove the regret corollaries.  By Theorem~\ref{thm:main}, it suffices in
each case to verify Assumption~\ref{ass:regularity} and identify the active
weighted-mass exponent \(\pp\).

\begin{proof}[Proof of Corollary~\ref{cor:regret-regular}]
A density bounded above and below gives the active mass condition with
\(\pp=1\), and Proposition~\ref{prop:independent-ratio} gives a fully dominated
cover.  The result is the \(\pp=1\) case of Theorem~\ref{thm:main}.
\end{proof}

\begin{proof}[Proof of Corollary~\ref{cor:regret-ratio}]
By Proposition~\ref{prop:independent-ratio}, the active mass exponent is
\(\pp=\theta\), and the finite cover is dominated.  Substituting this exponent
into Theorem~\ref{thm:main} gives the bound.
\end{proof}

\begin{proof}[Proof of Corollary~\ref{cor:regret-bounded}]
Densities bounded above and below have endpoint exponent \(1\).  Therefore
Proposition~\ref{prop:independent-value-size} verifies the active mass condition
with \(\pp=2\).  The result is Theorem~\ref{thm:main} at \(\pp=2\), for which
the logarithmic exponent is \((\pp+1)/(2\pp)+1=7/4\).  Uniform, truncated normal,
and truncated exponential distributions on compact intervals bounded away from
zero are instances.
\end{proof}

\begin{proof}[Proof of Corollary~\ref{cor:regret-beta}]
The uniform value distribution has endpoint exponents \(a_V^-=a_V^+=1\).  The
affine \(\mathrm{Beta}(a_k,b_k)\) distribution has endpoint exponent \(a_k\) at
\(\betamin_k\) and \(b_k\) at \(\betamax_k\).  By
Proposition~\ref{prop:independent-value-size}, the weighted ratio measure has
endpoint exponents \(1+b_k\) and \(1+a_k\).  Hence the active mass exponent is
\(\pp=1+\max_k\{a_k,b_k\}=1+q\).  Substituting this exponent into
Theorem~\ref{thm:main} gives the bound.
\end{proof}

\subsection{Fluid duals in two one-resource examples}
\label{app:two-uniform-duals}

This subsection compares two one-resource examples.  Each has one request
type, horizon-\(T\) capacity \(b_T=cT\), scalar size \(\beta\), reward \(V\),
and ratio \(R=V/\beta\).  The fluid relaxation is
\[
        \max_{0\le x(V,\beta)\le 1}
        \E[Vx(V,\beta)]
        \quad
        \text{s.t.}\quad
        \E[\beta x(V,\beta)]\le c .
\]
Its Lagrange dual is
\[
        \min_{\lambda\ge0}
        D(\lambda),
        \qquad
        D(\lambda):=c\lambda+\E[(V-\lambda\beta)^+].
\]
Equivalently, since \(V=\beta R\), we have
\(D(\lambda)=c\lambda+\E[\beta(R-\lambda)^+]\).  The optimal fluid policy
accepts requests whose ratio \(R\) exceeds an optimal dual price \(\lambda\),
with arbitrary tie-breaking at \(R=\lambda\).  The examples below show that dual
degeneracy and the mass exponent \(\pp\) are distinct phenomena.

\paragraph{Independent \(V,\beta\sim U[1,2]\) and \(c=3/2\).}
Let \(V\) and \(\beta\) be independent and uniformly distributed on \([1,2]\),
and set \(c=\E[\beta]=3/2\).  Since \(V>0\) almost surely, the accept-all
solution \(x\equiv1\) is optimal whenever it is feasible.  It is feasible here
because \(\E[\beta]=c\).  Thus the fluid relaxation accepts all requests.

The dual objective is \(D(\lambda)=(3/2)\lambda+\E[(V-\lambda\beta)^+]\).  If
\(0\le\lambda\le1/2\), then \(V-\lambda\beta\ge0\) for all
\((V,\beta)\in[1,2]^2\).  Hence
\[
        D(\lambda)
        =
        \frac32\lambda+\E[V-\lambda\beta]
        =
        \E[V]+\lambda(c-\E[\beta])
        =
        \frac32 .
\]
Thus every \(\lambda\in[0,1/2]\) is dual optimal, since \(D(\lambda)\ge\E[V]\)
by weak duality while the accept-all primal solution already attains value
\(\E[V]=3/2\).  For \(\lambda>1/2\), on the other hand,
\(D'(\lambda)=3/2-\E[\beta\,\1\{R>\lambda\}]>0\) because \(\Prob(R<\lambda)>0\);
so \(D\) is strictly increasing beyond \(1/2\), and the dual optimal set is
exactly \(\Lambda^*=[0,1/2]\).  The fluid dual is not unique.

In this example, the active cutoff is at the lower edge of the ratio support.
The ratio \(R=V/\beta\) reaches its minimum \(1/2\) only at the corner
\((V,\beta)=(1,2)\).  Consequently, the size-weighted ratio mass near the active
cutoff is thinner than linear.  Indeed, for small \(\varepsilon>0\),
\[
        \E\!\left[\beta\,
        \1\left\{\frac{V}{\beta}\in[1/2,1/2+\varepsilon]\right\}\right]
        \asymp \varepsilon^2 .
\]
Thus this example has \(\pp=2\).  It is dual degenerate and belongs to the
polynomial regime; it is the instance of Corollary~\ref{cor:regret-bounded}.

\paragraph{Independent \(\beta\sim U[1,2]\), \(R\sim U[1/2,2]\), and \(c=3/2\).}
Finally, let \(\beta\sim U[1,2]\), let \(R\sim U[1/2,2]\), assume independence,
and set \(V=\beta R\).  Again take \(c=\E[\beta]=3/2\).  Since \(R>0\), the
accept-all solution is optimal and feasible.

The dual objective is
\[
        D(\lambda)
        =
        \frac32\lambda+\E[\beta(R-\lambda)^+]
        =
        \frac32\left\{\lambda+\E[(R-\lambda)^+]\right\},
\]
where the last equality uses independence and \(\E[\beta]=3/2\).  If
\(0\le\lambda\le1/2\), then \(R-\lambda\ge0\) almost surely, and therefore
\[
        \lambda+\E[(R-\lambda)^+]
        =
        \lambda+\E[R-\lambda]
        =
        \E[R]
        =
        \frac54 .
\]
Thus \(D(\lambda)=15/8\) for all \(\lambda\in[0,1/2]\).

For \(\lambda\in[1/2,2]\),
\[
        \E[(R-\lambda)^+]
        =
        \frac{1}{2-1/2}\int_\lambda^2(r-\lambda)\,\dd r
        =
        \frac{(2-\lambda)^2}{3}.
\]
Hence \(\lambda+\E[(R-\lambda)^+]=\lambda+(2-\lambda)^2/3\), whose derivative is
\((2\lambda-1)/3\).  This derivative is nonnegative on \([1/2,2]\) and strictly
positive for \(\lambda>1/2\).  For \(\lambda>2\), the positive-part term
vanishes and \(D(\lambda)=(3/2)\lambda\).  Therefore the dual optimal set is
again \(\Lambda^*=[0,1/2]\).  The fluid dual is not unique.

This example is dual degenerate, but it is not in the polynomial regime.  For a
Borel set \(B\subset[1/2,2]\), the size-weighted ratio measure is
\[
        \mu(B)=\E[\beta\,\1\{R\in B\}]
        =
        \E[\beta]\Prob(R\in B)
        =
        \frac32\Prob(R\in B),
\]
by independence.  Since \(R\) has a density bounded above and below on
\([1/2,2]\), we have
\(\mu([1/2,1/2+\varepsilon])\asymp \varepsilon\).  Thus \(\pp=1\).  The instance
is dual degenerate, but the size-weighted ratio mass near the active cutoff is
linear, so it belongs to the logarithmic-type regime; it is the instance of
Corollary~\ref{cor:regret-regular}.

These two examples separate the roles of dual degeneracy and cutoff mass.  Both
are dual degenerate, with the same dual-optimal set \([0,1/2]\).  Independent
\(V,\beta\sim U[1,2]\) has thin cutoff mass and \(\pp=2\), while independent
\(\beta\sim U[1,2]\) and \(R\sim U[1/2,2]\) has linear cutoff mass and
\(\pp=1\).  Thus the polynomial
regret mechanism is not dual degeneracy alone; it is dual degeneracy together
with insufficient size-weighted ratio mass near the active cutoff.

\end{document}